\documentclass{article} %
\usepackage{packages/iclr2027/iclr2027_conference}
\usepackage{times}

\usepackage{hyperref}       %
\usepackage{url}            %
\usepackage{booktabs}       %
\usepackage{amsfonts}       %
\usepackage{nicefrac}       %
\usepackage{microtype}      %
\usepackage[x11names]{xcolor}          %
\usepackage{wrapfig}
\usepackage{amsmath}
\usepackage{amssymb}
\usepackage{mathtools}
\usepackage{amsthm}
\usepackage{bm}
\usepackage[capitalize,noabbrev]{cleveref}
\usepackage[edges]{forest}
\usepackage{xfrac}
\usepackage{multicol}
\usepackage{makecell}

\theoremstyle{plain}
\newtheorem{theorem}{Theorem}[section]
\newtheorem{proposition}[theorem]{Proposition}

\theoremstyle{definition}
\newtheorem{definition}[theorem]{Definition}

\theoremstyle{remark}

\usepackage{xparse} %

\usepackage{fontspec}
\usepackage{soul}
\usepackage{multirow}
\usepackage{algorithm}
\usepackage{algpseudocode}
\usepackage{cancel}

\usepackage[skip=4pt,compatibility=false]{caption}
\usepackage{subcaption}
\usepackage{enumitem}

\newcommand{\printbibliography}{%
  \bibliographystyle{packages/iclr2027/iclr2027_conference}%
  \bibliography{main}
}

\usepackage{tcolorbox}
\tcbuselibrary{breakable, listings, skins}

\definecolor{beige}{RGB}{245,245,220}

\tcbset{
  mybeigestyle/.style={
    colback=beige,
    colframe=gray,
    coltext=black,
    boxrule=0.5mm,
    arc=3mm,
    breakable,
    enhanced, %
    fonttitle=\bfseries
  }
}

\newtcblisting[auto counter]{beigebox}[2][]{
  mybeigestyle,
  listing only,
  title={Snippet \thetcbcounter: #2},
    listing options={
    basicstyle=\ttfamily\small,
    breaklines=true,
    columns=fullflexible
  },
  #1
}

\newtcbinputlisting[auto counter]{\beigefile}[3][]{
  mybeigestyle,
  listing file={#3},
  listing only,
  title={Template \thetcbcounter: #2},
  listing options={
    basicstyle=\ttfamily\small,
    breaklines=true,
    columns=fullflexible
  },
  #1
}

\newcommand{\E}{\mathbb{E}}

\author{%
  T. Duy Nguyen-Hien\textsuperscript{1} \quad
  Yee Whye Teh\textsuperscript{2} \quad
  Wee Sun Lee\textsuperscript{1} \quad
  Tan Zhi-Xuan\textsuperscript{1,3} \\[0.5em]
  \textsuperscript{1}Department of Computer Science, National University of Singapore \\
  \textsuperscript{2}Department of Statistics, University of Oxford \\
  \textsuperscript{3}Agency for Science, Technology and Research (A*STAR)
  \\[0.3em]
  \texttt{duynht@u.nus.edu} \quad
  \texttt{y.w.teh@stats.ox.ac.uk} \\
  \texttt{dcsleews@nus.edu.sg} \quad
  \texttt{xuan.cs@nus.edu.sg}
}

\newcommand{\finalonly}[1]{%
  \ifdefined\iclrfinalcopy
    #1%
  \fi
}

\newcommand{\finalifelse}[2]{%
  \ifdefined\iclrfinalcopy
    #1%
  \else
    #2%
  \fi
}

\iclrfinalcopy %

\title{Rational Clarification by Assistive Agents via Value-of-Information Reasoning}

\begin{document}
\maketitle

\begin{abstract}
Users of language-based assistive agents often make ambiguous requests. In response, an assistant can either directly act on its interpretation of the request --- risking misalignment with the user --- or ask a clarifying question. Which option is the most safe and helpful? A common approach is to ask questions that minimize uncertainty about the user's intent until a threshold is reached. However, this neglects the impact of uncertainty reduction on downstream performance, the costs of asking versus acting immediately, and the possibility that users may provide corrections without being asked. To navigate these trade-offs, we introduce \emph{Rational Enquiry via Value-of-Information Reasoning} (REVOIR). REVOIR makes clarification decisions via inference-time reasoning about the \emph{value-of-information} of a question, which captures the expected improvement in task reward due to the answer received. In two assistive tasks --- ambiguous question answering (CondAmbigQA) and preference-aligned household task planning (ADAPT) --- we show that REVOIR achieves greater success with fewer questions than approaches based on prompting, chain-of-thought, fine-tuning, or information gain, improving preference satisfaction on ADAPT by $13$--$15\%$ over a fine-tuned clarification policy while requiring no training and asking five times fewer questions. Furthermore, when the assistant can receive cheap user corrections after acting, REVOIR naturally infers that asking questions is not always efficient, demonstrating the adaptivity of our approach. In contrast, we find that vanilla reasoning agents fail to adaptively clarify user requests, and request \emph{fewer} clarifications as reasoning effort increases.\looseness=-1

\end{abstract}
\section{Introduction}

\begin{figure}[t]
    \centering
    \includegraphics[width=\textwidth]{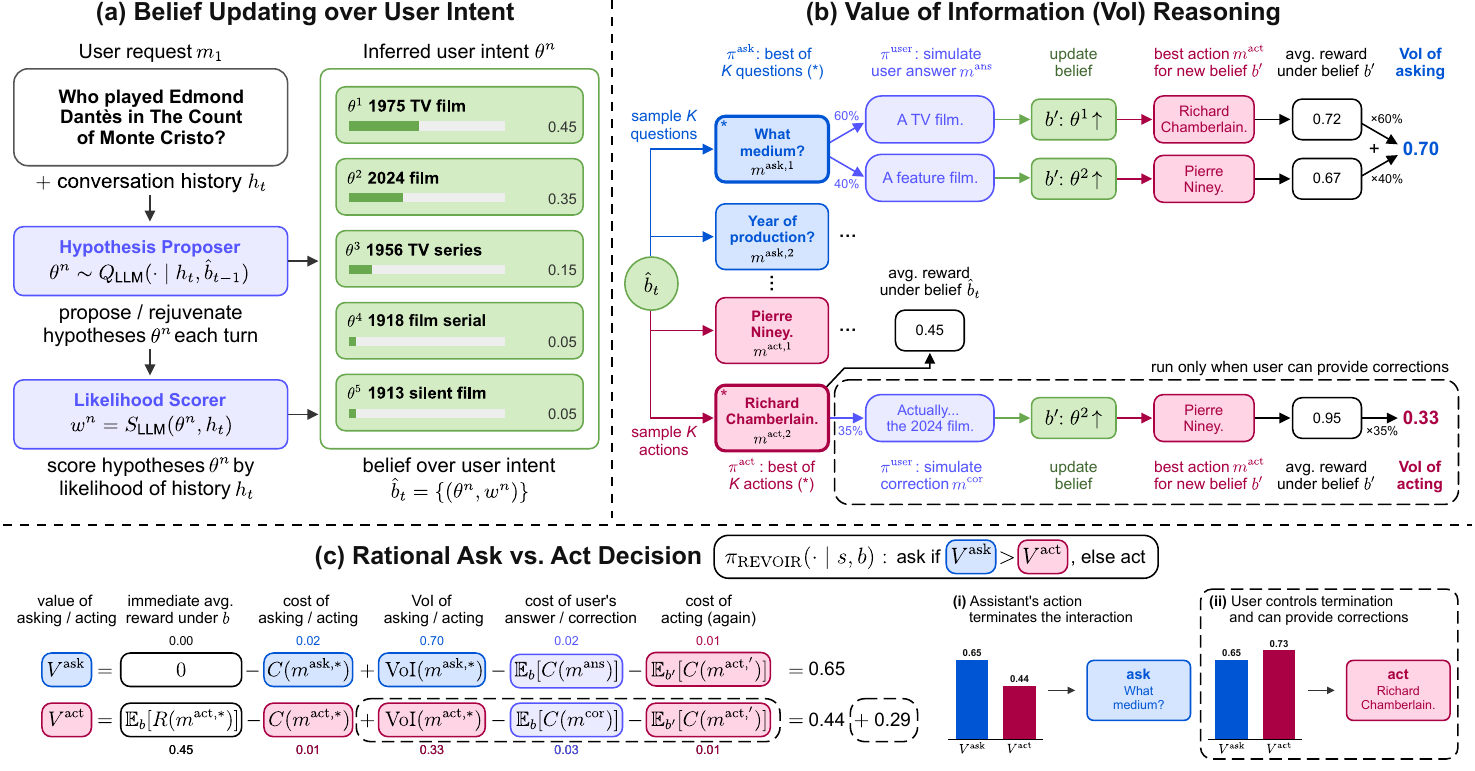}
    \caption{\textbf{Overview of REVOIR.} \textbf{(a)} Given an ambiguous user request, REVOIR infers and updates a belief over the user's intent $\theta$. \textbf{(b)} REVOIR computes the value-of-information (VoI) of asking (vs. acting) by simulating the expected benefit of acting after a clarifying question (vs. a user correction). \textbf{(c)} REVOIR decides between asking or acting by maximizing the cost-adjusted value of asking $V^\mathrm{ask}$ vs. acting $V^\mathrm{act}$, rationally adapting to cases where: (i) the assistant's action is terminal; (ii) the user can give corrections.}
    \label{fig:overview}
\end{figure}

People often express their intentions to AI agents in under-specified ways.
Faced with such ambiguity, how should an assistive agent respond? Agents that directly execute the request might misinterpret the user, resulting in misaligned actions with potentially irreversible consequences~\citep{ming2025replit}. Instead, an agent might ask clarifying questions to determine the right interpretation, thereby avoiding costly mistakes. However, asking questions may not always be necessary. If the assistant is certain enough about the user's intent, or if the user can easily correct the assistant afterwards, it may be best to act on the request directly.

In this paper, we take a \emph{rational} approach to addressing these trade-offs in assistive agents based on large language models (LLMs): 
\begin{enumerate}[leftmargin=*,itemsep=1pt, topsep=0pt]
    \item We formulate clarificatory assistance as a cooperative game \citep{hadfield-menell_cooperative_2016}, where the assistant starts off uncertain about the user's intent, but can learn by asking clarifying questions or (possibly) receiving corrections after acting upon the user's request --- a form of learning neglected by past work on multi-turn clarification.
    \item To solve these assistance problems, we introduce \textbf{Rational Enquiry via Value-of-Information Reasoning} (\textbf{REVOIR}, Figure \ref{fig:overview}), a decision-theoretic assistant that (a) updates its beliefs about the user's intentions in natural language; (b) estimates the 
    \emph{value-of-information} (VoI) \citep{howard_information_1966,raiffa_applied_1961} of clarificatory actions --- i.e., the expected improvement in task reward due to the answer received; (c) uses this to make rational decisions about asking questions vs. acting immediately.
\end{enumerate}

We evaluate REVOIR across ambiguous question answering (CondAmbigQA, \citealp{li_condambigqa_2025}) and preference-aligned household task planning (ADAPT, \citealp{patel_adapt_2025}). We compare REVOIR against prompting baselines, generic inference-time reasoning, fine-tuning for clarifying questions, and using expected information gain, and find that our approach achieves greater task success and preference alignment with fewer questions. On ADAPT, REVOIR reaches preference satisfaction rates of $56$--$59\%$ using only in-context examples and no training, exceeding a clarification policy fine-tuned for the task ($44\%$) by $13$--$15\%$ while asking roughly $5\times$ fewer questions, and surpassing even an ``always ask'' baseline by around $6\%$. REVOIR also adapts rationally to a variant of the QA task where users can provide post-hoc corrections --- a setting ignored by prior clarification methods \citep{andukuri_stargate_2024,zhang_modeling_2025} --- successfully inferring that asking a question often provides no benefit beyond acting now and being corrected later. In contrast, generic reasoners (ReAct + reasoning-trained LLMs) fall short of REVOIR at reasoning about whether further clarifications are helpful, and indeed generally engage in \emph{less} clarifications as reasoning effort is increased.

\textbf{Related Work.} REVOIR is distinct from prior work in several ways. Unlike generic inference-time reasoning (e.g. chain-of-thought  \citealp{wei_chainofthought_2022,deng_prompting_2023}, or ReAct \citealp{yao_react_2022}), which lack a normative standard for clarification decisions, REVOIR builds upon the theory of assistance games \citep{hadfield-menell_cooperative_2016,zhi-xuan_pragmatic_2024a,ma_openuniverse_2025,laidlaw_assistancezero_2025}, showing how rational assistance under uncertainty can be applied to LLM agents while remaining tractable. Finetuning methods optimize for clarification by maximizing the quality of subsequent assistant responses \citep{andukuri_stargate_2024,zhang_modeling_2025,patel_adapt_2025,wu_collabllm_2025,chi_clarinet_2024}, but typically assume fixed user costs, and that conversations terminate after the assistant commits to an answer \citep{andukuri_stargate_2024,zhang_modeling_2025}. In comparison, REVOIR can accommodate user-specified costs for questions, and naturally adapts to interactions where users can choose whether to accept the assistant's response or provide further corrections. Finally, our approach optimizes the cost-adjusted VoI of a question rather than its expected information gain (EIG) \citep{lindley_measure_1956,mackay_information_1992}, a metric often used to select uncertainty-reducing questions until a fixed budget or confidence threshold is reached \citep{hu_uncertainty_2024,handa_bayesian_2024,piriyakulkij_asking_2023,grand_shoot_2025}. Apart from neglecting the variable costs of questions, EIG is insensitive to downstream task reward, and may score questions highly even if they do not provide task-relevant information.
\section{Clarificatory Assistance as a Cooperative Game}
\label{sec:assistance-game}
How should an assistive agent rationally decide whether to clarify user requests? REVOIR builds on the framework of \emph{assistance games} \citep{hadfield-menell_cooperative_2016}, which formalizes how a user-aligned agent should provide assistance under uncertainty about what the user wants. 
We adapt this framework to the setting of LLM-based assistive agents, which interact with a user via language while (optionally) taking actions in an external environment:
\vspace{2pt}

\begin{definition}
A \emph{conversational assistance game} is a tuple $\langle \mathcal{S}, \{\mathcal{A}^{\mathrm{user}}, \mathcal{A}^{\mathrm{agent}}\}, T, \Theta, R, P_0 \rangle$. Each element has the following definitions:

\textbf{States.} The state space factorizes as $\mathcal{S} = \mathcal{W} \times \mathcal{H}$, where $w \in \mathcal{W}$ is a state of the external world (e.g., a physical room or code-base), and $h_t \in \mathcal{H}$ is a conversational history $h_t = (m_1, m_2, \dots, m_t)$ composed of alternating user and agent messages $m_t$. A subset of $\mathcal{S}^{\mathrm{term}} \subseteq \mathcal{S}$ are terminal states, depending on whether the final message $m_t$ ends the conversation.

\textbf{Actions.} $\mathcal{A}^{\mathrm{user}}$ and $\mathcal{A}^{\mathrm{agent}}$ are the user and agent action spaces. Each action space $\mathcal{A}^i = \mathcal{X}^i \cup \mathcal{M}$ decomposes into external actions $\mathcal{X}^i$ and messages $\mathcal{M}$. We further assume that $\mathcal{M}$ can be categorized into user requests $\mathcal{M}^{\mathrm{req}}$, agent questions $\mathcal{M}^{\mathrm{ask}}$, user answers $\mathcal{M}^{\mathrm{ans}}$, verbal actions $\mathcal{M}^{\mathrm{act}}$, user corrections $\mathcal{M}^{\mathrm{cor}}$, and end-of-conversation messages $\mathcal{M}^{\mathrm{end}}$. 

\textbf{Transitions.} The transition function $T(s_t \mid s_{t-1}, a^{_\mathrm{user}}_t, a^{_\mathrm{agent}}_t)$ factorizes into world and conversation transitions. We assume players take turns acting, so $a^{_\mathrm{user}}_t = \bot$ when $a^{_\mathrm{agent}}_t \neq \bot$ and vice versa; let $a_t$ denote the acting player's action. External actions ($a_t \in \mathcal{X}$) update the world via $T^{\mathrm{world}}(w_t \mid w_{t-1}, a_t)$ while leaving the conversational history unchanged.

\textbf{User Intents.} $\Theta$ is the space of user intents $\theta \in \Theta$, which we represent in natural language. Crucially, $\theta$ is initially \emph{unknown} to the agent, and is not fully revealed by the user's initial action. An intent $\theta$ determines the goal reward associated with a terminal state $s \in \mathcal{S}^{\mathrm{term}}$.

\textbf{Rewards and Costs.} The user's reward function, shared by the agent, decomposes as $R(s_t, a^{\mathrm{user}}_t, a^{\mathrm{agent}}_t, \theta) := G(s_t, \theta) \;-\; C(a^{\mathrm{user}}_t, a^{\mathrm{agent}}_t)$. Here, $G(s_t, \theta)$ is a (weakly) positive \emph{goal reward} dependent on the hidden intent $\theta$, realized only when $s_t \in \mathcal{S}^{\mathrm{term}}$ is terminal.  Unlike clarification tasks with multiple choice answers~\citep{li_mediq_2024,hu_uncertainty_2024}, $G(s_t, \theta)$ can vary continuously, so maximizing reward does not reduce to discovering the user's intent $\theta$. $C(a^{\mathrm{user}}_t, a^{\mathrm{agent}}_t)$ is a user-configurable \emph{cost function} known to the agent, capturing the costs to the user of reading questions, writing answers, or receiving the wrong assistive action.

\textbf{Gameplay.} The user's intent $\theta \sim P_0(\theta)$ and initial world state $w_0 \sim P_0(w_0)$ are drawn from their priors. Starting from the empty conversation $h_0 = ()$, the user and agent take turns according to their policies $\pi^{\mathrm{user}}(a^{\mathrm{user}}_t | s_{t-1}, \theta)$ and $\pi^{\mathrm{agent}}(a^{\mathrm{agent}}_t | s_{t-1})$, with the user going first. The game ends when a player sends a terminal message $m_t \in \mathcal{M}^{\mathrm{end}}$.
\end{definition}

\textbf{Termination Control.} Prior work on clarification generally assumes that the interaction ends once the assistant takes an action $m^{\mathrm{act}}$ (e.g. answering the user's  query)~\citep{andukuri_stargate_2024,zhang_modeling_2025}. Our formulation instead allows termination to be controlled by \emph{either} the assistant or the user. In one case, assistant actions $m^{\mathrm{act}}$ are terminal. In the other case, users can decide between providing corrections $m^{\mathrm{cor}}$ or ending the conversation with a message $m^{\mathrm{end}}$ (see \ref{sec:condambigqa-exp}), capturing a wide range of more natural interactions.

\textbf{Modeling the User.} To solve an assistance game from the agent's perspective, we fix a user policy $\pi^{\mathrm{user}}(a^{\mathrm{user}}_t \mid h_t, \theta)$. This reduces the game to an \emph{assistive partially observable Markov decision process} (A-POMDP) \citep{hadfield-menell_cooperative_2016,laidlaw_assistancezero_2025}. Solving this POMDP amounts to optimal assistance with respect to the user model $\pi^{\mathrm{user}}$.

For our assistant to be helpful to \emph{actual} users, its user model needs to be reasonably realistic. We thus make the following assumptions suited for the conversational assistance context:

\begin{itemize}[leftmargin=*,itemsep=1pt,topsep=0pt]
    \item The user always initiates the conversation with a request message $m_1 \in \mathcal{M}^{\mathrm{req}}$.
    \item In response to a question $m_{t-1} \in \mathcal{M}^{\mathrm{ask}}$, the user replies with an answer $m_t \in \mathcal{M}^{\mathrm{ans}}$.
    \item When the user has termination control, after the assistant acts via $m_{t-1} \in \mathcal{M}^{\mathrm{act}}$, the user can choose to issue a correction $m_t \in \mathcal{M}^{\mathrm{cor}}$ or end the conversation $m_t \in \mathcal{M}^{\mathrm{end}}$.
\end{itemize}

We instantiate this in our experiments by prompting an LLM to simulate a human user, providing it with the true or hypothesized user intent $\theta$.
For increased realism, we do not give the assistant a true model of the user; instead, it has an \emph{internal} user model that differs from the \emph{external} user simulator in our benchmark environments. 
In domains with external actions (e.g., ADAPT), we also assume that users delegate all such actions to the agent, leaving dual-control environments to future work \citep{barres_tau2_2025}.

\section{Rational Enquiry via Value-of-Information Reasoning}
\label{sec:voi-reasoning}

In principle, solving the assistive POMDP in Section~\ref{sec:assistance-game} yields optimal clarification behavior. However, exact POMDP planning is intractable, and existing solvers for POMDPs \citep{kurniawati_sarsop_2008,silver_pomcp_2010,somani_despot_2013} and assistance games \citep{malik_efficient_2018,laidlaw_assistancezero_2025} cannot be applied to natural language interactions and goal spaces. We therefore propose \textbf{Rational Enquiry via Value-of-Information Reasoning} (\textbf{REVOIR}, Figure \ref{fig:overview}), an inference-time algorithm that provides rational assistance with limited lookahead. REVOIR is a hierarchical control policy that decides between (i) an \emph{ask} sub-policy $\pi^{\mathrm{ask}}$ that produces a clarifying question $m^{\mathrm{ask}} \in \mathcal{M}^{\mathrm{ask}}$; (ii) an \emph{act} sub-policy $\pi^{\mathrm{act}}$ that produces an assistive action $a^{\mathrm{act}} \in \mathcal{A}^{\mathrm{act}} = \mathcal{X} \cup \mathcal{M}^{\mathrm{act}}$. At each turn $t$, REVOIR:
\begin{enumerate}[leftmargin=*,itemsep=1pt,topsep=0pt]
    \item Updates its belief $b_t(\theta)$ over the user's intent given the conversational history $h_t$;
    \item Evaluates the VoI of $\pi^{\mathrm{ask}}$ under belief $b_t$, and (if corrections are possible) the VoI of $\pi^{\mathrm{act}}$;
    \item Decides between asking with $\pi^{\mathrm{ask}}$ vs. acting with $\pi^{\mathrm{act}}$ based on their cost-adjusted values.
\end{enumerate}
We unpack each step of REVOIR below.

\subsection{Belief Updating over User Intent}
\label{sec:belief-updating}

To estimate how much decision quality improves with new information, the agent needs to both update and simulate future changes in its beliefs about rewarding outcomes, which in turn depend on its belief $b_t$ about the user's intent $\theta$. This belief is given via Bayes rule as:
\begin{equation}
b_t(\theta) \propto b_0(\theta) \textstyle\prod_{\tau=1}^t \pi^{\mathrm{user}}(a^{\mathrm{user}}_\tau \mid s_{\tau-1}, \theta) = b_{t-1}(\theta) \cdot \pi^{\mathrm{user}}(a^{\mathrm{user}}_t \mid s_{t-1}, \theta),
\end{equation}
where $b_0(\theta) := P_0(\theta)$ is the agent's prior over intents, and $\pi^{\mathrm{user}}$ is the user model.

Representing and updating $b_t(\theta)$ is challenging in practice: The space of (language-represented) intents $\Theta$ is intractable to enumerate over, and we may not have trustworthy models of the prior $b_0(\theta)$ or robust estimates of the likelihood $\pi^{\mathrm{user}}(a^{\mathrm{user}}_t \mid s_{t-1}, \theta)$
\footnote{LLM log-probabilities can be used to score the prior and the likelihood, since $\theta$ and $a^{\mathrm{user}}$ are in language, but these are often sensitive to semantically-irrelevant features like typos or text length.}. We address this by approximating $b_t$ with a weighted particle set $\hat b_t = \{(\theta^n, w^n)\}_{n=1}^N$, where:
\begin{itemize}[leftmargin=*,itemsep=1pt,topsep=0pt]
    \item Intent hypotheses $\theta^n$ are proposed by an LLM $Q_{\textrm{LLM}}(\theta | h_t, \hat{b}_{t-1})$ given the history $h_t$ and past belief $\hat{b}_{t-1}$ (if available), with the option of retaining or pruning particles in $\hat{b}_{t-1}$;
    \item Each $\theta^n$ is assigned a weight $w^n = S_{\textrm{LLM}}(\theta^n, h_t)$ by an LLM that verbally scores the consistency of $\theta^n$ with  $h_t$, serving as a ``semantic'' log-likelihood of conversation $h_t$.
\end{itemize}
This scheme can be viewed as a form of variational particle approximation \citep{saeedi_variational_2017,afshar_optimal_2024}, which approximates the full posterior belief by weighting a small set of particles by their (log) probabilities. In our experiments, we explore several variants that improve performance, including incremental updating of $w^n$ as each user message $a^{\mathrm{user}}_t$ is received, and factorizing $\theta$ into independently-updated components (see Appendix \ref{sec:adapt_details}).

\subsection{Evaluating the Value of Information (VoI)}
\label{sec:voi-evaluation}

Assessing the value of a question involves determining how it will improve later decisions. When a question is answered, this updates the belief $b$ and changes the expected reward of addressing the user's request with an action $a^{\mathrm{act}}$. This reward is computed as follows.

\textbf{Expected Goal Reward.} The expected goal reward of \textit{$a^{\mathrm{act}} \in \mathcal{A}^{\mathrm{act}}$} under belief $b$ is:
\begin{equation}
\bar{G}(s, b, a^{\mathrm{act}}) = \mathbb{E}_{\theta \sim b, s' \sim T(\cdot \mid s, a^{\mathrm{act}})} \big[ G(s', \theta) \big],
\label{eq:expected-goal-reward}
\end{equation}
where $s'$ is the state reached by executing $a^{\mathrm{act}}$ from $s$. In practice, the true goal reward $G$ is not available to the agent. Instead we use an internal LLM-based reward model $\hat G_{\textrm{LLM}}(s', \theta)$ to score $a^{\mathrm{act}}$ against each intent hypothesis $\theta^{k}$, and estimate $\bar G$ as the weighted average.

In REVOIR, we need to select between policy branches, not just actions. As such, we also need to compute the (cost-adjusted, immediate) reward of the \emph{act} policy $\pi^{\mathrm{act}}$:

\textbf{Net Reward of Act Policy.} The net immediate reward of an \emph{act} policy $\pi^{\mathrm{act}}(\cdot \mid s, b)$ is:
\begin{equation}
R^{\mathrm{act}}(s, b) = \mathbb{E}_{a^{\mathrm{act}} \sim \pi^{\mathrm{act}}(\cdot \mid s, b)} \big[ \bar{G}(s, b, a^{\mathrm{act}}) - C(a^{\mathrm{act}}) \big]
\label{eq:net-act-reward}
\end{equation}
For improved action selection, we use a best-of-$K$ policy for $\pi^{\mathrm{act}}$, proposing $K$ action candidates $\{a^{\mathrm{act},k}\}_{k=1}^K \sim \mathrm{LLM}(\cdot \mid s, b)$ from an LLM and selecting the best. In this case, the net reward of $\pi^{\mathrm{act}}$ reduces to the maximum cost-adjusted reward over $K$ proposed actions. 

In some settings, the assistant's action is terminal and provides no information. However, when the user controls termination and may follow an action $a^{\mathrm{act}}$ with a correction $m^{\mathrm{cor}} \in \mathcal{M}^{\mathrm{cor}}$, acting itself can result in valuable information:

\textbf{Value of Information for an Action.} The (one-step) value of information of an action $a^{\mathrm{act}} \in \mathcal{A}^{\mathrm{act}}$ is the expected net reward of acting again after a user's correction $m^{\mathrm{cor}}$:
\begin{equation}
\mathrm{VoI}(s, b, a^{\mathrm{act}}) = \mathbb{E}_{m^{\mathrm{cor}}} \big[R^{\mathrm{act}}(s', b') \big]
\label{eq:voi-act}
\end{equation}
Here, $s' \sim T(\cdot \mid (h_t, a^{\mathrm{act}}), m^{\mathrm{cor}})$ and $b'(\theta) \propto b(\theta)\, \pi^{\mathrm{user}}(m^{\mathrm{cor}} \mid (h_t, a^{\mathrm{act}}), \theta)$ are the updated state and belief after receiving $m^{\mathrm{cor}}$. The expectation is taken by sampling a user intent $\theta \sim b(\cdot)$ from $b$, simulating the user's correction $m^{\mathrm{cor}} \sim \pi^{\mathrm{user}}(\cdot \mid (h_t, a^{\mathrm{act}}), \theta)$ to the action $a^{\mathrm{act}}$, then updating the state $s'$ and belief $b'$. Note that $m^{\mathrm{cor}}$ can be a user message that accepts $a^{\mathrm{act}}$ and ends the conversation, s.t. $R^{\mathrm{act}}(s', b') = 0$.

With this, we can define the \emph{value function} (i.e. cumulative reward) for the \emph{act} policy $\pi^{\mathrm{act}}$, combining the net immediate reward with the expected future reward after a correction:

\textbf{Value of Act Policy.} The (one-step) cumulative reward of an \emph{act} policy $\pi^{\mathrm{act}}(\cdot \mid s, b)$ is:
\begin{equation}
V^{\mathrm{act}}(s, b) = \mathbb{E}_{a^{\mathrm{act}}, m^{\mathrm{cor}}} \big[ \mathrm{VoI}(s, b, a^{\mathrm{act}}) - C(m^{\mathrm{cor}}) \big] + R^{\mathrm{act}}(s, b)
\label{eq:value-of-acting}
\end{equation}
where $a^{\mathrm{act}} \sim \pi^{\mathrm{act}}(\cdot \mid s, b)$ and $m^{\mathrm{cor}}$ is the user's correction to $a^{\mathrm{act}}$. Note that $R^{\mathrm{act}}(s, b)$ is an expectation that accounts for the possibility that the user does not accept the action $a^{\mathrm{act}}$, and that $V^{\mathrm{act}}(s, b)$ reduces to $R^{\mathrm{act}}(s, b)$ when assistant actions $a^{\mathrm{act}}$ are terminal.

Having defined the value of acting, we can now quantify the VoI gained from asking a question, revising beliefs based on the answer, and acting based on the revised beliefs.

\textbf{Value of Information for a Question.} The (one-step) value of information of a question $m^{\mathrm{ask}} \in \mathcal{M}^{\mathrm{ask}}$ is the expected cumulative reward of acting once the answer $m^{\mathrm{ans}}$ is received:
\begin{equation}
\mathrm{VoI}(s, b, m^{\mathrm{ask}}) = \mathbb{E}_{m^{\mathrm{ans}}} \big[V^{\mathrm{act}}(s', b') \big]
\label{eq:voi-ask}
\end{equation}
Similar to above, the expectation is taken by sampling an intent $\theta \sim b(\cdot)$ from $b$, simulating the user's answer $m^{\mathrm{ans}} \sim \pi^{\mathrm{user}}(\cdot \mid (h_t, m^{\mathrm{ask}}), \theta)$ to the question $m^{\mathrm{ask}}$, then updating both $s'$ and $b'$. In Appendix~\ref{sec:voa}, we show that as long as we maximize over enough samples when acting, the VoI of a question is never worse than the net reward of acting immediately.

Finally, we compute the value function of the \emph{ask} policy $\pi^{\mathrm{ask}}$ as its (cost-adjusted) expected future reward, assuming that the agent acts immediately after receiving an answer:

\textbf{Value of Ask Policy.} The (one-step) cumulative reward of an \emph{ask} policy $\pi^{\mathrm{ask}}(\cdot \mid s, b)$ is:
\begin{equation}
V^{\mathrm{ask}}(s, b) = \mathbb{E}_{m^{\mathrm{ask}}, m^{\mathrm{ans}}} \big[ \mathrm{VoI}(s, b, m^{\mathrm{ask}}) - C(m^{\mathrm{ans}}) - C(m^{\mathrm{ask}}) \big]
\label{eq:value-of-asking}
\end{equation}
where $m^{\mathrm{ask}} \sim \pi^{\mathrm{ask}}(\cdot \mid s, b)$ and $m^{\mathrm{ans}}$ is the user's simulated answer to $m^{\mathrm{ask}}$. For improved performance, we use a best-of-$K$ policy for $\pi^{\mathrm{ask}}$, maximizing over $K$ questions sampled from an LLM prompted with a summary of the current belief $b$.

\subsection{Deciding Between Asking and Acting}
\label{sec:decision}
\enlargethispage{2\baselineskip}
With the value functions for $\pi^{\mathrm{ask}}$ and $\pi^{\mathrm{act}}$ defined, REVOIR asks questions if and only if $V^{\mathrm{ask}}$ (the cost-adjusted VoI of asking) is higher than the value $V^{\mathrm{act}}$ of acting immediately:
\begin{equation}
\pi^{\mathrm{REVOIR}}(\cdot \mid s, b) = 
\begin{cases}
\pi^{\mathrm{ask}}(\cdot \mid s, b), & V^{\mathrm{ask}}(s, b) > V^{\mathrm{act}}(s, b), \\
\pi^{\mathrm{act}}(\cdot \mid s, b), & \text{otherwise}.
\end{cases}
\end{equation}
We note that in practice, $V^{\mathrm{ask}}(s, b)$ and $V^{\mathrm{act}}(s, b)$ are expectations that we can only effectively estimate via sampling. With best-of-$K$ policies, this amounts to estimating $V^{\mathrm{ask}}(s, b)$ and $V^{\mathrm{act}}(s, b)$ with the value of the best question $m^{\mathrm{ask}}$ or action $m^{\mathrm{act}}$ among the $K$ samples.

\section{Experiments}

\subsection{Ambiguous Question Answering (CondAmbigQA)}\label{sec:condambigqa-exp}

Our first domain is ambiguous question-answering. We use CondAmbigQA~\citep{li_condambigqa_2025}, a dataset of 2,000 ambiguous queries drawn from AmbigNQ~\citep{min_ambigqa_2020}. Every query in the dataset is paired with a set of \emph{conditions}, each representing a possible latent intent $\theta \in \Theta$ behind the ambiguous query. The assistant's goal is to produce a final answer $m^{\mathrm{act}}$ that is semantically similar to the ground-truth answer $y_\theta$ to the user's true intent $\theta$.

\textbf{Benchmark Configuration.} Assistants interact with an LLM-based user simulator $\pi^{\mathrm{user}}$ which provides clarifications or accepts/corrects answers based on the user's true intent $\theta$. Following \citet{li_condambigqa_2025}, we evaluate the assistant's final answer $m^{\mathrm{act}}$ against the ground-truth answer $y_\theta$ using a rubric-based $\mathtt{AnswerCorrectness}(m^{\mathrm{act}}, y_\theta)$ metric with the G-Eval harness~\citep{liu_geval_2023}. This metric lies in $[0, 1]$, and uses an LLM judge (GPT-4o-mini) to inspect $m^{\mathrm{act}}$ for contradictions, omissions, relevance, and depth. We also record the number of clarifications (questions + corrections) as a measure of clarification \texttt{Effort}. We validate both the user simulator and $\mathtt{AnswerCorrectness}$ metric with human annotators (Appendix \ref{sec:pilot-human-replay}), finding high levels of agreement w.r.t. answer acceptance (72 out of 89) and answer quality (41 out of 57). We also test for robustness against an ``inattentive'' user simulator in Appendix \ref{sec:condambigqa_inatt_user}. Further details and validations are in Appendix \ref{sec:condambigqa_details}.

\textbf{Termination Variants.}
We instantiate two benchmark variants: With \emph{agent termination}, the assistant's $m^{\mathrm{act}}$ is terminal. With \emph{user termination}, the user corrects $m^{\mathrm{act}}$ if it fails to address their intent $\theta$; otherwise the user accepts the answer, ending the conversation.

\textbf{REVOIR.} REVOIR optimizes a proxy goal reward $\hat{G}_{\mathrm{LLM}}(m^{\mathrm{act}}, \hat \theta) \in [0, 1]$ which evaluates $m^{\mathrm{act}}$ against a hypothesized intent $\hat \theta$ similarly to $\mathtt{AnswerCorrectness}$, but \emph{without} knowing the true answer $y_\theta$ or using the G-Eval harness. We use a word-budgeted cost function $C(m^{\mathrm{user}}, m^{\mathrm{agent}})$ which assigns a cost of 0 to a (user/agent) message $m$ when its word count $|m|$ is less than the respective budget $B_{\mathrm{user/agent}}$, then scales linearly to $1$ as $|m|$ grows to twice the budget. The internal user model (distinct from the benchmark's user simulator), proxy goal reward, hypothesis proposer ($N$=$5$), hypothesis scorer, question generator ($K$=$5$) and answer generator ($K$=$1$) all use the same LLM backbone. See Appendix~\ref{sec:revoir-condambigqa} for details.

\textbf{Baselines.} We compare REVOIR against: \textbf{REIGN} (Rational Enquiry via Information-Gain Net-utility): A REVOIR variant configured to maximize expected information gain (EIG)~\citep{lindley_measure_1956, mackay_information_1992} instead of VoI (see Appendix~\ref{sec:eig}), balanced against message costs; \textbf{Direct Answer}: zero-shot, no clarification; \textbf{ReAct}~\citep{yao_react_2022}: chain-of-thought interleaved with actions, no explicit belief or cost accounting;
\finalonly{\pagebreak}
\textbf{SC-BoN-ReAct}: generates $K$ ReAct samples, decides to ask or act via majority vote (a.k.a. self-consistency, \citet{wang_selfconsistency_2023}), then selects the best of the $N \leq K$ majority samples with an LLM judge similar to REVOIR's proxy goal reward $\hat{G}_{\mathrm{LLM}}$;
\textbf{Entropy Thresholding}: same particle belief as REVOIR, but asks until $\frac{H_{\max}(b)-H(b)}{H_{\max}(b)} > \tau$, with $\tau$ fitted; \textbf{ReflectionDPO-ReAct}: our adaptation of ReflectionDPO~\citep{patel_adapt_2025} (Appendix~\ref{sec:reflection-dpo}), which finetunes a model to ask clarifying questions in order to better mimic an oracle; and \textbf{Oracle-ReAct}: ReAct privileged with ground-truth $y_\theta$, serving as a non-interactive oracle / topline.

\begin{figure}[t]
    \centering

    \begin{subfigure}{\finalifelse{0.48\linewidth}{0.425\linewidth}}
        \centering
        \includegraphics[width=\linewidth]{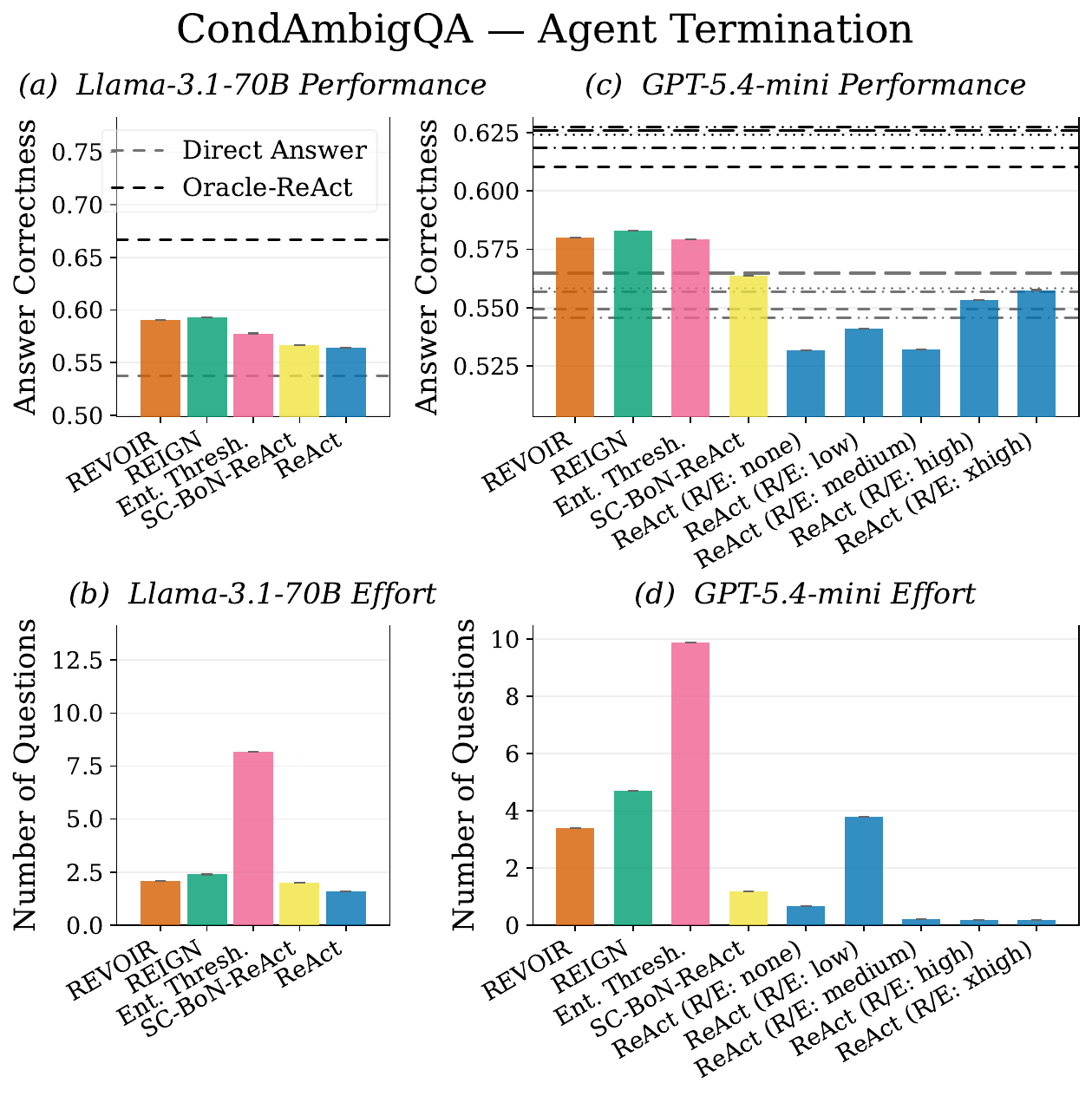}
    \end{subfigure}
    \begin{subfigure}{\finalifelse{0.48\linewidth}{0.425\linewidth}}
        \centering
        \includegraphics[width=\linewidth]{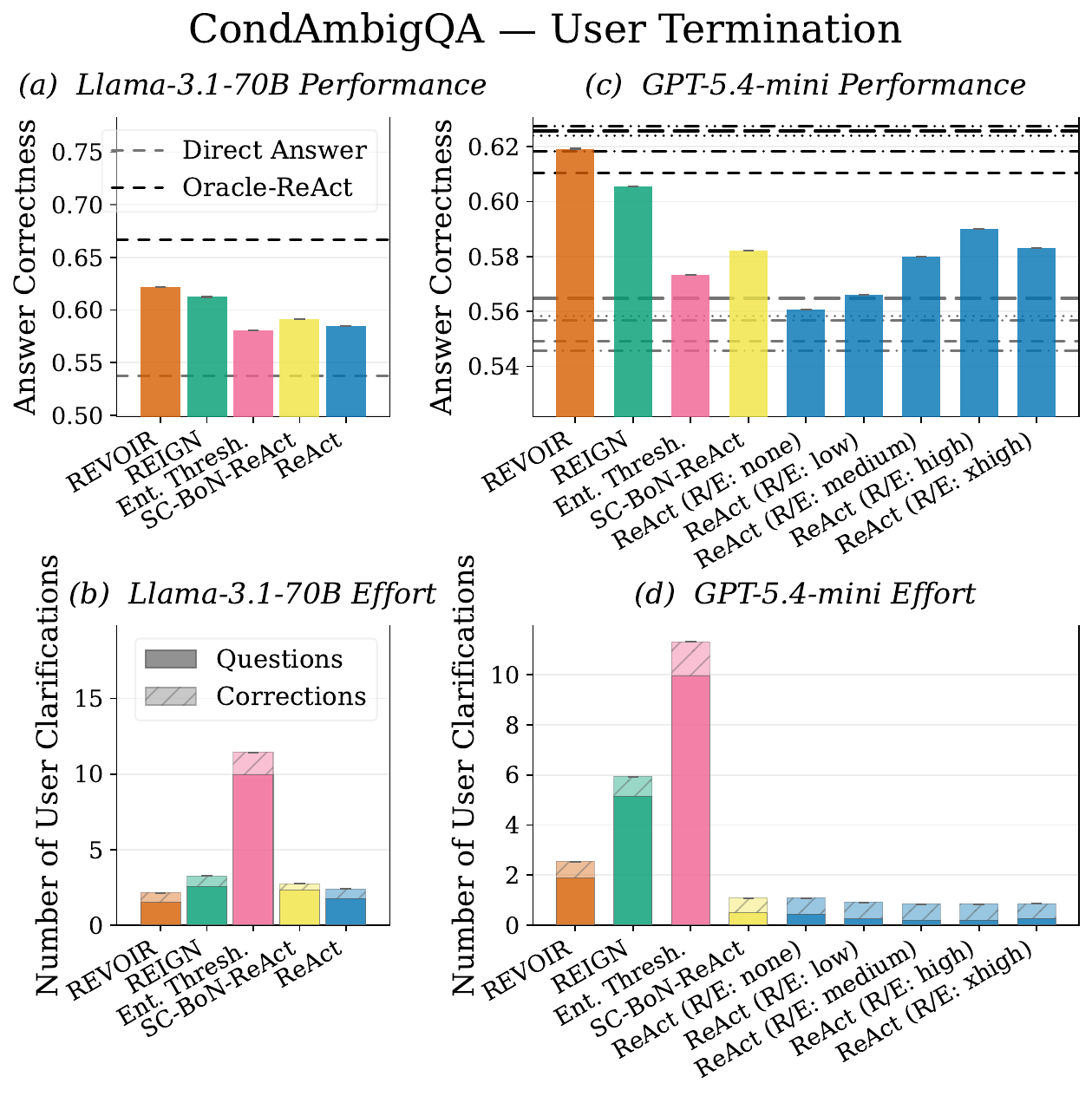}
    \end{subfigure}
    \vspace{-0.2em}
    \caption{\textbf{CondAmbigQA results for Llama-3.1-70B and GPT-5.4-mini}. ReAct variants and non-interactive base/top lines are evaluated at five reasoning-effort (R/E) levels for GPT-5.4-mini. Horizontal lines correspond to baselines and toplines across R/E levels: \textbf{\textcolor[HTML]{757575}{gray lines}} represent the \textbf{Direct Answer} baselines, and \textbf{black lines} represent \textbf{Oracle-ReAct} toplines, with line styles indicating effort levels: {\scriptsize$(-\ \cdot\ \cdot\ -)$} for \texttt{none}, {\scriptsize$(\mbox{---}\ \mbox{---})$} for \texttt{low}, {\scriptsize$(-\ \cdot\ -)$} for \texttt{medium}, {\scriptsize$(\cdot\ \cdot\ \cdot)$} for \texttt{high}, and {\scriptsize$(-\ \ -)$} for \texttt{xhigh}. REVOIR and REIGN use budgets of $(B_{\mathrm{agent}}, B_{\mathrm{user}}) = (100, 50)$, and SC-BoN-ReAct uses $K = 5$ samples.}
    \label{fig:condambigqa_llama3.1_70b_gpt54_mini}
    \finalifelse{}{\vspace{-0.6em}}
\end{figure}

\subsubsection{Experiment Results}

\textbf{REVOIR achieves high correctness while minimizing clarification effort.} Figure \ref{fig:condambigqa_llama3.1_70b_gpt54_mini} compares REVOIR against baselines using Llama-3.1-70B, and GPT-5.4-mini as backbones. Under \emph{agent termination}, REVOIR beats all ReAct variants in correctness. REIGN and Entropy Thresholding are close in correctness, with slight edge for REIGN, but at the cost of significantly more questions than REVOIR (1.3--2.8$\times$ for GPT-5.4-mini). REVOIR's advantage is even sharper under \emph{user termination}, outperforming all baselines in correctness while using the least clarifications of non-ReAct methods. ReAct variants clarify very little in all cases, but at a substantial cost to correctness. We find similar results for Gemini 3.5 Flash-Lite and Claude Haiku 4.5 in Appendix \ref{sec:condambigqa_gemini_claude}, albeit with model-specific idiosyncrasies. Appendix \ref{sec:condambigqa_inatt_user} also shows REVOIR's robustness to having an inaccurate user model.

\begin{figure}[t]
    \centering
    \includegraphics[width=\finalifelse{0.925\linewidth}{0.87\linewidth}]{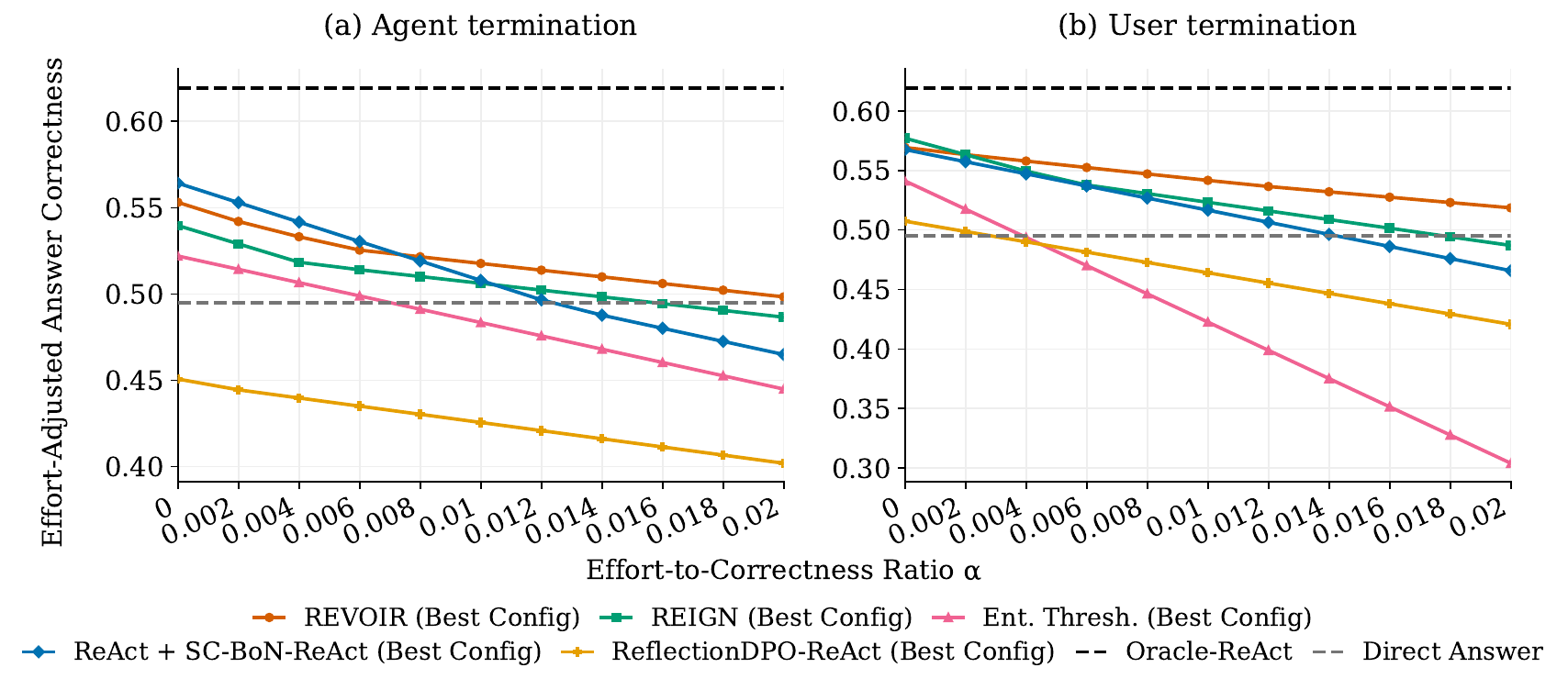}
    \vspace{-0.2em}
    \caption{\textbf{Effort-adjusted answer correctness on CondAmbigQA} for Llama-3.1-8B as a function of the effort-to-correctness ratio $\alpha$. The $y$-axis plots $\texttt{AnswerCorrectness} - \alpha \cdot \texttt{Effort}$ for the best method configuration at each value of $\alpha$. \texttt{Effort} is the number of clarifications per conversation. REVOIR dominates other methods for most values of $\alpha$ under both (a) agent termination ($\alpha > 0.008$) and (b) user termination ($\alpha > 0.002$).}
    \label{fig:condambigqa_llama3.1-8b_alpha_util_joint}
\end{figure}

\begin{figure}[t]
    \finalifelse{}{\vspace{-1em}}
    \centering
    \includegraphics[width=\finalifelse{0.925\linewidth}{0.85\linewidth}]{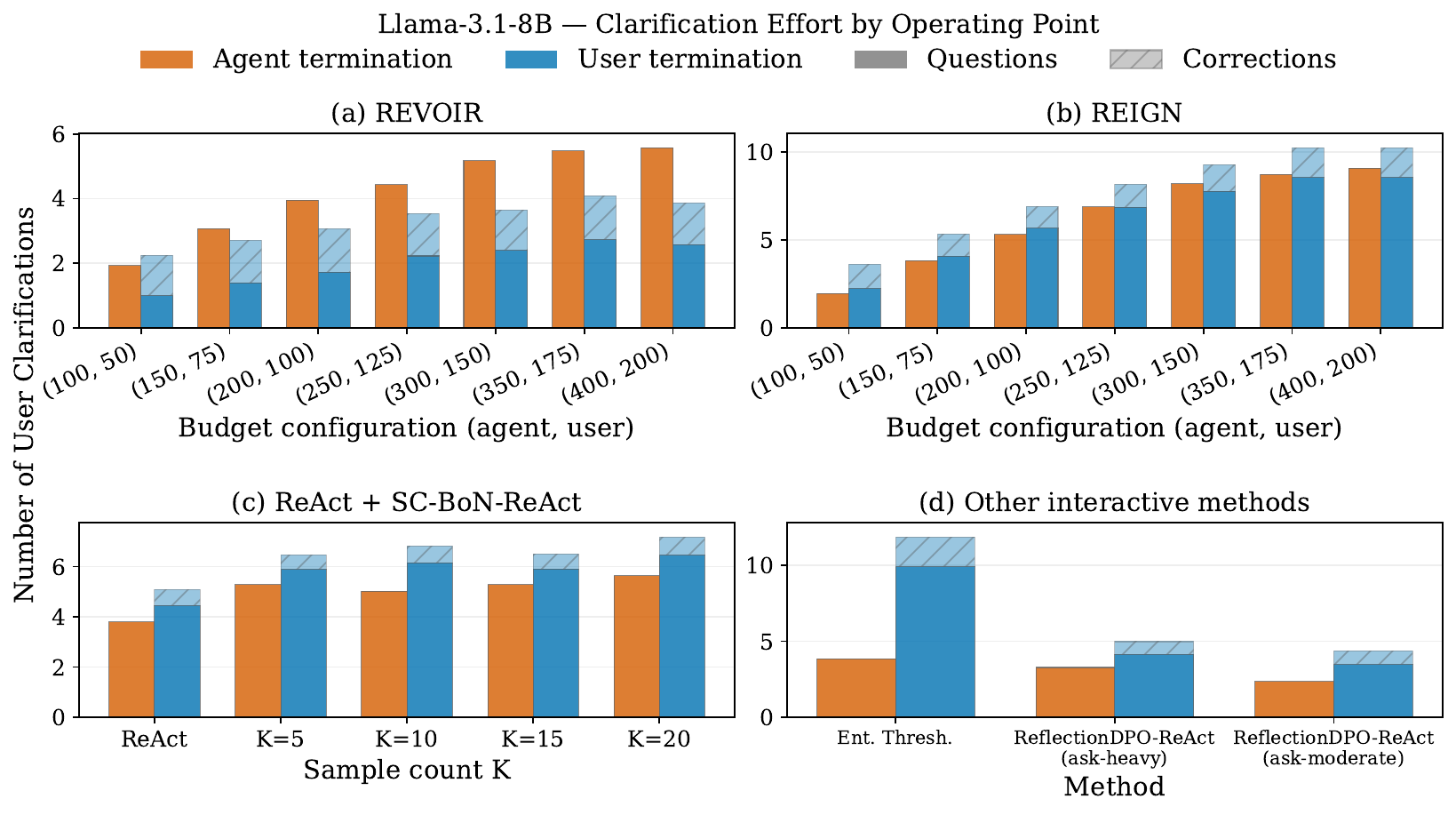}
    \vspace{-0.6em}
    \caption{\textbf{Clarification effort on CondAmbigQA for Llama-3.1-8B} across termination variants and methods. When moving from \emph{agent termination} to \emph{user termination}, only REVOIR adapts to using less clarifications in total by relying on user corrections.}
    \label{fig:condambigqa_llama3.1-8b_effort_joint}
    \finalifelse{}{\vspace{-0.3em}}
\end{figure}
\textbf{REVOIR best trades-off correctness vs. clarification effort} for most trade-off ratios. To study how well REVOIR adapts to different preferences about correctness vs. clarification effort, we run REVOIR across many budget points $(B_{\mathrm{agent}}, B_{\mathrm{user}})$, using a smaller Llama-3.1-8B model due to compute costs. We then compute the \emph{effort-adjusted answer correctness} $= \texttt{AnswerCorrectness} - \alpha \cdot \texttt{Effort}$ as a measure of overall utility, where $\alpha$ is the effort-to-correctness trade-off ratio, and find the REVOIR budget that maximizes this metric at each value of $\alpha$. Figure \ref{fig:condambigqa_llama3.1-8b_alpha_util_joint} plots this metric across a range of trade-off ratios, comparing REVOIR at each $\alpha$ against the best configuration for each baseline (across REIGN budgets, SC-BoN-ReAct sample sizes $K \in \{1, 5, 10, 15, 20\}$, ReflectionDPO data mixtures; see Tables \ref{tab:condambigqa_llama3.1-8b_at_details}--\ref{tab:condambigqa_llama3.1-8b_ut_details}). REVOIR outperforms baselines across most trade-off ratios $\alpha$ (agent term.: $\alpha > 0.008$, user term.: $\alpha > 0.002$), and is second best even at $\alpha=0$ (i.e. cost-free clarifications), demonstrating that REVOIR is the most adaptable of the methods to varying trade-offs. %

\textbf{REVOIR rationally adapts to the provision of user corrections.} Figure \ref{fig:condambigqa_llama3.1-8b_effort_joint} shows the clarification effort for each method on Llama-3.1-8B across termination variants. Only REVOIR adapts rationally when moving from \emph{agent termination} to \emph{user termination}; since user corrections provide information, REVOIR not only asks less questions, but makes less clarifications overall. All other methods increase in both questions and clarifications.

\textbf{Generic inference-time scaling does not consistently improve performance or rational information-gathering.} In Figure \ref{fig:condambigqa_llama3.1_70b_gpt54_mini}, increasing GPT-5.4-mini ReAct's reasoning effort yields non-monotonic and limited gains. Notably, higher reasoning effort (above \texttt{low}) results in very few clarifications. We find similar trends for other reasoning models (Section \ref{sec:condambigqa_gemini_claude}), indicating that reasoning models do not spend extra reasoning tokens on whether more clarification is helpful to answer the user's query; instead tokens appear to be spent on directly formulating a better answer. We separately find that SC-BoN-ReAct does not consistently improve with more samples $K$, and only exceeds medium-to-high budget REVOIR by using up to 2--4$\times$ more tokens (Tables~\ref{tab:token-usage} \& \ref{tab:condambigqa_llama3.1-8b_at_details}). These results suggest that neither reasoning effort nor parallel sampling substitutes for principled VoI reasoning. 

\subsection{Preference-Aligned Household Assistance (ADAPT)}
\begin{figure}[t]
    \centering
    \vspace{-8pt}
    \includegraphics[width=\linewidth]{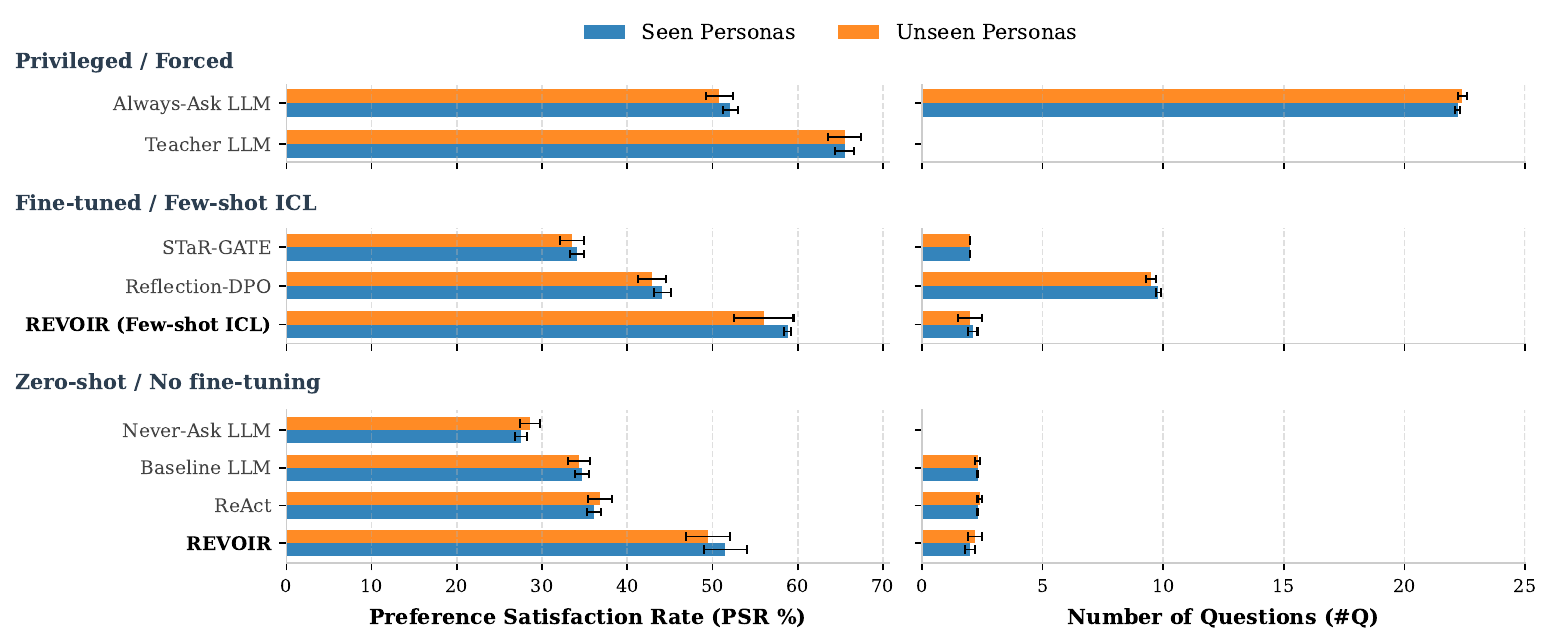}
    \vspace{-1.6em}
    \caption{\textbf{ADAPT preference satisfaction rate and question count across methods} (4-fold cross-validated following ADAPT splits in \citealt{patel_adapt_2025}). The standard deviation is computed across split means. Baseline results are reproduced from Table~1 in \citet{patel_adapt_2025}. For ICL variants, all seen personas are provided in context during belief updating. More details are in Table~\ref{tab:adapt}.}
    \label{fig:adapt}
    \finalifelse{}{\vspace{-0.5em}}
\end{figure}

Our second domain is preference-aligned household task planning. We build on the ADAPT benchmark~\citep{patel_adapt_2025}, where an agent must complete long-horizon cooking tasks (e.g., ``prepare an omelet for breakfast'') while respecting a user's hidden preference set $\theta = \{\gamma_1, \dots, \gamma_M\}$, with preferences like ``use non-dairy milk'' or ``serve beverages first.'' 

\textbf{Benchmark Description.}
The agent operates in a grounded text-based simulation with over 230 objects. %
At each turn $t$ it may take an external action $a_t \in \mathcal{X}$ (e.g. physical manipulation) or send a question $m^{\mathrm{ask}}$ to the user. %
The user simulator $\pi^{\mathrm{user}}(\cdot \mid h_t, \theta)$ answers questions based on the true preferences $\theta$, following the original benchmark exactly.
Episodes are scored by the \emph{preference satisfaction rate} $\mathrm{PSR} = p^+/(p^+ + p^-) \in [0,1],$ where $p^+$ and $p^-$ count the preferences in $\theta$ that are satisfied or not in the final state $s_t$. 

\textbf{Methods \& Baselines.}
REVOIR is evaluated on a two-stage variant of ADAPT where all questions are asked before committing to a physical plan (Appendix~\ref{sec:adapt-2stage}); this disadvantages it relative to the baselines, which interleave asking and acting. REVOIR uses PSR as the goal reward $G(s, \hat\theta) = \mathrm{PSR} \cdot \mathbf{1}[s \in \mathcal{S}^{\mathrm{term}}]$, where $\hat\theta$ is inferred. To match \citet{patel_adapt_2025}, we also use Llama-3.1-70B-Instruct as our LLM, assume agent termination, and use a very low question cost $C(m^{\textrm{ask}}) = 5 \times 10^{-4}$ since the original baselines are not cost-aware.

We compare against finetuning methods (ReflectionDPO \citep{patel_adapt_2025}; STaR-GATE \citep{andukuri_stargate_2024}), prompting baselines (Never-Ask; Baseline [$\equiv$ Direct Answer]; ReAct), a Teacher oracle with the true $\theta$, and an Always-Ask baseline. While REVOIR requires no finetuning on user personas $\theta$ (unlike ReflectionDPO), we still evaluate on the benchmark's \emph{seen} and held-out (\emph{unseen}) persona splits, and test both a \emph{few-shot} variant (seen personas are in-context examples for belief updating) and a \emph{zero-shot} variant (no examples). Full details, along with belief-update variants, can be found in Appendices \ref{sec:adapt_details}--\ref{sec:adapt_hyperparameters}.

\subsubsection{Experiment Results}
Figure~\ref{fig:adapt} shows REVOIR with the best belief updating configuration (factored beliefs + incremental updating). In zero-shot, REVOIR achieves $51.5\%$ PSR on seen personas and $49.5\%$ PSR on unseen personas, virtually matching Always-Ask with $11\times$ fewer questions. With ICL, REVOIR reaches $58.8\%$ on seen personas and $56.0\%$ on unseen personas, exceeding Always-Ask by $6$ points and outperforming the fully-trained Reflection-DPO ($44\%$) by $13$--$15$ points with $5\times$ fewer questions. Overall, REVOIR outperforms all finetuning and prompting baselines. More results are in Appendix~\ref{sec:adapt_details}, showing that factored belief updating helps substantially, but that joint beliefs still do well with few-shot examples.

\section{Conclusion}
In this paper, we introduced REVOIR, an inference-time algorithm for rational clarification via VoI reasoning. Across ambiguous question-answering and preference-aligned task planning, REVOIR achieves a better accuracy-efficiency frontier than alternatives, adapts naturally to user-correctable settings, and surpasses strong fine-tuning baselines while asking substantially fewer questions. However, several limitations point towards future work: REVOIR's one-step VoI estimate is tractable but myopic, suggesting the need to develop efficient multi-step estimates of VoI. Second, REVOIR is currently not designed for dual control or full interleaving of questions and non-terminal actions; addressing this will require novel ways to efficiently estimate VoI under interleaving. Finally, while REVOIR is reasonably robust to imperfect user modeling (Appendix \ref{sec:condambigqa_inatt_user}), evaluating with real users and their idiosyncrasies is an important next step toward deployment.

\finalifelse{}{\clearpage}
\subsection*{AI use statement}

In this work, we used generative AI tools for augmenting a dataset (Section~\ref{sec:condambigqa_benchmark}), implementing parts of the presented methods, and implementing parts of the result visualization and interpretation.
We have not used generative AI tools for generating synthetic data sets, helping develop theoretical models or conceptual frameworks, proposing or refining hypotheses, designing or providing feedback on research methodology or experiments, assisting with translation, or cleaning datasets.
Formulating mathematical claims played a minimal role in this work, and AI assistance was used only for extending known results with elementary techniques, formatting, and restatement.
Additionally, we used generative AI tools for writing suggestions in parts of the manuscript. 
We have reviewed all AI-assisted work.  We take responsibility for the final content of this work, including text, claims or artifacts produced with the aid of generative AI.

\finalonly{
    \section*{Acknowledgement}
        This research is supported by the Ministry of Digital Development and Information (MDDI) under the Singapore Global AI Visiting Professorship Program (Award No. AIVP-2024-002), the NUS Presidential Young Professorship grant to Tan Zhi-Xuan, and the National Research Foundation, Singapore under its AI Singapore Programme (AISG Award No: AISG3-PhD 2023-08-053).
}

\finalifelse{}{\clearpage\newpage}
\printbibliography

\clearpage
\newpage
\appendix
\counterwithin{figure}{section}
\counterwithin{table}{section}

\section*{APPENDIX}

\section{Value of Adaptivity}
\label{sec:voa}

Here, we develop on a well-known result about why acting after gaining new information leads to better outcomes~\citep{howard_information_1966}. By employing a best-of-$K$ approach for the $\pi^{\mathrm{act}}$ sub-policy, $R^{\mathrm{act}}(s',b')$ effectively maximizes the cost-adjusted reward over candidate actions conditioned on the updated belief $b'$. Because the maximum of a set of linear functions is convex, Jensen's inequality guarantees that the expected value of this maximum across all possible user answers is greater than or equal to the maximum reward obtainable under the current belief, provided the same candidate set is available before and after clarification.

For the following results, we consider terminal act decisions and questions that change only information, leaving the feasible actions and the utility of each fixed action unchanged.

\begin{proposition}[Value of Adaptivity]
\label{prop:voa-adaptivity}
For any fixed nonempty finite action set $\mathcal{A}$,
\begin{equation}
\mathbb{E}_{{m}^{\mathrm{ans}}} \left[
 \max_{a^{\mathrm{act}} \in \mathcal{A}} \big(\bar{G}(s,b',a^{\mathrm{act}}) - C(a^{\mathrm{act}})\big)
\right]
\geq
\max_{a^{\mathrm{act}} \in \mathcal{A}} \big(\bar{G}(s,b,a^{\mathrm{act}}) - C(a^{\mathrm{act}})\big).
\label{eq:voa}
\end{equation}
\end{proposition}

\begin{proof}[Proof sketch.]
Under Bayesian updating, with answers drawn from the corresponding predictive distribution, $\mathbb{E}_{m^{\mathrm{ans}}}[b']=b$ by the law of total expectation.
For each fixed action, $\bar{G}(s,b,a^{\mathrm{act}})-C(a^{\mathrm{act}})$ is linear in $b$, so its maximum over $\mathcal{A}$ is convex. Jensen's inequality gives the
result. The inequality can be strict because the maximizing action may differ across the realized answers.
\end{proof}

Thus, for a policy that maximizes over this set, $\mathrm{VoI}(s,b,m^{\mathrm{ask}}) - R^{\mathrm{act}}(s,b)$ (and similarly, $\textrm{VoI}(s, b, m^\mathrm{act}) - R^\mathrm{act}(s', b')$) is non-negative, because more information affords the agent the opportunity to pivot to a higher-reward action once its uncertainty is reduced.

When $\pi^{\mathrm{act}}$ does not adapt to the updated belief $b'$ (i.e., $\pi^{\mathrm{act}}(\cdot\mid b)=\pi^{\mathrm{act}}(\cdot\mid b')$), the expected future goal reward of a fixed action, $\mathbb{E}_{m^{\mathrm{ans}}}[\bar{G}(s,b',a^{\mathrm{act}})]$, collapses back to the current expected reward $\bar{G}(s,b,a^{\mathrm{act}})$ by the law of total expectation ($\mathbb{E}_{m^{\mathrm{ans}}}[b']=b$). Gathering more information thus provides no benefit to non-adaptive policies. 

In practice, REVOIR generates candidates conditioned on its current belief $b$ from some policy $\pi_0(a^{\mathrm{act}} \mid b)$, so their distribution can change after the belief updates to $b'$. In such cases, under mild assumptions about $\pi_0(a^{\mathrm{act}} \mid b')$, best-of-$K$ sampling after clarification can outperform even exact maximization before clarification when the expected benefit of the answer $m^{\mathrm{ans}}$ exceeds the sampling error.

\begin{proposition}[Value of Adaptivity with Belief-Conditioned Candidates]
\label{prop:voa-sampled}
Let $R_*^{\mathrm{act}}(s,b)$ be the supremum of $\bar G(s,b,a^{\mathrm{act}})-C(a^{\mathrm{act}})$ over feasible actions. Suppose net rewards lie in a common interval of width $D$. For almost every answer, assume each of $K$ independent draws from the updated proposal $\pi_0(\cdot \mid b')$ lies within $\varepsilon\in[0,D]$ of the optimal value under $b'$ with probability at least $1-\delta$, where $\delta\in[0,1)$. Then the best-of-$K$ act policy satisfies
\begin{equation}
\mathbb{E}_{{m}^{\mathrm{ans}}}[R^{\mathrm{act}}(s',b')]
\geq
\mathbb{E}_{{m}^{\mathrm{ans}}}[R_*^{\mathrm{act}}(s',b')] - \varepsilon - (D - \varepsilon)\delta^K.
\label{eq:voa-sampled}
\end{equation}
Hence, if the expected gain in optimal reward $\mathbb{E}[R_*^{\mathrm{act}}(s',b')]-R_*^{\mathrm{act}}(s,b)$ exceeds $\varepsilon + (D - \varepsilon)\delta^K$, then
\begin{equation}
\mathbb{E}_{{m}^{\mathrm{ans}}}[R^{\mathrm{act}}(s',b')] > R_*^{\mathrm{act}}(s,b) \geq R^{\mathrm{act}}(s,b).
\end{equation}
In particular, $\delta^K \to 0$ as $K \to \infty$, so this inequality holds for sufficiently large $K$ whenever the expected gain in optimal reward exceeds $\varepsilon$ (i.e. whenever Eq. \ref{eq:voa} is strict with an $\varepsilon$ gap). For $K = 1$, the expected gain needs to exceed $\varepsilon + (D - \varepsilon)\delta$.
\end{proposition}

\begin{proof}[Proof sketch.]
Condition on the user's answer. The probability that all $K$ candidates miss the $\varepsilon$-optimal set is at most $\delta^K$. The selected action's shortfall from optimal is at most $\varepsilon$ otherwise and at most $D$ on a miss. Its expected shortfall is therefore at most $\varepsilon + (D - \varepsilon)\delta^K$. Averaging over answers gives Inequality~\eqref{eq:voa-sampled} and comparisons complete the proof.
\end{proof}

\section{Recovering Expected Information Gain}
\label{sec:eig}

REVOIR can be modified to follow an EIG maximization strategy as follows: If the goal reward in Section~\ref{sec:voi-reasoning} is replaced with the log belief-probability $G(s, \theta) = \log b(\theta)$, then the expected goal reward $\bar{G}(s, b, a^{\mathrm{act}})$ at belief $b$ is just the (negative) entropy $-H(b)$, and $\mathrm{VoI}(s, b, m^{\mathrm{ask}})$ reduces to the (negative) expected entropy. If costs are also zero and acting gives no information, then the value difference between asking and acting is precisely the expected information gain $H(b) - H(b \mid m^{\mathrm{ask}})$ from asking. EIG-based clarification can thus be viewed as a special case of REVOIR in which task reward is replaced with pure uncertainty reduction.

\section{Extended Details on CondAmbigQA}
\label{sec:condambigqa_details}

\subsection{Benchmark Configuration}
\label{sec:condambigqa_benchmark}

\textbf{Dataset.} CondAmbigQA~\citep{li_condambigqa_2025} contains 2,000 ambiguous queries drawn from AmbigNQ~\citep{min_ambigqa_2020}, where each question is endowed with a set of explicit contextual constraints (called \textit{conditions}) that approximate the assumptions underlying different valid interpretations (information-seeking intents) of the query. Each condition corresponds to an intent $\theta \in \Theta$ in our framework, and is supported by a set of retrieved Wikipedia fragments that serve as its source. Ground-truth answers $y \in \mathcal{Y}$ are extracted from Wikipedia fragments, and were generated via a human-LLM collaborative annotation process.

\begin{wraptable}[7]{r}{0.48\textwidth}
    \vspace{-10pt}
    \centering
    \footnotesize
    \caption{CondAmbigQA Data Splits}
    \begin{center}
    \begin{tabular}{lcc}
    \toprule
    \textbf{Split} & \textbf{Examples} & \textbf{Unique Questions} \\
    \midrule
    Train & 3,062 & 1,600 \\
    Dev   & 380   & 200   \\
    Test  & 380   & 200   \\
    \bottomrule
    \end{tabular}
    \end{center}
    \label{tab:condambigqa_splits}
\end{wraptable}

We augment each (question, condition) pair with an \textit{explicit disambiguating question}---a question whose answer uniquely identifies the scoping condition $\theta$. These are generated by GPT-4o-2024-08-06 prompted with the ambiguous question and the condition text. The explicit question serves two purposes: it is provided to the user simulator as a reference for what information should be revealed when answering clarification requests, and it serves as the oracle reflection question for our ReflectionDPO baseline (Appendix~\ref{sec:reflection-dpo}). The augmented dataset is partitioned at the question level into train, validation, and test splits to ensure that no condition of a question seen during training appears in evaluation; split statistics are reported in Table~\ref{tab:condambigqa_splits}.

\textbf{User Simulator.} We simulate the user by prompting an LLM (GPT-4o-mini) with the ground-truth condition $\theta$. At each turn, the user either responds to the assistant's clarification request $m^{\mathrm{ask}}$ based on the condition, or produces a dismissal if the inquiry cannot be answered (Templates \ref{prt:condambigqa-user-system}--\ref{prt:condambigqa-user-message}). In the user termination setting, the user simulator also decides whether to accept the assistant's answer $m^{\mathrm{act}}$ (Templates \ref{prt:condambigqa-user-term-system}--\ref{prt:condambigqa-user-term-message}) or respond with a correction $m^{\mathrm{cor}}$ (Templates \ref{prt:condambigqa-user-corr-system}--\ref{prt:condambigqa-user-corr-message}). Note that REVOIR has \emph{no access} to the true user simulator $\pi^{\mathrm{user}}$, and instead uses its own proxy user model $\hat \pi^{\mathrm{user}}$ with a different LLM backbone. We validate the user simulator against human annotators in Appendix \ref{sec:pilot-human-replay} in terms of answer acceptance. We also test for robustness against an ``inattentive'' user simulator in Appendix \ref{sec:condambigqa_user_simulator}.

\textbf{Evaluation Metric.} We evaluate the quality of the assistant's final answer $m^{\mathrm{act}}$ against the ground-truth answer $y_\theta$ for the true intent $\theta$ using an LLM-based judge (Template \ref{prt:condambigqa-llmjudge-geval}).  Following \citet{li_condambigqa_2025}, we use a rubric-based $\mathtt{AnswerCorrectness}$ metric implemented in the G-Eval framework~\citep{liu_geval_2023}. The rubric performs a \emph{contradiction check}, inspecting whether the answer contradicts established facts in $y_\theta$; an \emph{omission penalty}, heavily penalizing omission of critical details in $y_\theta$; and a \emph{relevance check}, ensuring the answer directly and concisely addresses the question. Since our assistants are not provided with access to web-retrieved resources for the user's query
(unlike in \citet{li_condambigqa_2025}), we add an \emph{acumen reward} that rewards relevant answers presented with more depth and detail than $y_\theta$, thereby avoiding over-penalization of answers of higher quality than the reference answer.

We verify that our $\mathtt{AnswerCorrectness}$ metric is discriminative and rank-consistent by comparing it to topline and baseline anchors: \textbf{Oracle}, which submits the gold answers $y_\theta$ verbatim, scores $0.995$ out of 1; \textbf{Oracle-ReAct}, an agent shown $y_\theta$ but answering in its own words, scores $0.619$ for Llama-3.1-8B and $0.63$ for GPT-5.4-mini). \textbf{Direct Answer}, a non-interactive baseline that answers the user's query without reasoning, achieves around $0.50$--$0.56$. Our metric thus orders the tiers as expected, and our interactive methods score between these topline and baseline anchors.

\subsection{REVOIR Implementation in CondAmbigQA}
\label{sec:revoir-condambigqa}

\textbf{Goal Reward.} The proxy goal reward $\hat G_\textrm{LLM}(m^\mathrm{act}, \theta^n)$ scores an answer $m^\mathrm{act}$ against each hypothesis $\theta^n$ using criteria analogous to the G-Eval \texttt{AnswerCorrectness} metric used by the benchmark, but omitting the contradiction check since no ground-truth answer $y_\theta$ is available at planning time (Templates \ref{prt:condambigqa-proxy-reward-system}--\ref{prt:condambigqa-proxy-reward-message}):
\begin{enumerate}[leftmargin=*,itemsep=2pt,topsep=0pt]
    \item Check whether the predicted answer addresses the asking intent implied by $\theta^n$.
    \item Heavily penalize omission of information-seeking targets implied by the asking intent.
    \item Ensure the answer addresses the intended question $\theta^n$ without irrelevant information.
\end{enumerate}
The expected reward is then $\bar{G}(s, \hat{b}_t, m^{\mathrm{act}}) = \sum_n w^n \cdot \hat{G}_{\mathrm{LLM}}(m^{\mathrm{act}}, \theta^n)$. 

\textbf{Cost Function.} We assume that users specify budgets $B_{\mathrm{user}}$ and $B_{\mathrm{agent}}$ for how many words they are willing to (respectively) write and read per message:
$$C(a^{\mathrm{user}}_t, a^{\mathrm{agent}}_t) = \textstyle\frac{1}{B_{\mathrm{user}}} \cdot \big[ |a^{\mathrm{user}}_t| - B_{\mathrm{user}} \big]_+ + \textstyle\frac{1}{B_{\mathrm{agent}}} \cdot \big[ |a^{\mathrm{agent}}_t| - B_{\mathrm{agent}} \big]_+,$$
where $|m|$ is the word count of $m$ and $[x]_+ = \max(x, 0)$. A message $m$ has zero cost when $|m|$ is less than the respective budget, and costs within $[0, 1]$ while less than twice the word budget. We assume costs are mostly incurred by the \emph{user}, and that token inference costs are marginal, in line with LLM product trends where increasingly many tokens are spent to deliver user value.

\textbf{Belief Updating.} At each turn, REVOIR maintains a particle belief $\hat{b}_t = \{(\theta^n, w^n)\}_{n=1}^N$ over latent intents $\theta^n$ behind the ambiguous query as follows: $N = 5$ intent hypotheses are sampled from a hypothesis proposer LLM $Q_{\textrm{LLM}}(\theta | h_t, \hat{b}_{t-1})$ provided with the history $h_t$ and past belief $\hat{b}_{t-1}$ (if available). This LLM is prompted to enumerate $N$ distinct hypotheses $\theta^n$ (Templates \ref{prt:condambigqa-hyp-system}--\ref{prt:condambigqa-hyp-message}), and may keep, or prune hypotheses present in the past belief $\hat{b}_{t-1}$.

An LLM-based hypothesis scorer $S_{\mathrm{LLM}}$ then computes a  weight $w^n_t = S_{\mathrm{LLM}}(\theta^n, h_t)$ for each hypothesis $\theta^n$, reflecting how consistent $\theta^n$ is with the history $h_t$. The LLM is prompted to output a score between 0 and 10 (Templates \ref{prt:condambigqa-consistency-score-system}--\ref{prt:condambigqa-consistency-score-message}), and this is then normalized to lie within $[0, 1]$ s.t. $\sum_{n=1}^N w^n = 1$. Weights are recomputed from scratch at every step $t$ (batch rescoring), rather than being incrementally updated.

\textbf{Internal User Model.} REVOIR uses an internal user model $\hat \pi^{\mathrm{user}}(\cdot | h_t, \theta^n)$ to forecast how a user will respond to the assistant given a hypothesized intent $\theta^n$. We use different prompts for user answers $m^{\mathrm{ans}}$ to clarifying questions $m^{\mathrm{act}}$ (Templates \ref{prt:condambigqa-forecast-system}--\ref{prt:condambigqa-forecast-message}) and user corrections $m^{\mathrm{cor}}$ to assistant actions $m^{\mathrm{act}}$ (Templates \ref{prt:condambigqa-user-corr-system}--\ref{prt:condambigqa-user-corr-message}). $\hat \pi^{\mathrm{user}}$ differs from the true user simulator $\pi^{\mathrm{user}}$ in several ways: (i) $\hat \pi^{\mathrm{user}}$ has no access to the true intent $\theta$; (ii) $\hat \pi^{\mathrm{user}}$ uses a different LLM (e.g. Llama-3.1-8B) than $\pi^{\mathrm{user}}$ (GPT-4o-mini); (iii) different prompts are used. In Section \ref{sec:condambigqa_inatt_user}, we test REVOIR's robustness to user model misspecification by prompting the external simulator $\pi^{\mathrm{user}}$ to mimic an inattentive user while keeping $\hat \pi^{\mathrm{user}}$ the same.

\textbf{Question Generation and Selection.} We use a best-of-$K$ \emph{ask} policy $\pi^{\mathrm{ask}}$ with $K = 5$ question candidates, which are proposed from a \emph{question generator}  $\pi_0^{\mathrm{ask}}(\cdot | \hat b_t)$ prompted with the particle belief $\hat b_t$ (Templates \ref{prt:condambigqa-qg-system}--\ref{prt:condambigqa-qg-message}). We then select among these $K$ questions by maximizing the cost-adjusted VoI (the expectand in Eq. \ref{eq:value-of-asking}). The VoI of each question $m^{\mathrm{ask}}$ is computed per Eq. \ref{eq:voi-ask} by simulating 1 user answer $m^{\mathrm{ans}}$ per hypothesis $\theta^n$ in the belief $\hat b_t$, and taking the sample average of the expectand.

\textbf{Answer Generation and Evaluation.} For efficiency, the \emph{act} policy $\pi^{\mathrm{act}}$ samples only $K = 1$ answer candidate $m^\mathrm{act}$ from an \emph{answer generator}  $\pi_0^{\mathrm{act}}(\cdot | \hat b_t) \equiv \pi^{\mathrm{act}}(\cdot | \hat b_t)$ prompted with the particle belief $\hat b_t$ (Templates \ref{prt:condambigqa-answer-system}--\ref{prt:condambigqa-answer-message}). Despite not maximizing over many candidates, we note that this policy is still \emph{adaptive} in the sense of Appendix~\ref{sec:voa} to the extent that the LLM generates better answers when it is prompted with an updated belief $b'$. Proposition~\ref{prop:voa-sampled} provides a sufficient condition for belief-adaptivity that results in positive value of information, including when $K = 1$. Thus, explicit selection among multiple candidates is not necessary for clarification to improve downstream decisions.

With one answer candidate $m^{\mathrm{act}}$, the value of acting $V^{\mathrm{act}}$ (Eq. \ref{eq:value-of-acting}) is estimated by the expected cumulative reward (or $Q$-value) of $m^{\mathrm{act}}$:
\begin{align*}
    Q^{\mathrm{act}}(s, b, m^{\mathrm{act}}) &= R^{\mathrm{act}}(s, b, m^{\mathrm{act}}) + \mathbb{E}_{\theta \sim b, m^{\mathrm{cor}}}\big[\, R^{\mathrm{act}}(s', b') - C(m^{\mathrm{cor}}) \,\big]\\
    \text{where} \quad
    R^{\mathrm{act}}(s, b, m^{\mathrm{act}}) &= \bar{G}(s, b, m^{\mathrm{act}}) - C(m^{\mathrm{act}}) \nonumber \\
    &= \E_{\theta \sim b,m^{\mathrm{cor}}}[\hat G_{\mathrm{LLM}}(m^{\mathrm{act}}, \theta)] - C(m^{\mathrm{act}})
\end{align*}
Under \emph{user termination}, the net immediate reward $R^{\mathrm{act}}(s, b, m^{\mathrm{act}})$ is an expectation over cases where: (i) the user accepts $m^{\mathrm{act}}$ (resulting in positive goal reward $\hat G_{\mathrm{LLM}}(m^{\mathrm{act}}, \theta) > 0$); (ii) the user's intent $\theta$ leads them to reject $m^{\mathrm{act}}$ and give a correction $m^{\mathrm{cor}}$ (resulting in zero goal reward $\hat G_{\mathrm{LLM}}(m^{\mathrm{act}}, \theta) = 0$). While the user model $\hat \pi^{\mathrm{user}}$ is technically required to determine acceptance vs. correction, our implementation avoids this check and directly takes the belief-weighted average over $\hat G_{\mathrm{LLM}}(m^{\mathrm{act}}, \theta)$.  This is a good approximation since $\hat G_{\mathrm{LLM}}(m^{\mathrm{act}}, \theta) \approx 0$ for any $\theta$ that would lead to a correction.

For the residual expectation $\mathbb{E}_{m^{\mathrm{cor}}}[\, R^{\mathrm{act}}(s', b') - C(m^{\mathrm{cor}})]$, the expectand is only non-zero when a correction $m^{\mathrm{cor}}$ is issued. We can thus rewrite it as:
\begin{equation*}
    \mathbb{E}_{m^{\mathrm{cor}}}\big[\, R^{\mathrm{act}}(s', b') - C(m^{\mathrm{cor}}) \,\big] = \\
    \mathbb{E}_{\theta \sim b}\big[ P_{\mathrm{cor}}(\theta, m^{\mathrm{act}}) \big(\, R^{\mathrm{act}}(s', b') - \mathbb{E}_{m^{\mathrm{cor}}}[C(m^{\mathrm{cor}})] \,\big) \,\big],
\end{equation*}
where $P_{\mathrm{cor}}(\theta, m^{\mathrm{act}})$ is the probability that a user with intent $\theta$ corrects the answer $m^{\mathrm{act}}$. To avoid sampling $\hat \pi^{\mathrm{user}}$ many times to estimate $P_{\mathrm{cor}}$, we approximate it as a fixed value across all $\theta$ and $m^{\mathrm{act}}$, parameterized by the uncertainty of the agent's current particle belief $b = \{(\theta^n, w^n)\}_{n=1}^N$: $P_{\mathrm{cor}}(\theta, m^{\mathrm{act}}) \approx 1 - \max_{n} w^n$.
When the agent is uncertain (a near-uniform belief with low peak probability $\max_{n} w^n$) corrections are likely across the board, whereas a confident agent trusts that the user is largely satisfied.

Factoring $P_{\mathrm{cor}}$ out of the expectation gives the value of acting via $m^\mathrm{act}$ that REVOIR evaluates at decision time:
\begin{equation}
    Q^{\mathrm{act}}(s, b, m^{\mathrm{act}}) \approx R^{\mathrm{act}}(s, b, m^{\mathrm{act}}) + \big(1 - \max_{n} w^n \big) \textstyle\sum_{n=1}^{N} w^n\, \big(\, R^{\mathrm{act}}(s', b'_n) - C(m^{\mathrm{cor}}_n) \,\big),
\end{equation}
where $m^{\mathrm{cor}}_n$ is the simulated correction for intent $\theta^n$ and $b'_n$ is the resulting updated belief. Under \emph{agent termination}, the correction branch is absent ($P_{\mathrm{cor}} = 0$), and so the value of acting with $m^\mathrm{act}$ reduces to $Q^{\mathrm{act}}(s, b, m^{\mathrm{act}}) = R^{\mathrm{act}}(s, b, m^{\mathrm{act}})$.

\subsection{Baselines for CondAmbigQA}

We compare REVOIR against EIG maximization, bare prompting, generic inference-time reasoning algorithms, and models finetuned to ask clarifying questions:

\begin{itemize}[leftmargin=*,topsep=0pt]
    \item \textbf{REIGN.} A REVOIR variant whose (proxy) goal reward is $\hat G(s, \theta) = \log b(\theta)$. The expected goal reward is thus the negative entropy of the belief $-H(\hat{b})$ (see Appendix~\ref{sec:eig}), which we normalize to $1 - H(\hat{b}_t)/H_{\max} \in [0,1]$ to ensure a similar scale as our cost function. REIGN is thus the cost-aware counterpart of the EIG maximization strategy in Bayesian experimental design~\citep{lindley_measure_1956, mackay_information_1992, rainforth_modern_2024}, and a stand-in for other EIG-style algorithms like \citet{hu_uncertainty_2024}. Since EIG is indifferent to whether a sharper belief yields a better answer, comparing REIGN to REVOIR isolates the value of \emph{task-grounded} VoI.

    \item \textbf{Direct Answer.} The LLM is minimally prompted and directly answers the ambiguous question without any clarification.

    \item \textbf{ReAct.} One of the most widely adopted frameworks for LLM-based agents~\citep{yao_react_2022}, ReAct interleaves chain-of-thought reasoning with action execution. The ask-vs-act decision is entirely implicit in the LLM's generation with no explicit belief, no principled stopping criterion, and no cost accounting.

    \item \textbf{Self-Consistency Best-of-N ReAct (SC-BoN-ReAct).} Generates $K$ ReAct samples, then decides to ask or act via majority vote among the samples (a.k.a. self-consistency, \citet{wang_selfconsistency_2023}), with a fitted threshold for determining majority. Out of the $N \leq K$ majority samples, the best sample is selected with an LLM judge. This judge is similar to REVOIR's proxy goal reward $\hat{G}_{\mathrm{LLM}}$ when the majority votes to \emph{act}, and uses a separate question selection prompt when the majority votes to \emph{ask}.
    
    \item \textbf{Entropy Thresholding.} Shares the same particle belief as REVOIR but replaces VoI reasoning with a fixed stopping rule: ask until the normalized negative entropy $\frac{H_{\max}(b) - H(b)}{H_{\max}(b)}$ exceeds a threshold $\tau$ tuned on the validation set -- i.e., when the belief is ``concentrated enough''.

    \item \textbf{ReflectionDPO-ReAct.} An LLM finetuned via our adaptation of the ReflectionDPO algorithm~\citep{patel_adapt_2025} to CondAmbigQA, learning when to ask clarifying questions through a privileged-teacher (Oracle-ReAct) reflection mechanism (Appendix~\ref{sec:reflection-dpo}).

    \item \textbf{Oracle-ReAct.} A ReAct agent given the ground-truth extractive answer $y$ as privileged context, which rephrases it fluently without asking any questions. It serves both as the oracle teacher for ReflectionDPO training and as a non-interactive topline.
\end{itemize}

\section{Additional Results on CondAmbigQA}

\subsection{Full Results of Llama-3.1-8B on CondAmbigQA}
\label{sec:condambigqa_llama3.1-8b_full}

\textbf{Under \emph{agent termination}, REVOIR best trades-off correctness against clarification costs.} We first report results in the ``standard'' QA setting where the interaction terminates after the agent gives an answer. Figure~\ref{fig:condambigqa_llama3.1-8b_at} shows results on Llama-3.1-8B. Panel~(a) reveals a clear accuracy--efficiency frontier on which REVOIR Pareto-dominates REIGN: at matched word budgets, REVOIR achieves higher correctness, reaching its best correctness (0.55) in 5.6 questions, whereas REIGN keeps asking up to 9 questions yet plateaus below 0.54 in correctness. This is because REIGN's objective is insensitive to whether a sharper belief improves the answer, so it asks until uncertainty is low rather than until expected task reward stops rising.
Entropy Thresholding, which ignores costs and task reward, asks about 4 questions yet achieves the lowest correctness among multi-turn methods.
Panel~(b) shows the correctness increase of REVOIR over each baseline with standard errors (computed from 10,000 bootstrap samples), establishing that REVOIR's advantage  over most methods (considering answer correctness alone) is robust when accounting for dataset noise.

\begin{figure}[htbp]
    \centering

    \includegraphics[width=1\linewidth]{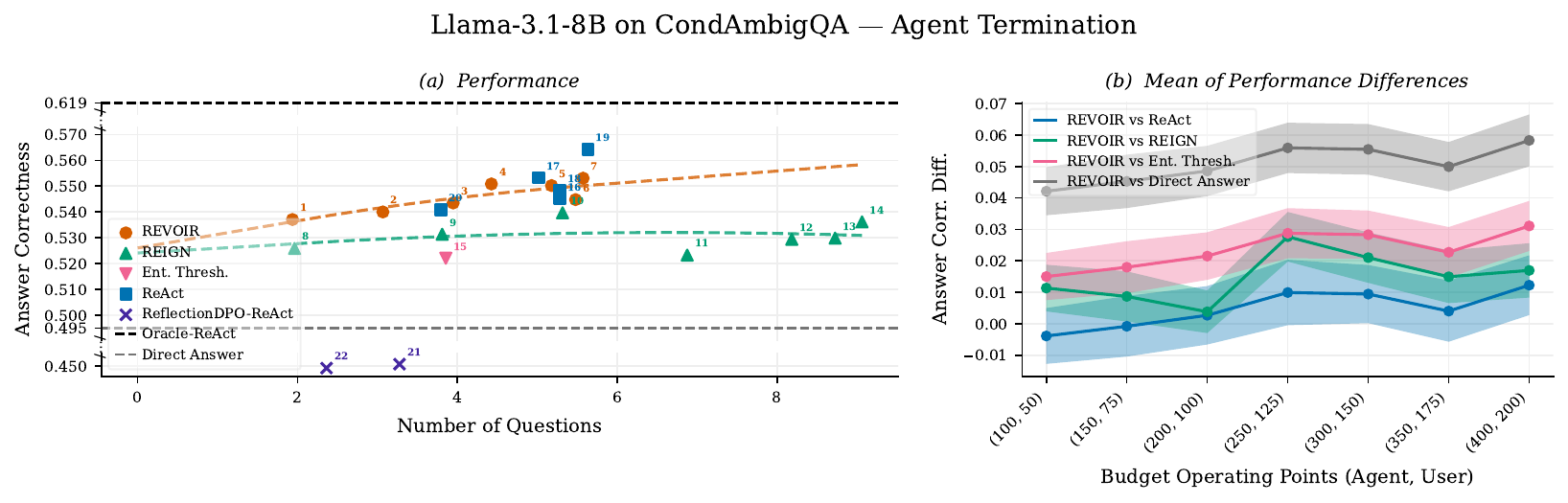}
    \caption{\textbf{CondAmbigQA results for Llama-3.1-8B under \emph{agent} termination.} \textbf{(a)} Answer correctness vs.\ average number of questions asked. Each numbered point corresponds to an agent/user word budget $(B_{\mathrm{agent}}, B_{\mathrm{user}})$ or (for ReAct) sample budget $K$; higher numbers map to higher values;  Table~\ref{tab:condambigqa_llama3.1-8b_at_details} lists exact values. 
    \textbf{(b)} Mean pairwise differences in answer correctness between REVOIR and each baseline across budget points. ReAct is a single operating point broadcast across the budget axis. Shaded bands are bootstrapped standard errors, with $N_{\mathrm{bootstrap}} = 10{,}000$.}
    \label{fig:condambigqa_llama3.1-8b_at}
    \vspace{10pt}

    \footnotesize
    \begin{tabular}{rlrrrr}
        \toprule
        \textbf{ } & \textbf{Method} & \shortstack{\textbf{Agent} \\ \textbf{Budget}} & \shortstack{\textbf{User} \\ \textbf{Budget}} & \shortstack{\textbf{No.} \\ \textbf{Questions}} & \shortstack{\textbf{Answer} \\ \textbf{Correct.}} \\
        \midrule
        \textcolor[HTML]{D55E00}{\textbf{1}} & REVOIR & 100 & 50 & 1.94 & 0.5370 \\
        \textcolor[HTML]{D55E00}{\textbf{2}} & REVOIR & 150 & 75 & 3.07 & 0.5400 \\
        \textcolor[HTML]{D55E00}{\textbf{3}} & REVOIR & 200 & 100 & 3.95 & 0.5435 \\
        \textcolor[HTML]{D55E00}{\textbf{4}} & REVOIR & 250 & 125 & 4.43 & 0.5509 \\
        \textcolor[HTML]{D55E00}{\textbf{5}} & REVOIR & 300 & 150 & 5.18 & 0.5502 \\
        \textcolor[HTML]{D55E00}{\textbf{6}} & REVOIR & 350 & 175 & 5.48 & 0.5448 \\
        \textcolor[HTML]{D55E00}{\textbf{7}} & REVOIR & 400 & 200 & 5.57 & 0.5531 \\
        \textcolor[HTML]{009E73}{\textbf{8}} & REIGN & 100 & 50 & 1.96 & 0.5258 \\
        \textcolor[HTML]{009E73}{\textbf{9}} & REIGN & 150 & 75 & 3.81 & 0.5313 \\
        \textcolor[HTML]{009E73}{\textbf{10}} & REIGN & 200 & 100 & 5.32 & 0.5397 \\
        \textcolor[HTML]{009E73}{\textbf{11}} & REIGN & 250 & 125 & 6.88 & 0.5232 \\
        \textcolor[HTML]{009E73}{\textbf{12}} & REIGN & 300 & 150 & 8.18 & 0.5294 \\
        \textcolor[HTML]{009E73}{\textbf{13}} & REIGN & 350 & 175 & 8.73 & 0.5298 \\
        \textcolor[HTML]{009E73}{\textbf{14}} & REIGN & 400 & 200 & 9.06 & 0.5361 \\
        \textcolor[HTML]{F06292}{\textbf{15}} & Ent. Thresh. & N/A & N/A & 3.86 & 0.5220 \\
        \textcolor[HTML]{0072B2}{\textbf{16}} & SC-BoN-ReAct (K=5) & N/A & N/A & 5.28 & 0.5451 \\
        \textcolor[HTML]{0072B2}{\textbf{17}} & SC-BoN-ReAct (K=10) & N/A & N/A & 5.02 & 0.5533 \\
        \textcolor[HTML]{0072B2}{\textbf{18}} & SC-BoN-ReAct (K=15) & N/A & N/A & 5.28 & 0.5482 \\
        \textcolor[HTML]{0072B2}{\textbf{19}} & SC-BoN-ReAct (K=20) & N/A & N/A & 5.63 & 0.5642 \\
        \textcolor[HTML]{0072B2}{\textbf{20}} & ReAct & N/A & N/A & 3.80 & 0.5409 \\
        \textcolor[HTML]{4527A0}{\textbf{21}} & ReflectionDPO-ReAct (ask-heavy) & N/A & N/A & 3.28 & 0.4507 \\
        \textcolor[HTML]{4527A0}{\textbf{22}} & ReflectionDPO-ReAct (ask-moderate) & N/A & N/A & 2.36 & 0.4492 \\
        \textcolor[HTML]{000000}{\textbf{23}} & Oracle-ReAct & N/A & N/A & 0.00 & 0.6194 \\
        \textcolor[HTML]{757575}{\textbf{24}} & Direct Answer & N/A & N/A & 0.00 & 0.4948 \\
        \bottomrule
    \end{tabular}
    \captionof{table}{Details for each operating point in Figure~\ref{fig:condambigqa_llama3.1-8b_at}}
    \label{tab:condambigqa_llama3.1-8b_at_details}
\end{figure}

\begin{figure}[htbp]
    \centering

    \includegraphics[width=1\linewidth]{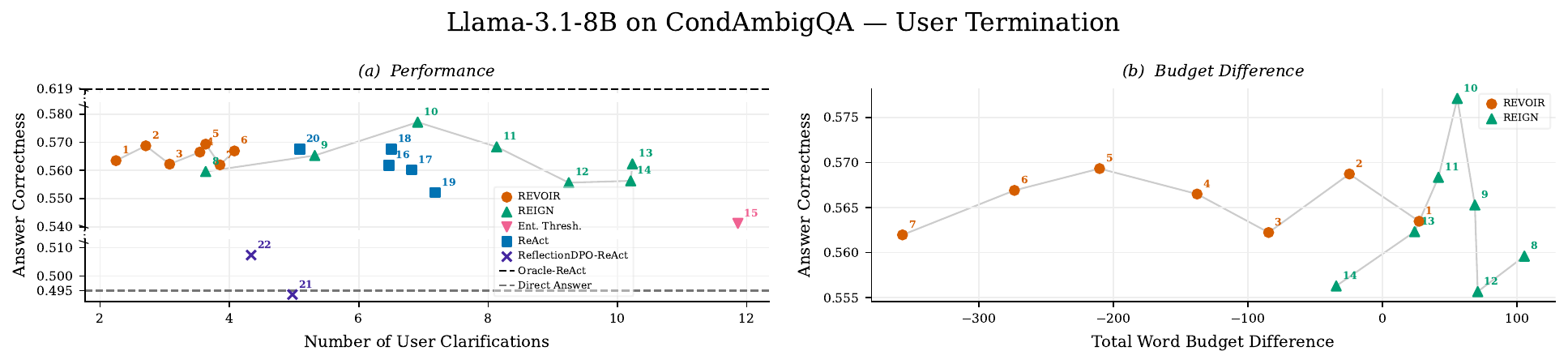}
    \caption{\textbf{CondAmbigQA results for Llama-3.1-8B under \emph{user} termination.} \textbf{(a)} Answer correctness vs.\ number of user clarifications (questions asked by the agent \emph{or} corrections issued by the user). Each numbered point corresponds to an agent/user word budget $(B_{\mathrm{agent}}, B_{\mathrm{user}})$ or (for ReAct) sample budget $K$; higher numbers map to higher values; Table~\ref{tab:condambigqa_llama3.1-8b_ut_details} lists exact values. \textbf{(b)} The operating points of budget-aware methods (REVOIR and REIGN) plotted against total word budget difference (words used minus configured budget); negative values indicate the agent acted before exceeding its budget for cost-free words.}
    \label{fig:condambigqa_llama3.1-8b_ut}
    \vspace{10pt}

    \footnotesize %
    \setlength{\tabcolsep}{3pt} %
    \begin{tabular*}{\linewidth}{@{\extracolsep{\fill}} rl rrrrrr @{}} %
    \toprule
    \textbf{ } & \textbf{Method} & \shortstack{\textbf{Agent} \\ \textbf{Budg.}} & \shortstack{\textbf{User} \\ \textbf{Budg.}} & \shortstack{\textbf{No.} \\ \textbf{Ques.}} & \shortstack{\textbf{No.} \\ \textbf{Corr.}} & \shortstack{\textbf{No.} \\ \textbf{Clar.}} & \shortstack{\textbf{Answer} \\ \textbf{Correct.}} \\
    \midrule
    \textcolor[HTML]{D55E00}{\textbf{1}} & REVOIR & 100 & 50 & 1.00 & 1.25 & 2.24 & 0.5635 \\
    \textcolor[HTML]{D55E00}{\textbf{2}} & REVOIR & 150 & 75 & 1.39 & 1.31 & 2.71 & 0.5687 \\
    \textcolor[HTML]{D55E00}{\textbf{3}} & REVOIR & 200 & 100 & 1.73 & 1.34 & 3.07 & 0.5622 \\
    \textcolor[HTML]{D55E00}{\textbf{4}} & REVOIR & 250 & 125 & 2.23 & 1.32 & 3.54 & 0.5665 \\
    \textcolor[HTML]{D55E00}{\textbf{5}} & REVOIR & 300 & 150 & 2.40 & 1.23 & 3.63 & 0.5693 \\
    \textcolor[HTML]{D55E00}{\textbf{6}} & REVOIR & 350 & 175 & 2.73 & 1.34 & 4.08 & 0.5669 \\
    \textcolor[HTML]{D55E00}{\textbf{7}} & REVOIR & 400 & 200 & 2.58 & 1.27 & 3.85 & 0.5620 \\
    \textcolor[HTML]{009E73}{\textbf{8}} & REIGN & 100 & 50 & 2.24 & 1.39 & 3.63 & 0.5596 \\
    \textcolor[HTML]{009E73}{\textbf{9}} & REIGN & 150 & 75 & 4.08 & 1.24 & 5.32 & 0.5653 \\
    \textcolor[HTML]{009E73}{\textbf{10}} & REIGN & 200 & 100 & 5.66 & 1.25 & 6.91 & 0.5771 \\
    \textcolor[HTML]{009E73}{\textbf{11}} & REIGN & 250 & 125 & 6.84 & 1.29 & 8.13 & 0.5684 \\
    \textcolor[HTML]{009E73}{\textbf{12}} & REIGN & 300 & 150 & 7.77 & 1.48 & 9.25 & 0.5557 \\
    \textcolor[HTML]{009E73}{\textbf{13}} & REIGN & 350 & 175 & 8.55 & 1.68 & 10.23 & 0.5623 \\
    \textcolor[HTML]{009E73}{\textbf{14}} & REIGN & 400 & 200 & 8.53 & 1.67 & 10.21 & 0.5563 \\
    \textcolor[HTML]{F06292}{\textbf{15}} & Ent. Thresh. & N/A & N/A & 9.91 & 1.95 & 11.86 & 0.5412 \\
    \textcolor[HTML]{0072B2}{\textbf{16}} & SC-BoN-ReAct (K=5) & N/A & N/A & 5.88 & 0.58 & 6.47 & 0.5618 \\
    \textcolor[HTML]{0072B2}{\textbf{17}} & SC-BoN-ReAct (K=10) & N/A & N/A & 6.16 & 0.66 & 6.82 & 0.5602 \\
    \textcolor[HTML]{0072B2}{\textbf{18}} & SC-BoN-ReAct (K=15) & N/A & N/A & 5.91 & 0.59 & 6.50 & 0.5676 \\
    \textcolor[HTML]{0072B2}{\textbf{19}} & SC-BoN-ReAct (K=20) & N/A & N/A & 6.45 & 0.73 & 7.18 & 0.5522 \\
    \textcolor[HTML]{0072B2}{\textbf{20}} & ReAct & N/A & N/A & 4.45 & 0.64 & 5.09 & 0.5675 \\
    \textcolor[HTML]{4527A0}{\textbf{21}} & ReflectionDPO-ReAct (ask-heavy) & N/A & N/A & 4.14 & 0.83 & 4.97 & 0.4934 \\
    \textcolor[HTML]{4527A0}{\textbf{22}} & ReflectionDPO-ReAct (ask-moderate) & N/A & N/A & 3.47 & 0.86 & 4.33 & 0.5074 \\
    \textcolor[HTML]{000000}{\textbf{23}} & Oracle-ReAct & N/A & N/A & 0.00 & 0.00 & 0.00 & 0.6194 \\
    \textcolor[HTML]{757575}{\textbf{24}} & Direct Answer & N/A & N/A & 0.00 & 0.00 & 0.00 & 0.4948 \\
    \bottomrule
    \end{tabular*}
    \captionof{table}{Details for each operating point in Figure~\ref{fig:condambigqa_llama3.1-8b_ut}}
    \label{tab:condambigqa_llama3.1-8b_ut_details}
\end{figure}

\textbf{Under \emph{user termination}, REVOIR avoids costly questioning by being correction-aware.} Figure~\ref{fig:condambigqa_llama3.1-8b_ut} presents results in the more realistic setting where the user can correct an inaccurate answer instead of terminating the conversation. While all interactive baselines anticipate user corrections, panel (a) shows that REVOIR rationally adapts to user corrections, using only 2--4 total user clarifications (i.e., questions asked + corrections given) with comparable correctness (0.56--0.57) to the agent termination setting. Panel~(b) further shows that REVOIR's word count is substantially less than the budgeted amount (negative budget differences) --- it acts well within budget by rationally relying on user corrections rather than asking more questions. REIGN, in contrast, keeps asking many questions: at its best it beats REVOIR slightly in correctness (0.58 vs 0.57), but only by relying on 6.9 clarifications compared to REVOIR's 3.6. Entropy Thresholding reaches 12 clarifications yet falls below REVOIR in correctness, having no way to reason that correction is cheaper than asking. 

\subsubsection*{Llama-3.1-8B Token Usage}
\begin{table}[H]
    \centering
    \begin{tabular}{lrrrr}
    \toprule
    Method & \makecell[r]{Input\\Tokens} & \makecell[r]{Output\\Tokens} & \makecell[r]{Total\\Tokens} & \makecell[r]{Answer\\Correctness} \\
    \midrule
    REVOIR (100, 50) & 7.0K & 3.5K & 10.6K & 0.5371 \\
    REVOIR (150, 75) & 9.6K & 4.5K & 14.2K & 0.5400 \\
    REVOIR (200, 100) & 11.8K & 5.1K & 16.8K & 0.5435 \\
    REVOIR (250, 125) & 12.7K & 5.5K & 18.2K & 0.5509 \\
    REVOIR (300, 150) & 14.8K & 6.1K & 20.9K & 0.5502 \\
    REVOIR (350, 175) & 15.3K & 5.9K & 21.3K & 0.5448 \\
    REVOIR (400, 200) & 15.6K & 5.9K & 21.5K & 0.5531 \\
    \midrule
    ReAct & 2.3K & 0.3K & 2.6K & 0.5409 \\
    SC-BoN-ReAct (K=5) & 17.3K & 1.6K & 18.9K & 0.5451 \\
    SC-BoN-ReAct (K=10) & 30.5K & 3.1K & 33.6K & 0.5533 \\
    SC-BoN-ReAct (K=15) & 47.4K & 4.9K & 52.3K & 0.5482 \\
    SC-BoN-ReAct (K=20) & 67.1K & 6.9K & 73.9K & 0.5642 \\
    \bottomrule
    \end{tabular}
    \vspace{6pt}
    \caption{Token usage and answer correctness for REVOIR and SC-BoN-ReAct on Llama-3.1-8B under agent termination. SC-BoN-ReAct's token cost grows linearly with $N$, reaching ${\sim}74$K tokens at $N{=}20$ --- over $3\times$ the cost of the most expensive REVOIR operating point --- yet correctness gains are marginal and non-monotonic. REVOIR at comparable token budgets achieves similar or greater correctness, demonstrating that principled ask-vs-act decisions are more token-efficient than inference-time scaling within a fixed policy.}

    \label{tab:token-usage}
\end{table}

\subsection{Gemini 3.5 Flash-Lite and Claude Haiku 4.5 on CondAmbigQA}
\label{sec:condambigqa_gemini_claude}

\begin{figure}[h]
    \centering
    \captionsetup{aboveskip=2pt, belowskip=5pt}

    \begin{subfigure}{0.49\textwidth}
        \centering
        \includegraphics[width=\linewidth]{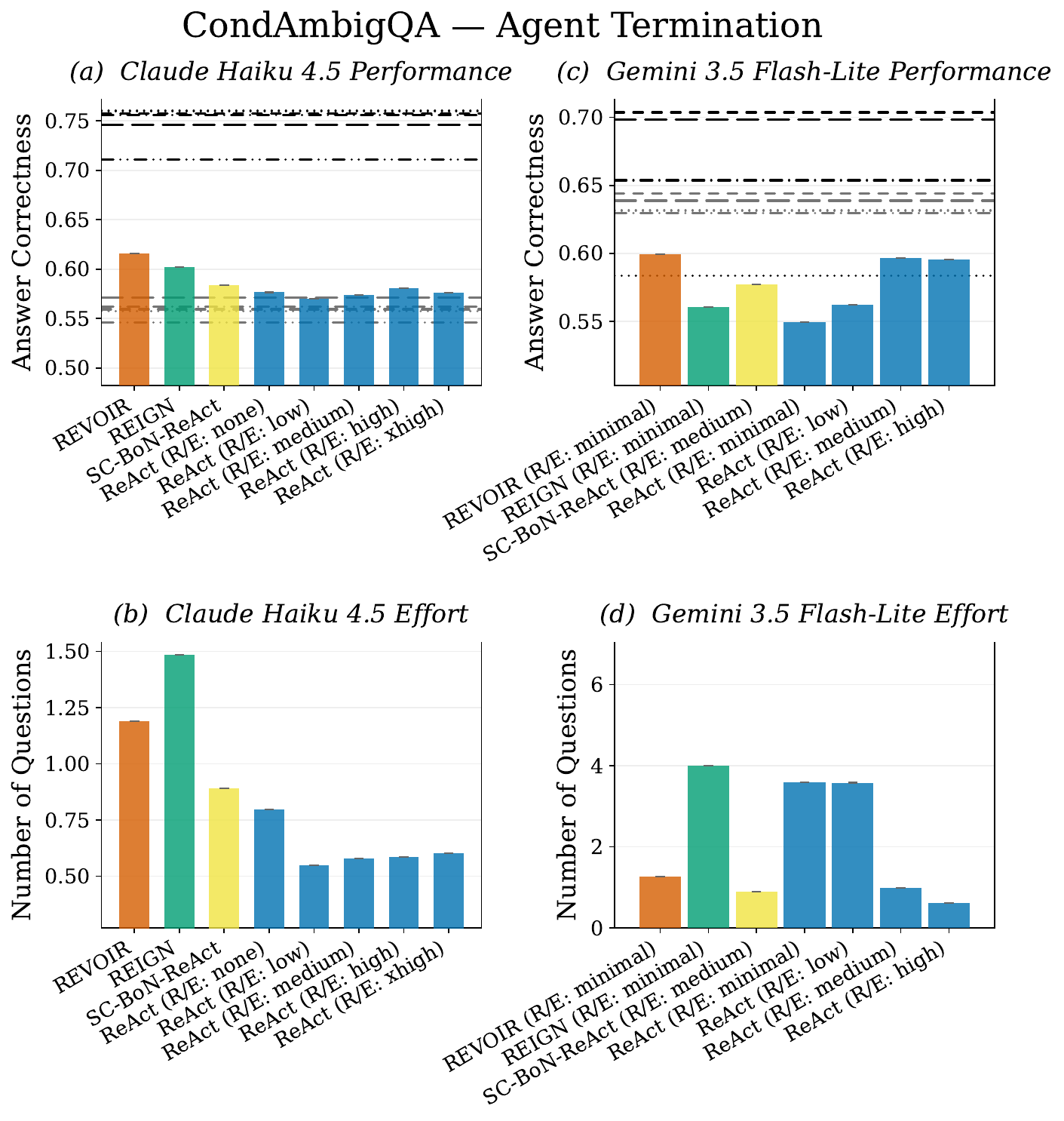}
    \end{subfigure}
    \hfill
    \begin{subfigure}{0.49\textwidth}
        \centering
        \includegraphics[width=\linewidth]{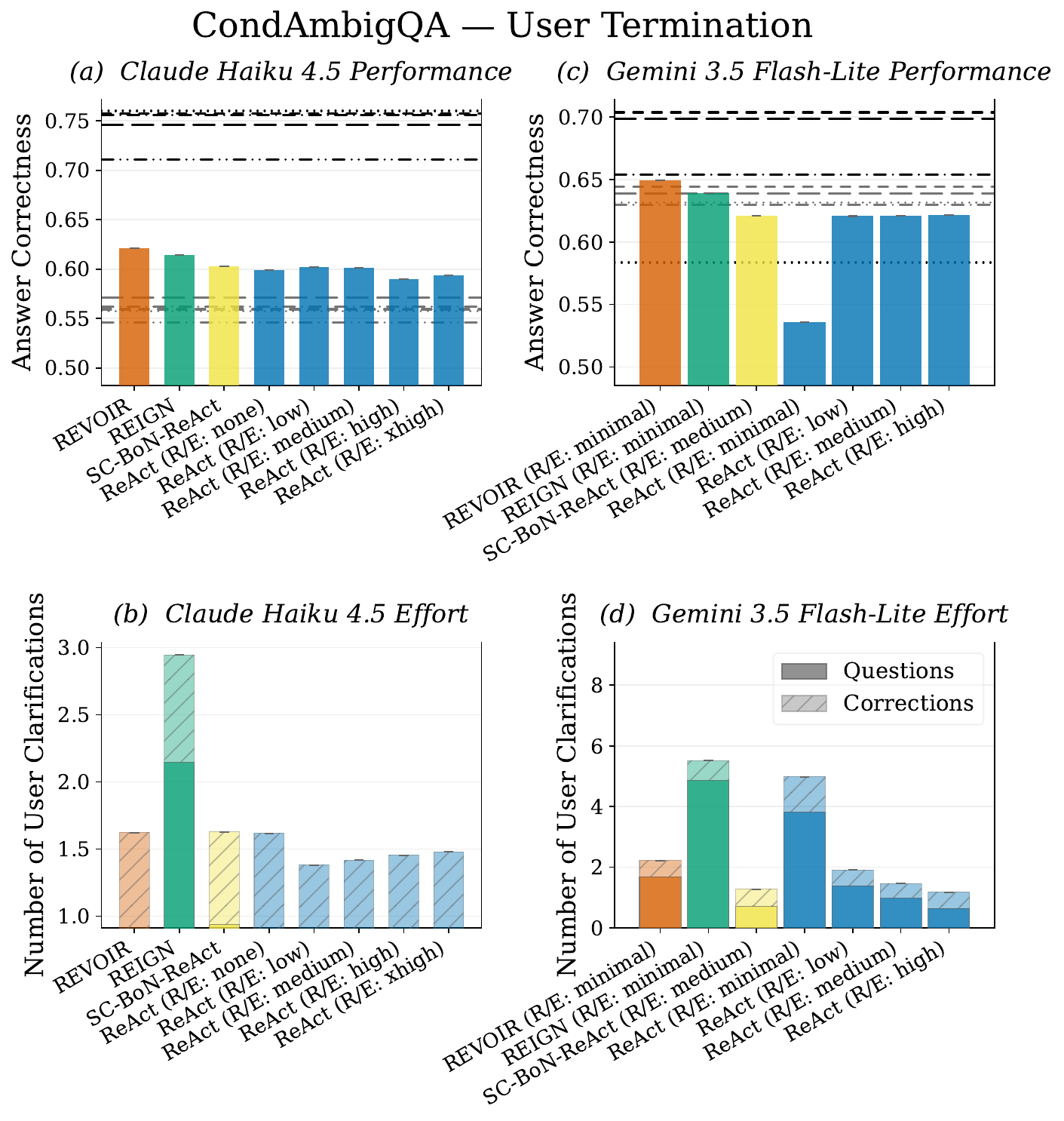}
    \end{subfigure}
    \caption{\textbf{CondAmbigQA results with Claude Haiku 4.5 and  Gemini 3.5 Flash-Lite}. ReAct variants and non-interactive base/top lines at available (Gemini 3.5 Flash-Lite) and emulated (Claude Haiku 4.5) reasoning-effort (R/E) levels. Horizontal lines in the top panels correspond to baselines and toplines across R/E levels: \textbf{\textcolor[HTML]{757575}{gray lines}} represent the \textbf{Direct Answer} baselines, and \textbf{black lines} represent \textbf{Oracle-ReAct} toplines, with line styles indicating effort levels: {\scriptsize$(-\ \cdot\ \cdot\ -)$} for \texttt{none} (or \texttt{minimal}), {\scriptsize$(\mbox{---}\ \mbox{---})$} for \texttt{low}, {\scriptsize$(-\ \cdot\ -)$} for \texttt{medium}, {\scriptsize$(\cdot\ \cdot\ \cdot)$} for \texttt{high}, and {\scriptsize$(-\ \ -)$} for \texttt{xhigh}.}
    \label{fig:condambigqa_gemini3.5-flashlite_claude-haiku4.5}
    \vspace{-6pt}
\end{figure}

We replicate experiments in Section~\ref{sec:condambigqa-exp} on Gemini 3.5 Flash-Lite and Claude Haiku 4.5, except for the Entropy Thresholding method due to cost constraints (Figure~\ref{fig:condambigqa_gemini3.5-flashlite_claude-haiku4.5}). Since it is not possible to disable reasoning for Gemini 3.5 Flash-Lite, we evaluate REVOIR and REIGN with \texttt{minimal} reasoning. For SC-BoN-ReAct, we evaluate with \texttt{medium} reasoning level, since lower levels yield empty responses due to context-window-overflowed from reasoning. For Claude Haiku 4.5, we emulate thinking levels with the following token budgets for its \textit{extending thinking mode}: \{\texttt{low}: 1024, \texttt{medium}: 4096, \texttt{high}: 8192, \texttt{xhigh}: 12288\}.

We find that across both models and termination settings, REVOIR generally achieves the highest \texttt{AnswerCorrectness} among interactive methods with minimal/no reasoning. Yet, each model exhibits different quirks. Both models outperform GPT-5.4-mini non-interactive toplines: Gemini 3.5 Flash-Lite bests at $0.70$, Claude Haiku 4.5 bests at $0.76$.  Surpisingly, however, DirectAnswer baselines with non-\texttt{minimal} reasoning efforts with Gemini 3.5 Flash-Lite also outperform the interactive methods we evaluated --- not what we expected of a baseline. Our qualitative inspections suggest that, in non-interactive settings, Gemini 3.5 Flash-Lite produces shotgun-style answers that are significantly more verbose than other frontier models (Figure~\ref{fig:condambigqa_ni_avg_words_gpt54-mini_gemini3.5-flashlite_claude-haiku4.5}), and cover multiple potential information-seeking intents.

For instance, a direct answer by Gemini 3.5 Flash-Lite with \texttt{low} reasoning effort to the question ``Who played Santa in Miracle on 34th St?'' is as follows:
\begin{quote}
     ``Multiple actors have played Santa Claus in the various versions of \textit{Miracle on 34th Street}:
     \begin{itemize}
         \item \textbf{1947 Original Movie:} \textbf{Edmund Gwenn} played Kris Kringle, and he won an Academy Award for his performance.
         \item \textbf{1994 Remake:} \textbf{Richard Attenborough} played the role in the 94 feature film.
         \item \textbf{1973 TV Movie:} \textbf{Sebastian Cabot} played Kris Kringle.
         \item \textbf{1959 TV Broadcast (The Ford Star Jubilee):} \textbf{Macdonald Carey} played the role (though the 1955 live TV version featured \textbf{Thomas Mitchell} as Santa).''
     \end{itemize}
\end{quote}

In contrast, GPT-5.4-mini's answer to the same question with \texttt{low} reasoning effort is 
\begin{quote}
    ``In \textit{Miracle on 34th Street} (1947), Santa was played by \textbf{Edmund Gwenn}. He even won an Academy Award for the role.''
\end{quote}

While Claude Haiku 4.5's answer, also with \texttt{low} reasoning effort, is 
\begin{quote}
    In the classic 1947 film ``Miracle on 34th Street,'' \textbf{Edmund Gwenn} played Santa Claus (Kris Kringle). He won an Academy Award for Best Supporting Actor for the role.    
    
    In the 1994 remake, \textbf{Richard Attenborough} played the same character.
\end{quote}

\begin{figure}
    \centering    \includegraphics[width=0.75\linewidth]{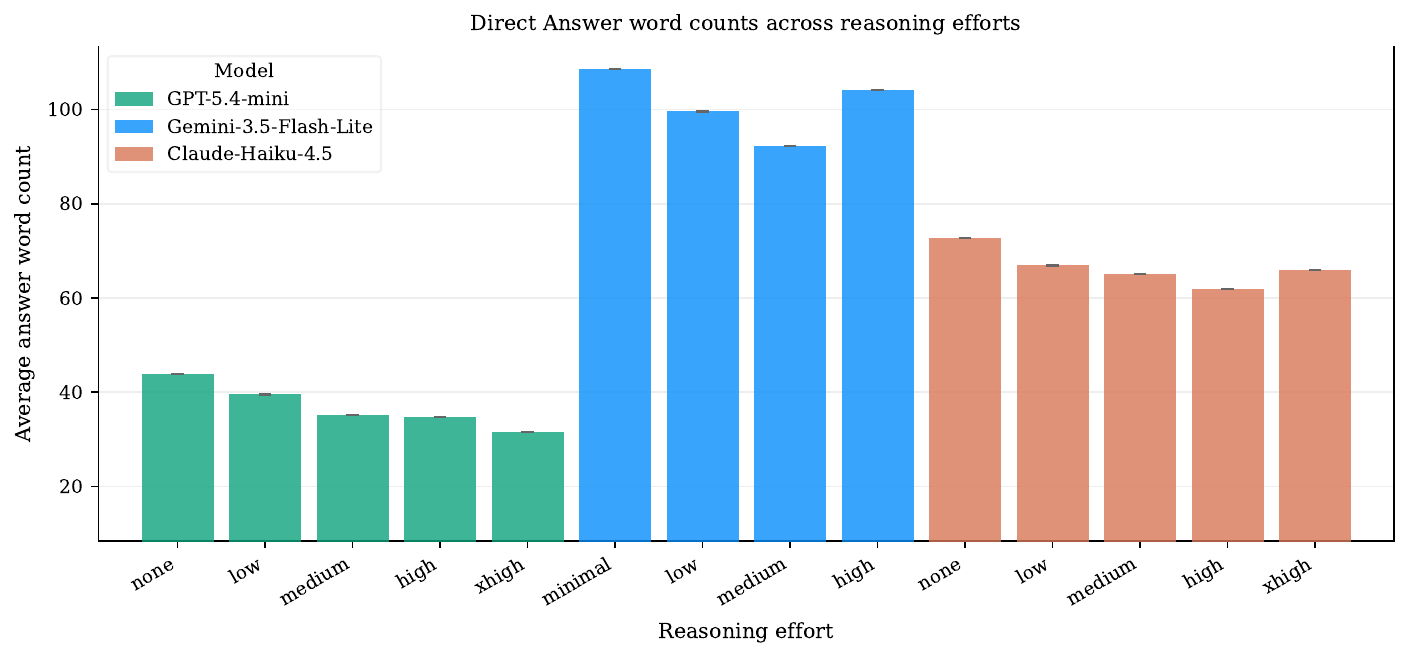}
    \caption{Average word counts of non-interactive one-shot queries from frontier models across reasoning efforts.}
\label{fig:condambigqa_ni_avg_words_gpt54-mini_gemini3.5-flashlite_claude-haiku4.5}
\end{figure}

While such hedging-based answer strategies are appropriate under certain use cases \textit{when the user is mentally prepared for an influx of information} (e.g., using a search engine), it might not always be appropriate. Regardless, our current evaluation rubrics do not discount answers for length or intent-irrelevance as long as they stay on-topic and cover the target information-seeking intent (Appendix~\ref{sec:condambigqa_eval}), thus benefiting the shotgun approach.

Claude Haiku 4.5, on the other hand, uses very few questions across all interactive methods and settings. It also practically stops asking (except for REIGN) when user corrections are anticipated. These behaviors are, perhaps, reflective of different reasoning training strategies and heuristics. Nevertheless, among interactive methods, REVOIR still achieves the best performance by reasoning about the value of information associated with each action.

\subsection{Results Under Inattentive User Simulator
}\label{sec:condambigqa_inatt_user}

\begin{figure}[h]
    \centering
    \captionsetup{aboveskip=2pt, belowskip=0pt}
    \begin{subfigure}{0.49\textwidth}
        \centering
        \includegraphics[width=\linewidth]{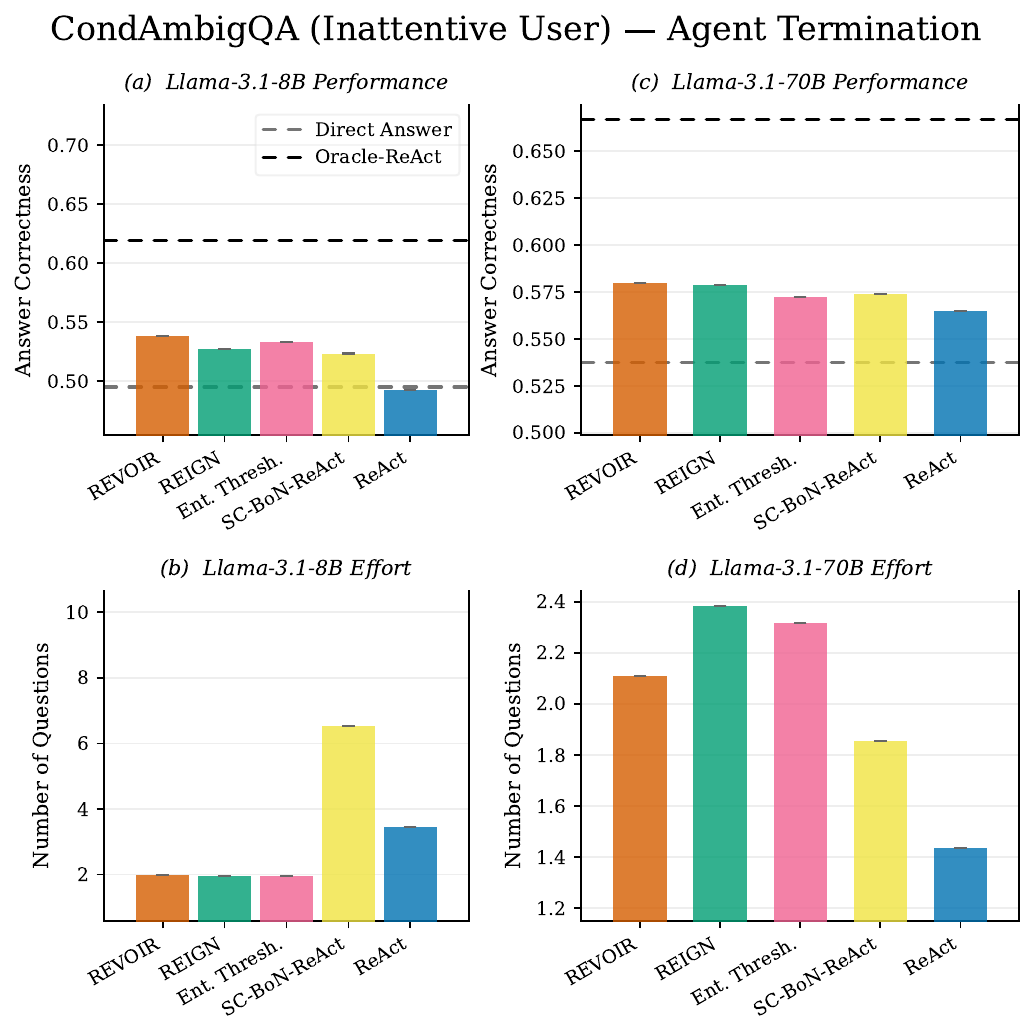}
        \caption{Agent termination.}
        \label{fig:condambigqa_at_inatt_llama3.1-8b-70b}
    \end{subfigure}
    \hfill
    \begin{subfigure}{0.49\textwidth}
        \centering
        \includegraphics[width=\linewidth]{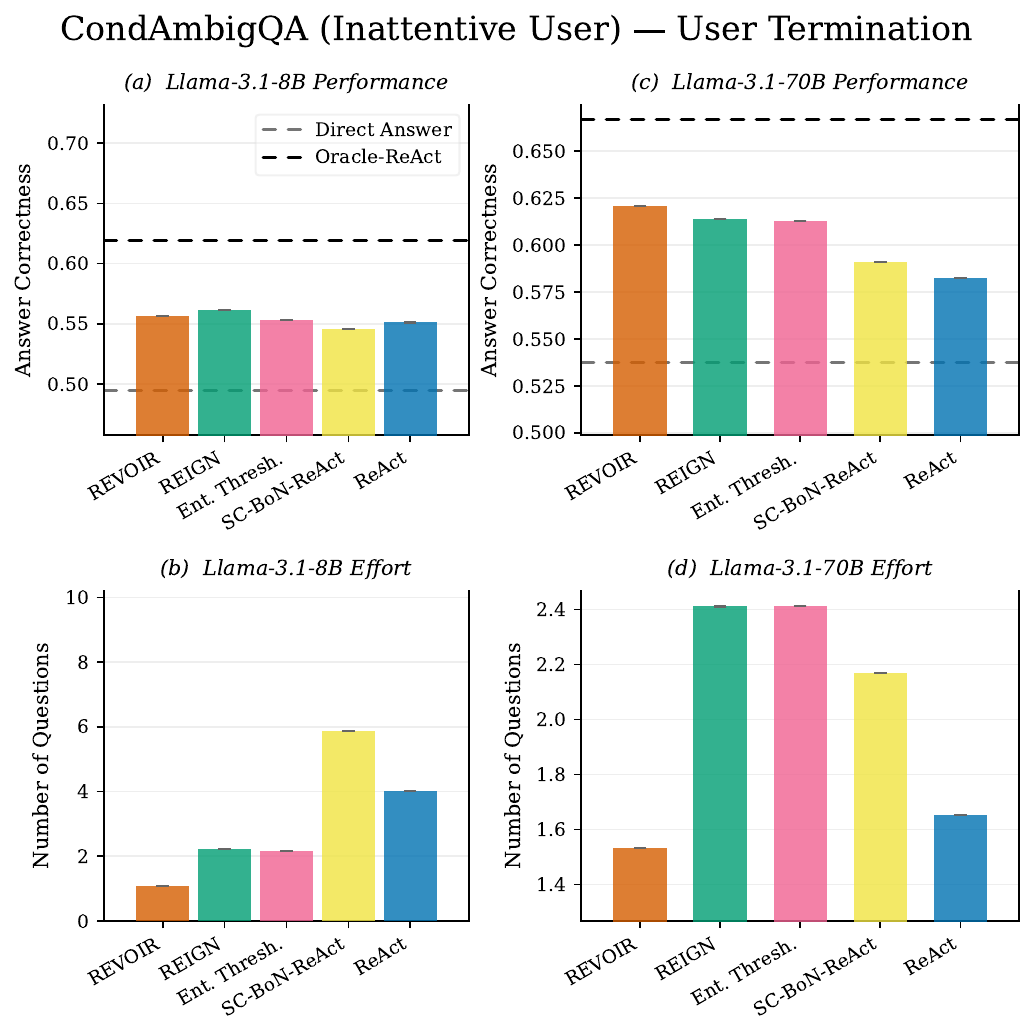}
        \caption{User termination.}
        \label{fig:condambigqa_ut_inatt_llama3.1-8b-70b}
    \end{subfigure}
    \caption{CondAmbigQA results with an \textit{inattentive user} simulation ($p_{dismissive}=0.2$) for Llama-3.1-8B and Llama-3.1-70B.}
    \label{fig:condambigqa_inatt_llama3.1-8b-70b}
\end{figure}

We introduced an inattentive runtime user who sporadically gives one of the dismissive responses with probability $p_{dismissive}=0.2$ at each clarification turn:
\begin{itemize}[topsep=0pt,itemsep=1pt]
    \item ``Not sure, whatever you think is right.''
    \item ``I don't remember exactly.''
    \item ``Just go with your best guess.''
    \item ``Hmm, I'm not really paying attention right now.''
\end{itemize}

This noise is absent from the user modeled internally by REVOIR, creating a direct internal-runtime simulator mismatch. Due to compute and cost constraints, we only replicate the experiments in Section~\ref{sec:condambigqa-exp} under the inattentive user simulator for Llama-3.1-8B and Llama-3.1-70B, with planning budgets of $(B_{\mathrm{agent}}, B_{\mathrm{user}}) = (100, 50)$ (i.e., 100 assistant words and 50 user words). \texttt{AnswerCorrectness} remains independent of dialogue length and cost.

Our findings in Section~\ref{sec:condambigqa-exp} persist under unmodeled responses (Figure~\ref{fig:condambigqa_inatt_llama3.1-8b-70b}). REVOIR has the highest correctness under agent termination for both backbones. Under user termination, it matches REIGN within 0.004 while asking about half as many questions for Llama-3.1-8B; for Llama-3.1-70B, it has the highest correctness and fewest total user interventions. Further results in Figure~\ref{fig:condambigqat_inatt_llama3.1-8b_across_p} from Llama-3.1-8B with $p_{dismissive} \in \{0.3, 0.4, 0.5, 0.6, 0.7, 0.8, 0.9\}$ confirm the pattern. Notably, only REVOIR and REIGN are robust to dismissive noise, while ReAct baselines visibly degrade in \texttt{AnswerCorrectness} as $p_{dismissive}$ increases.

\begin{figure}[h]
    \centering
    \captionsetup{aboveskip=2pt, belowskip=0pt}
    \begin{subfigure}{\textwidth}
        \centering
        \includegraphics[width=\linewidth]{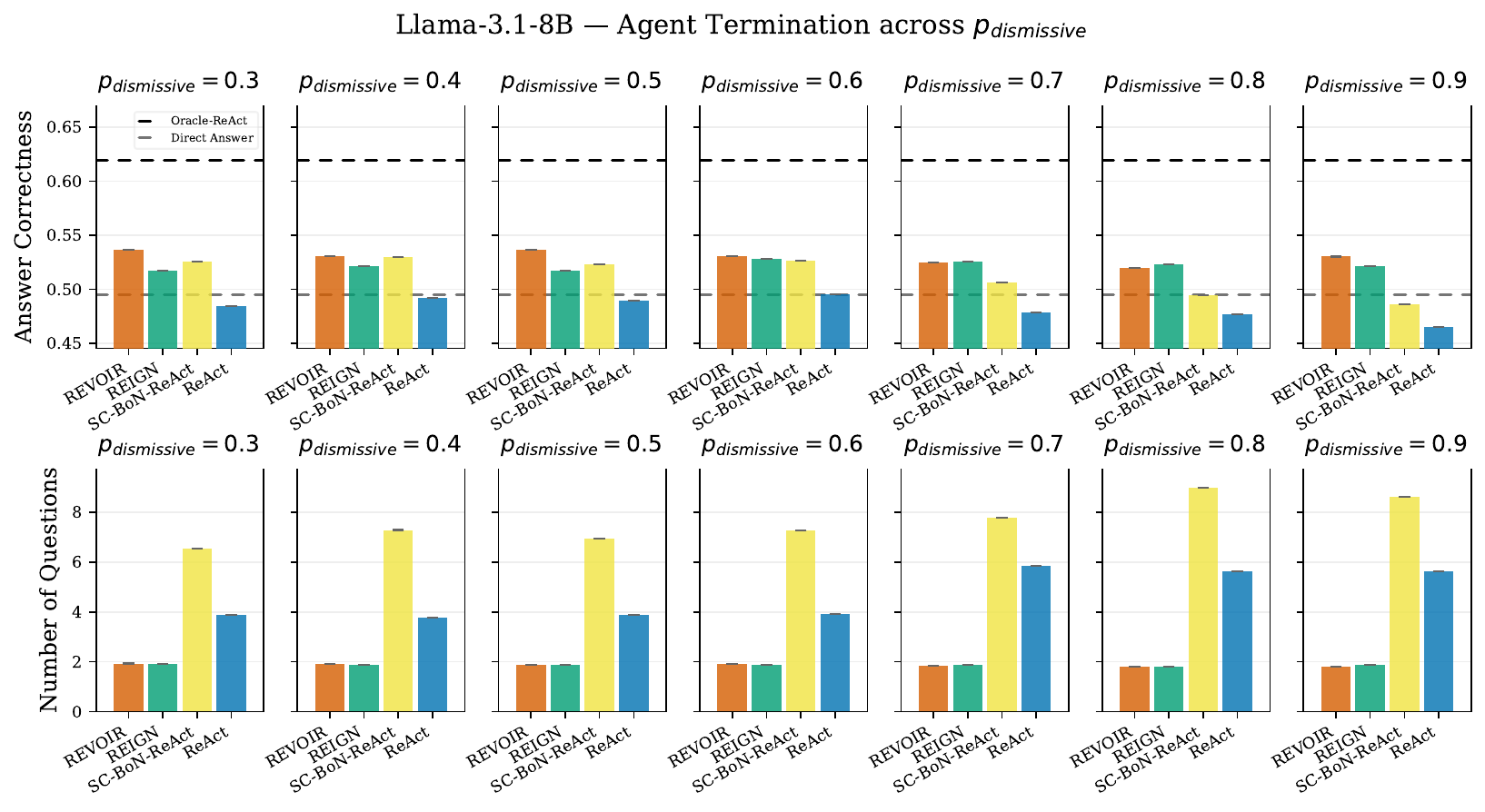}
        \caption{Agent termination.}
        \label{fig:condambigqa_at_inatt_llama3.1-8b_across_p}
    \end{subfigure}
    \vfill
    \begin{subfigure}{\textwidth}
        \centering
        \includegraphics[width=\linewidth]{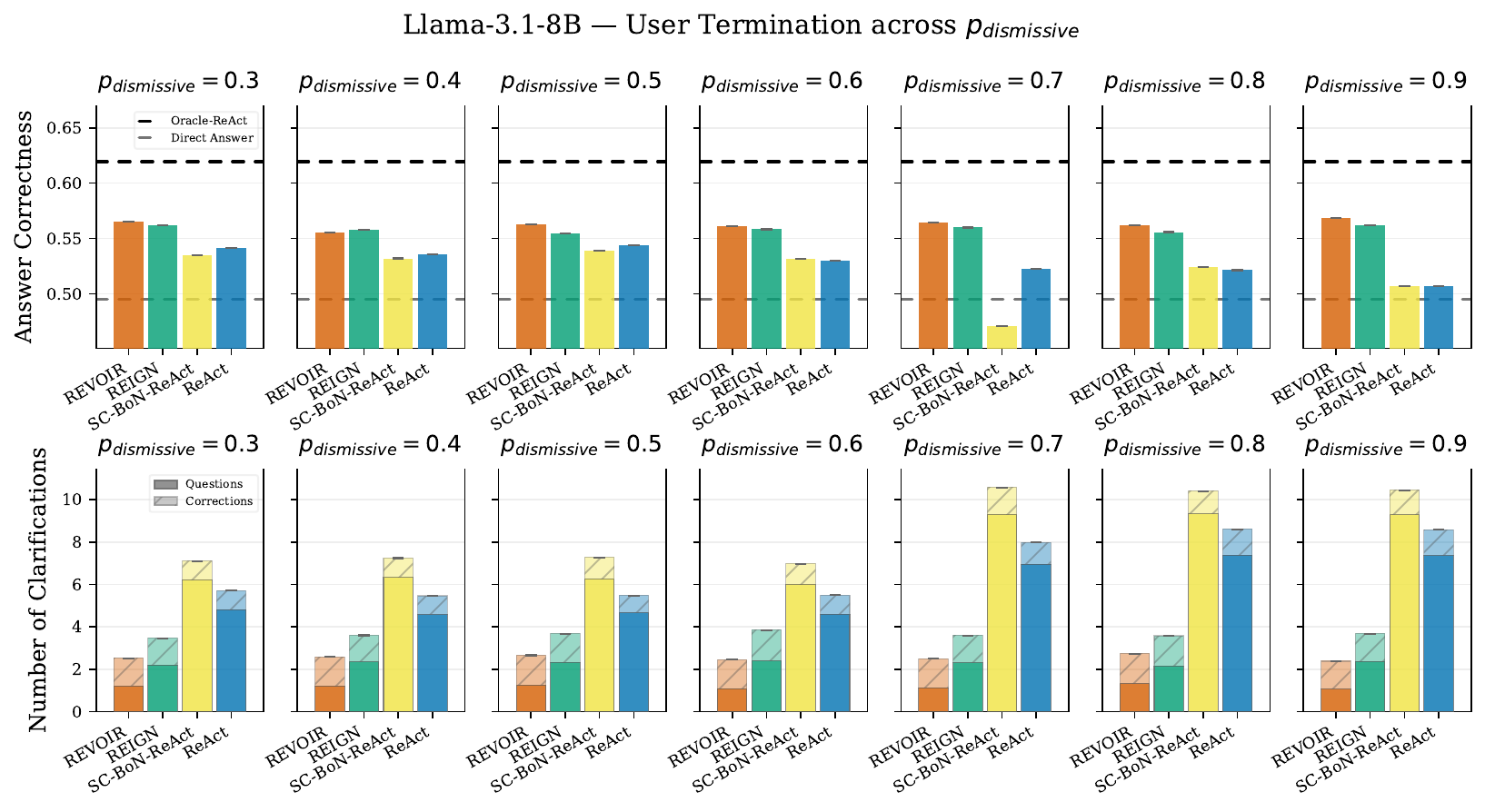}
        \caption{User termination.}
        \label{fig:condambigqa_ut_inatt_llama3.1-8b_across_p}
    \end{subfigure}
    \caption{Llama-3.1-8B CondAmbigQA results in \textit{inattentive user} simulations with $p_{dismissive} \in \{0.3, 0.4, 0.5, 0.6, 0.7, 0.8, 0.9\}$.}
    \label{fig:condambigqat_inatt_llama3.1-8b_across_p}
\end{figure}

\section{Validating the LLM-simulated user and LLM judge with human annotations}
\label{sec:pilot-human-replay}

We conducted a pilot conversation-replay study with human annotators to assess: (i) the user simulator's decision to terminate a conversation, (ii) the \texttt{AnswerCorrectness} rating of final answers. Annotators read logged GPT-5.4-mini conversations generated by REVOIR and ReAct (R/E: High, our strongest baseline) in the user termination setting.

Nine annotators were recruited from members of the university, and completed annotations for eight distinct item sets (with one response dropped due to an erroneous duplication of one item set between annotators). 
All reported intervals are two-sided 95\% Wilson score intervals for binomial proportions, as Wilson intervals remain well behaved with modest samples and rates near zero or one. 

\subsection{Does the user simulator terminate where a person would?}
\label{sec:replay-termination}

For the first part of the study, each annotator is instructed to roleplay the user who asked the original ambiguous question. For every conversation each annotator is provided:

\begin{enumerate}
    \item \textit{The context you had in mind -- but you didn't elaborate when asking}: the long-form disambiguating condition attached to the question;
    \item \textit{The unambiguous question you could have asked -- but didn't};
    \item \textit{How the conversation actually went}: the conversation transcript with interleaved turns (\textbf{Q:} the user, \textbf{A:} the assistant), ending with the assistant's final answer.
\end{enumerate}

The gold answer is withheld from the annotator for this part of the study. After reading each episode, annotators answer a binary query:

\begin{quote}
    \textbf{Would the assistant's last answer have satisfied you?}

    \textit{Yes — that would have answered me \\ No — that would not have answered me}
\end{quote}

Episodes are stratified by terminal outcome: \textit{accepted on the first answer}, \textit{accepted after clarification}, and \textit{never accepted}. Each annotator is tasked with 20 episodes: 10 uniquely sampled episodes balanced across the two policies and 10 \textit{anchor} episodes shared among all annotators. The anchors measure how often \textit{people agree with each other} on this task. After per-form
deduplication and removal of the eight attention-check responses, the analysis
contains 160 substantive judgements: 80 judgements on 80 form-specific episodes
and 80 judgements on the same 10 anchors. These judgements therefore cover 90
distinct episodes, rather than 160 distinct episodes.

For each episode, we compare one human verdict with the simulator's verdict.
The form-specific episodes have one human judgement each; the anchors use the
majority judgement across annotators. We do not break an evenly split vote
toward either class. One anchor episode received a tied 4--4 judgments is subsequently dropped from the analysis, leaving the 89 verdicts in Table~\ref{tab:satisfaction-outcomes}. Of the remaining anchors, 5 episodes get unanimous human-simulator agreement, 1 episode gets unanimous human-simulator disagreement, and 3 episodes get majority human-simulator agreement (two 5-3 splits and one 6-2 split). 

\begin{table}[h]
\centering
\small
\caption{Human and simulated-user satisfaction verdicts.}
\label{tab:satisfaction-outcomes}
\begin{tabular}{lrrr}
\toprule
 & \multicolumn{2}{c}{\textbf{Simulated user}} & \\
\cmidrule(lr){2-3}
\textbf{Human verdict} & \textbf{Satisfied} & \textbf{Not satisfied} & \textbf{Total} \\
\midrule
Human: satisfied     & \textbf{64} & 10 & 74 \\
Human: not satisfied & 7 & \textbf{8} & 15 \\
\midrule
\textbf{Total} & 71 & 18 & 89 \\
\bottomrule
\end{tabular}
\end{table}

The human-simulator agreement is $72/89=0.809$ (95\% CI $[0.72,0.88]$) across both simulator decisions. On the 71 episodes that the simulated user accepted, the human verdict also accepts 64, giving an acceptance agreement of $64/71=0.901$ (95\% CI $[0.81,0.95]$). Agreement is $26/29=0.897$ (95\% CI $[0.74,0.96]$) for episodes accepted after corrections and $38/42=0.905$ ($[0.78,0.96]$) for episodes accepted on the first answer. 

The disagreements are asymmetric. There are 10 episodes on which the simulator rejects an answer that the human verdict accepts, compared with 7/15 episodes on which the simulator accepts an answer that the human verdict rejects. 6 stricter simulator decisions concern answers that appear satisfactory but \textit{conflict with the grounding condition}. Of the other five, four are genuine over-strict decisions on answers not conflicting the grounding condition, and one is due to a defective grounding condition that does not resolve ambiguity. Thus, the simulator more often errs toward continuing a conversation instead of ending it prematurely, while human is more lenient and concur on only $8/18=0.444$ of simulator rejections (95\% CI $[0.25,0.66]$).

\subsection{Does the LLM judge track human answer-quality judgments?}
\label{sec:replay-answer-judge}

In the second part, annotators are shown pairs of conversations corresponding to the same question, along with the grounding context and final answers from the two policiesin randomized order. They select the answer that they perceive as better, or mark the answers as equally good. Each retained annotator judges 12 pairs, yielding 96 comparisons across eight distinct responses.

Annotators mark 39/96 comparisons as equally good. On the 57 comparisons with an expressed preference, that preference agrees with the direction of \texttt{AnswerCorrectness} in 41 cases. Table~\ref{tab:answer-judge-outcomes} summarizes these outcomes.

\begin{table}[h]
\centering
\small
\caption{Agreement between human answer-quality preferences and
\texttt{AnswerCorrectness}.}
\label{tab:answer-judge-outcomes}
\begin{tabular}{lrrr}
\toprule
 & \textbf{Agree} & \textbf{Disagree} & \textbf{Total} \\
\midrule
Human preference vs. LLM judge ranking & \textbf{41} & 16 & 57 \\
\bottomrule
\end{tabular}
\end{table}

On non-tied comparisons, human preferences and \texttt{AnswerCorrectness} agree at a rate of $41/57=0.719$ (95\% CI $[0.59,0.82]$). This provides evidence that the LLM judge tracks a human-recognizable notion of answer quality when a annotator perceives a difference. 

\section{CondAmbigQA Prompt Library and Episode Lifecycle}
\label{sec:condambigqa_prompts}

This section presents the full set of prompt templates used in the ambiguous question-answering domain, together with descriptions of the episode lifecycle and user simulator. Listings show the verbatim templates with run-time values substituted by \texttt{\{placeholders\}}.

\subsection{Episode Lifecycle}
\label{sec:condambigqa_lifecycle}

Each episode presents one ambiguous query $m_1$ whose latent intent $\theta$ is one of several valid scoping conditions. At each turn the agent updates its particle belief $\hat{b}_t = \{(\theta^n, w^n)\}$ over interpretations (Section~\ref{sec:belief-updating}) and chooses, via the cost-adjusted VoI rule of Section~\ref{sec:decision}, between the \emph{ask} policy $\pi^{\mathrm{ask}}$ (pose a clarifying question $m^{\mathrm{ask}}$, receive a user answer $m^{\mathrm{ans}}$) and the \emph{act} policy $\pi^{\mathrm{act}}$ (commit to a final answer $m^{\mathrm{act}}$). A hard budget caps the number of clarifying questions per episode. Under \textbf{agent termination}, $m^{\mathrm{act}}$ ends the episode; under \textbf{user termination}, the simulated user may instead issue a correction $m^{\mathrm{cor}}$ and let the agent answer again (Appendix~\ref{sec:condambigqa_user_simulator}). The final answer is scored by the $\mathrm{AnswerCorrectness}$ judge (Appendix~\ref{sec:condambigqa_eval}).

\subsection{User Simulator}
\label{sec:condambigqa_user_simulator}

The user model $\pi^{\mathrm{user}}(a^{\mathrm{user}}_t \mid s_{t-1}, \theta)$ is an LLM grounded in the ground-truth scoping condition $\theta$ and its retrieval context, which the agent never observes.

\paragraph{Answering clarifying questions.}
In response to a question $m^{\mathrm{ask}}$, the user answers strictly from the ground-truth condition, and produces a dismissal when the question cannot be answered from it. This behaviour is shared by both termination variants.
\beigefile[label=prt:condambigqa-user-system]{CondAmbigQA User Simulator System Prompt}{condambigqa_prompts/instruct_user_system.txt}
\beigefile[label=prt:condambigqa-user-message]{CondAmigQA User Simulator Message}{condambigqa_prompts/instruct_user_request.txt}

\paragraph{Termination and correction (user termination).}
Under user termination, after the agent commits to $m^{\mathrm{act}}$ the user first judges whether the answer matches the intended interpretation,
\beigefile[label=prt:condambigqa-user-term-system]{CondAmbigQA User Simulator Termination System Prompt}{condambigqa_prompts/satisfied_system.txt}
\beigefile[label=prt:condambigqa-user-term-message]{CondAmbigQA User Simulator Termination Message}{condambigqa_prompts/satisfied_request.txt}
and, if unsatisfied and the clarification budget remains, issues a one-sentence correction $m^{\mathrm{cor}}$ that steers the agent toward the correct interpretation without revealing the answer:
\beigefile[label=prt:condambigqa-user-corr-system]{CondAmbigQA User Simulator Correction System Prompt}{condambigqa_prompts/clarifying_user_system.txt}
\beigefile[label=prt:condambigqa-user-corr-message]{CondAmbigQA User Simulator Correction Message}{condambigqa_prompts/clarifying_user_request.txt}

\subsection{Evaluation Judge and Planning-Time Reward Proxies}
\label{sec:condambigqa_eval}

The evaluation metric (i.e. true goal reward) for CondAmbigQA is $G(s,\theta) = \texttt{AnswerCorrectness}(m^{\mathrm{act}}, y_\theta)$, scored against the reference answer $y_\theta$ for the true condition $\theta$ by a rubric-based G-Eval judge. The rubric combines a \emph{contradiction check}, an \emph{omission penalty}, a \emph{relevance check}, and an \emph{acumen reward} for relevant answers that add correct depth beyond the extractive $y$. The judge is implemented with a DeepEval \texttt{GEval} object, producing a score in $[0, 1]$ based on criteria and evaluation steps registration.

\beigefile[label=prt:condambigqa-llmjudge-geval]{CondAmbigQA LLM-judge G-Eval Registration:}{condambigqa_prompts/answer_correctness_geval.txt}

\paragraph{Planning-time reward proxy $\hat{G}_{\mathrm{LLM}}$.}
At planning time no reference answer $y_\theta$ is available, so REVOIR scores a candidate answer against each hypothesised intent $\theta^n$ with the proxy $\hat{G}_{\mathrm{LLM}}(a^{\mathrm{act}}, \theta^n)$ used to form $\bar{G}(s,\hat{b}_t,a^{\mathrm{act}}) = \sum_n w^n \hat{G}_{\mathrm{LLM}}(a^{\mathrm{act}}, \theta^n)$ (Appendix~\ref{sec:revoir-condambigqa}). Instead of using DeepEval, which requires multiple LLM calls for each scoring, the proxy employs directly prompted verbalized scoring to produce a $0$--$10$ score (to avoid complications with decimal tokens) that gets normalized to $[0,1]$. It mirrors the answer-correctness rubric but drops the contradiction check (which requires $y$). Template~\ref{}
\beigefile[label=prt:condambigqa-proxy-reward-system]{CondAmbigQA LLM-judge Proxy Reward System Prompt}{condambigqa_prompts/answer_for_condition_judge_system.txt}
\beigefile[label=prt:condambigqa-proxy-reward-message]{CondAmbigQA LLM-judge Proxy Reward Message}{condambigqa_prompts/answer_for_condition_judge_user.txt}

\subsection{REVOIR and REIGN Prompt Templates}
\label{sec:condambigqa_revoir_prompts}

REVOIR, REIGN, and Entropy Thresholding share the templates below; they differ only in the reward used in the decision rule --- task reward $\hat{G}_{\mathrm{LLM}}$ for REVOIR, belief concentration (negative entropy) for REIGN, and Entropy Thresholding --- not in their prompts. Each template realizes a component of the algorithm in Appendix~\ref{sec:revoir-condambigqa}.

\paragraph{Hypothesis proposer $Q_{\mathrm{LLM}}(\theta \mid h_t, \hat{b}_{t-1})$.}
Proposes the particle set of interpretation hypotheses, conditioned on the dialogue and (after the first turn) the current belief, so that high-probability hypotheses are refined and unlikely ones pruned (Section~\ref{sec:belief-updating}).
\beigefile[label=prt:condambigqa-hyp-system]{CondAmbigQA Hypothesis Proposer System Prompt}{condambigqa_prompts/hyp_system.txt}
\beigefile[label=prt:condambigqa-hyp-message]{CondAmbigQA Hypothesis Proposer Message}{condambigqa_prompts/hyp_request.txt}

\paragraph{Consistency scorer $S_{\mathrm{LLM}}(\theta^n, h_t)$.}
The semantic log-likelihood that drives the belief update (Section~\ref{sec:belief-updating}) is computed by scoring the user response to a clarifying question for consistency against each hypothesis on a $0$--$10$ scale, yielding the weights $w^n$.
\beigefile[label=prt:condambigqa-consistency-score-system]{CondAmbigQA Intent-Dialogue Consistency Scoring System Prompt}{condambigqa_prompts/score_system.txt}
\beigefile[label=prt:condambigqa-consistency-score-message]{CondAmbigQA Intent-Dialogoue Consistency Scoring Message}{condambigqa_prompts/score_request.txt}

\paragraph{User clarifying response forecast.}
At planning time, in order to score the value of each \textit{question}, the agent role-plays a user holding a hypothesis $\theta^n$ and forecasts a concise response to the question.
\beigefile[label=prt:condambigqa-forecast-system]{CondAmbigQA User Response Forecast System Prompt}{condambigqa_prompts/forecast_system.txt}
\beigefile[label=prt:condambigqa-forecast-message]{CondAmbigQA User Response Forecast Message}{condambigqa_prompts/forecast_request.txt}

\paragraph{User correction forecast (user-termination mode).}
To score the value of each \textit{answer} at planning time, the agent role-plays a user holding a hypothesis $\theta^n$ and forecasts a corresponding post-hoc correction. For this, we directly reuse Template~\ref{prt:condambigqa-user-corr-system} and Template~\ref{prt:condambigqa-user-corr-message}, conditioned the hypothesized intent $\theta^{n}$ with the evaluating answer $m^{\mathrm{act}}$ injected into the history.

\paragraph{Question generator $\pi^{\mathrm{ask}}$ (best-of-$K$).}
Generates the $K$ candidate clarifying questions whose VoI is evaluated by the 1-step lookahead. The model first reasons about what information is still needed, then proposes $K$ distinct questions conditioned on the belief summary.
\beigefile[label=prt:condambigqa-qg-system]{CondAmbigQA Question Generator System Prompt}{condambigqa_prompts/qg_system.txt}
\beigefile[label=prt:condambigqa-qg-message]{CondAmbigQA Question Generator Message}{condambigqa_prompts/qg_request.txt}

\paragraph{Answer generator $\pi^{\mathrm{act}}$.}
When the agent commits to an answer, a single response is generated conditioned on the full belief summary (the best-of-$K$ act policy with $K{=}1$); its expected reward $\bar{G} = \sum_n w^n \hat{G}_{\mathrm{LLM}}(a^{\mathrm{act}},\theta^n)$ supplies the value of acting.
\beigefile[label=prt:condambigqa-answer-system]{CondAmbigQA Answer Generator System Prompt}{condambigqa_prompts/answer_system.txt}
\beigefile[label=prt:condambigqa-answer-message]{CondAmbigQA Answer Generator Message}{condambigqa_prompts/answer_request.txt}

\subsection{ReAct and SC-BoN-ReAct}
\label{sec:condambigqa_react_prompts}

The ReAct baseline interleaves a one-sentence \texttt{Thought} with an \texttt{Action} that is either \texttt{Ask~|~\textless question\textgreater} or \texttt{Answer~|~\textless answer\textgreater}, maintaining no explicit belief. The system prompt differs by termination variant; the user-termination variant adds that the user may correct wrong answers, encouraging earlier commitment.
\beigefile[label=prt:condambigqa-react-at-system]{CondAmbigQA ReAct Agent System Prompt (Agent Termination Mode)}{condambigqa_prompts/react_system_tmpl.txt}
\beigefile[label=prt:condambigqa-react-ut-system]{CondAmbigQA ReAct Agent System Prompt (User Termination Mode)}{condambigqa_prompts/react_system_mode_b_tmpl.txt}

\paragraph{Dialogue scaffolding.}
The initial user turn presents the ambiguous query; user answers are injected as observation turns; an exhausted question budget triggers a forcing message; under user termination, corrections are injected as clarification observations.
\beigefile[label=prt:condambigqa-react-initial-message]{CondAmbigQA ReAct Agent Initial Message}{condambigqa_prompts/react_first_user_tmpl.txt}
\beigefile[label=prt:condambigqa-react-message]{CondAmbigQA ReAct Agent Message}{condambigqa_prompts/react_observation_tmpl.txt}
\beigefile[label=prt:condambigqa-react-budget-message]{CondAmbigQA ReAct Agent Exhausted Budget Trigger Message}{condambigqa_prompts/react_budget_exhausted.txt}
\beigefile[label=prt:condambigqa-react-corr-message]{CondAmbigQA ReAct Agent User Correction Message}{condambigqa_prompts/react_clarification_observation_tmpl.txt}

\paragraph{SC-BoN-ReAct selection judges.}
SC-BoN-ReAct requests $K$ samples from ReAct completions at generation-time, decides the macro-action (ask vs.\ answer) by self-consistency vote, and selects the best candidate among the $N \leq K$ winning-type actions with an LLM judge. The two selection judges receive all same-type candidates with the question and dialogue history and return the best index.
\beigefile[label=]{SC-BoN-ReAct Question Selection System Prompt}{condambigqa_prompts/bon_ask_system.txt}
\beigefile[label=]{SC-BoN-ReAct Question Selection Message}{condambigqa_prompts/bon_ask_user.txt}
\beigefile[label=]{SC-BoN-ReAct Answer Selection System Prompt}{condambigqa_prompts/bon_answer_system.txt}
\beigefile[label=]{SC-BoN-ReAct Answer Selection Message}{condambigqa_prompts/bon_answer_user.txt}

\subsection{ReflectionDPO-ReAct}
\label{sec:condambigqa_dpo_prompts}

ReflectionDPO-ReAct uses the ReAct templates above with a finetuned model; the data-generation pipeline is described in Appendix~\ref{sec:reflection-dpo}. The privileged reflection template below --- which appends the ground-truth answer and asks for the \texttt{Answer} action that should have been produced --- is used both to generate the \textsc{Answer}-type training signal and, at inference time, by the Oracle-ReAct teacher.
\beigefile[label=prt:condambigqa-oracle-reflection-message]{CondAmbigQA Privileged Reflection Message}{condambigqa_prompts/answer_reflection_user_tmpl.txt}

\subsection{Non-Interactive Baselines}
\label{sec:condambigqa_non_interactive_prompts}

\textbf{Direct Answer} answers the ambiguous query in a single turn under full ambiguity with a minimal system prompt:
\beigefile[label=prt:condambigqa-plain-system]{CondAmbigQA Direct Answer System Prompt}{condambigqa_prompts/plain_system.txt}
\textbf{Oracle-ReAct} (the ReflectionDPO teacher and non-interactive topline) is given the ground-truth answer $y$ at every call via the reflection template above and rephrases it without asking. 

\textbf{Oracle} reference submits $y$ verbatim with no LLM call, serving as a metric ceiling.

\section{CondAmbigQA Parameters}
\label{sec:condambigqa_hyperparameters}

We detail the hyperparameters of the CondAmbigQA experiments and the assignment of LLMs to the modular roles of REVOIR.

\subsection{Environment Roles}
\label{sec:condambigqa_roles}

As noted in the main text, REVOIR's reasoning components --- the hypothesis proposer ($Q_{\textrm{LLM}}$), the consistency scorer $S_{\mathrm{LLM}}$, the question generator ($\pi^{\mathrm{ask}}$), the answer generator ($\pi^{\mathrm{act}}$), and the judge/scorer ($\hat{G}_{\mathrm{LLM}}$, $S_{\mathrm{LLM}}$) --- can each be instantiated with a different model, yet we opt for a unifed backbone across experiments to isolate the effect of the method. Other LLM-based modules required in the experimental environment are the user simulator $\pi^{\mathrm{user}}$ and the post-episode evaluation judge $G$. Together, our implementation groups these into three configurable roles:
\begin{itemize}[leftmargin=*]
    \item \textbf{Agent model:} drives the reasoning process of the assistive agent.
    \item \textbf{User model:} drives the user simulator $\pi^{\mathrm{user}}$ (clarification answers, satisfaction checks, post-hoc corrections).
    \item \textbf{Judge model:} drives the post-episode $\mathrm{AnswerCorrectness}$ evaluation.
\end{itemize}
In our reported experiments:
\begin{itemize}[leftmargin=*]
    \item \textbf{Agent model:} Inference-time methods (REVOIR, REIGN, Entropy Thresholding, ReAct, SC-BoN-ReAct) are evaluated with three models: \texttt{gpt-5.4-mini-2026-03-17}, \texttt{meta-llama/Llama-3.1-70B-Instruct}, and \texttt{meta-llama/Llama-3.1-8B-Instruct}. \\ReflectionDPO-ReAct is finetuned from \texttt{meta-llama/Llama-3.1-8B-Instruct}.
    \item \textbf{Reasoning effort (\texttt{gpt-5.4-mini} only):} The VoI agents (REVOIR, REIGN, Entropy Thresholding) and SC-BoN-ReAct use \texttt{reasoning\_effort:\ none}; for ReAct, Oracle-ReAct, and Direct Answer sweep effort over \{none, low, medium, high, xhigh\}.
    \item \textbf{User and judge models:} \texttt{gpt-4o-mini-2024-07-18} throughout, for both the user simulator and the evaluation judge.
\end{itemize}

\subsection{Global Episode Parameters}
\label{sec:condambigqa_global_params}

Shared across all interactive methods:
\begin{itemize}[leftmargin=*]
    \item \textbf{Question budget:} $10$ clarifying questions per episode.%
    \item \textbf{Clarification budget:} under user termination, the user issues at most $5$ corrections after unsatisfactory answers.%
    \item \textbf{Word budgets $(B_{\mathrm{agent}}, B_{\mathrm{user}})$:} the piecewise-linearly  word-budgeted cost function $C$ defined in Section~\ref{sec:condambigqa-exp} penalizes agent and user messages beyond per-message word budgets. The main REVOIR and REIGN runs use $(B_{\mathrm{agent}}, B_{\mathrm{user}}) = (100, 50)$; for Llama-3.1-8B we report a budget sweep with $B_{\mathrm{agent}} \in \{100, 150, \ldots, 400\}$ and $B_{\mathrm{user}} = B_{\mathrm{agent}}/2$. Budgets are disabled ($0$) for Entropy Thresholding and do not constrain generation directly.
    \item \textbf{Agent generation:} all expert-model generation --- hypothesis and question proposal, mentalized forecasting, consistency scoring, and answer generation --- uses \texttt{temperature: 0.0} (greedy), with \texttt{agent\_max\_tokens: 2048}; candidate hypotheses and questions are produced as a single list per call rather than by repeated sampling.
    \item \textbf{Macro-action selection:} $\arg\max$ over $V^{\mathrm{ask}}$ vs.\ $V^{\mathrm{act}}$ (greedy) %
    \item \textbf{User generation:} greedy decoding with \texttt{user\_temperature: 0.0}, \texttt{user\_top\_p: 1.0}, and \texttt{user\_max\_tokens: 128}.
    \item \textbf{Judge evaluation:} the post-episode AnswerCorrectness metric uses DeepEval's \texttt{GEval}. The underlying LLM is calle at \texttt{temperature: 0} and \texttt{top\_logprobs: 20}. The metric forms a probability-weighted expectation of the integer scores over the (default) $0$--$10$ range and normalizes it to $[0,1]$.
\end{itemize}

\subsection{VoI Planner Parameters (REVOIR, REIGN, Entropy Thresholding)}
\label{sec:condambigqa_pomdp_params}

\begin{itemize}[leftmargin=*]
    \item \textbf{Belief particles:} $N$=5 interpretation hypotheses in the belief $\hat{b}_t$ (\texttt{n\_hypotheses}).
    \item \textbf{Question generation:} $K=5$ candidate questions generated and evaluated by the best-of-$K$ ask policy per turn.%
    \item \textbf{Act policy:} a single belief-conditioned answer, i.e.\ best-of-$K$ with $K{=}1$, whose expected reward $\bar{G} = \sum_n w^n \hat{G}_{\mathrm{LLM}}$ supplies $V^{\mathrm{act}}$ (\texttt{optimal\_answer\_approx:\ direct\_belief\_condition}).
\end{itemize}
\subsection{Entropy Thresholding}
\begin{table}[htbp] %
    \centering
    \begin{minipage}[c]{0.50\linewidth}
        \captionof{table}{The normalized concentration threshold $\tau$ of $1 - H(\hat{b}_t)/H_{\max}$, above which Entropy Thresholding answers immediately, is tuned on the Dev split, per agent model and termination variant.}
        \label{tab:entropy_threshold}
    \end{minipage}
    \hfill %
    \begin{minipage}[c]{0.46\linewidth}
        \small
        \begin{tabular*}{\linewidth}{@{\extracolsep{\fill}}lcc}
        \toprule
        \textbf{Configuration} & \shortstack{\textbf{Agent}\\\textbf{term.}} & \shortstack{\textbf{User}\\\textbf{term.}} \\
        \midrule
        \texttt{gpt-5.4-mini} & $0.85$ & $0.65$ \\
        \texttt{Llama3-70B}            & $0.25$ & $0.85$ \\
        \texttt{Llama3-8B}             & $0.65$ & $0.45$ \\
        \bottomrule
        \end{tabular*}
    \end{minipage}

\end{table}
\subsection{SC-BoN-ReAct Parameters}
\label{sec:condambigqa_scbon_params}

\begin{table}[htbp] %
\centering

    \begin{minipage}[c]{0.50\linewidth}
    \begin{itemize}[leftmargin=*]
        \item \textbf{Self-consistency ensemble size:} $K=5$ in the main comparisons; for Llama-3.1-8B we additionally sweep $K \in \{5, 10, 15, 20\}$.
        \item \textbf{Sampling temperature:} $0.7$ for the pool, required to be $>0$ for diversity (\texttt{sample\_temperature}); the best-of-$N$ selection judge runs greedily.
        \item \textbf{Self-consistency threshold $\tau$:} the fraction of \texttt{Answer} votes (over parseable samples) required to commit to answering rather than asking, tuned on the Dev split per agent model and termination variant. In the Llama-3.1-8B sweep, $\tau$ is additionally tuned per $K$ (Table~\ref{tab:sc_threshold}).
    \end{itemize}
    \end{minipage}
    \hfill %
    \begin{minipage}[c]{0.46\linewidth}
        \centering
        \small
        
        \captionof{table}{Self-consistency threshold $\tau$ for SC-BoN-ReAct, tuned on the Dev split.}
        \label{tab:sc_threshold}
        
        \vspace{0.5em}
        \begin{tabular*}{\linewidth}{@{\extracolsep{\fill}}lcc}
        \toprule
        \textbf{Configuration} & \shortstack{\textbf{Agent}\\\textbf{term.}} & \shortstack{\textbf{User}\\\textbf{term.}} \\
        \midrule
        \texttt{gpt-5.4-m}, $K{=}5$    & $0.9$ & $0.6$ \\
        \texttt{Llama3-70B}, $K{=}5$   & $0.6$ & $0.7$ \\
        \texttt{Llama3-8B}, $K{=}5$    & $0.6$ & $0.6$ \\
        \texttt{Llama3-8B}, $K{=}10$   & $0.5$ & $0.6$ \\
        \texttt{Llama3-8B}, $K{=}15$   & $0.6$ & $0.5$ \\
        \texttt{Llama3-8B}, $K{=}20$   & $0.7$ & $0.7$ \\
        \bottomrule
        \end{tabular*}
    \end{minipage}

\end{table}

\section{Adapting ReflectionDPO to Ambiguous Question Answering}
\label{sec:reflection-dpo}

We adapt the ReflectionDPO algorithm of \citep{patel_adapt_2025} to CondAmbigQA. The training data is curated by dynamically comparing a standard ReAct policy against a privileged oracle. First, the student ReAct model rolls out its interaction with the user simulator. Unlike ADAPT, where the reference actions at each step are entirely generated by a teacher conditioned on the partial history, CondAmbigQA provides extractive ground-truth answers --- the ``true plan'' degenerates into a known answer $y$, which we use directly as the oracle signal.

At each conversational turn, the student is presented with the partial history together with a ReAct-formatted oracle action: the \texttt{Thought} block contains the ground-truth clarification question and its triggering condition, while the \texttt{Answer} block contains the ground-truth extractive answer. The student is then prompted to \emph{reflect} on what question it could have asked to elicit the missing information required to arrive at this oracle answer. This process yields three components per turn: the student's original rollout action, the oracle's answer, and the student's candidate reflection question. The conditional log-probability of generating the ground-truth answer is also leveraged as the utility score for preference-pair filtering.

These components are filtered to construct preference pairs. The student's original rollout is assigned as the rejected response, and either the reflection question (\textsc{Ask}) or the oracle answer (\textsc{Answer}) is selected as the chosen response. However, because the oracle answer $a_{\mathrm{oracle}}$ is typically verbose, its generative log-probability under the student is naturally low relative to the rollout; directly performing preference alignment on rollout--oracle pairs risks degrading the model's general fluency. To mitigate this, we prompt the Oracle-ReAct model --- privileged to the oracle extractive answer --- to \emph{repackage} $a_{\mathrm{oracle}}$ in its own words, and use that rephrased version $a_{\mathrm{teacher}}$ as the preferred response whenever \textsc{Answer} is the chosen action.

\begin{algorithm}
\caption{Reflection Mechanism}
\begin{algorithmic}[1]
\Function{Reflection}{$P_{\pi_{\mathrm{student}}}(\cdot),\, a_{\mathrm{student}},\, a_{\mathrm{oracle}},\, h_t$}
    \State $a_{\mathrm{teacher}} \gets P_{\pi_{\mathrm{teacher}}}(\cdot \mid h_t, a_{\mathrm{oracle}})$ \Comment{Oracle-ReAct rephrases extractive answer}
    \State $a_q \gets \textsc{GetQuestion}(P_{\pi_{\mathrm{student}}}, a_{\mathrm{student}}, a_{\mathrm{oracle}})$
    \State $r \gets \textsc{AskUser}(h_t, a_q)$
    \State $\Delta_q \gets \log P_{\pi_{\mathrm{student}}}(a_{\mathrm{oracle}} \mid h_t \oplus a_q \oplus r) - \log P_{\pi_{\mathrm{student}}}(a_{\mathrm{oracle}} \mid h_t)$
    \State $\Delta_t \gets \log P_{\pi_{\mathrm{student}}}(a_{\mathrm{student}}) - \log P_{\pi_{\mathrm{student}}}(a_{\mathrm{teacher}})$
    \If{$\Delta_t < 0$}
        \State $a_{\mathrm{chosen}} \gets a_{\mathrm{teacher}}$ \Comment{teacher already more likely than rollout}
    \ElsIf{$\Delta_q > \varepsilon_1$}
        \State $a_{\mathrm{chosen}} \gets a_q$ \Comment{reflection meaningfully helps}
    \ElsIf{$\Delta_t < \varepsilon_2$}
        \State $a_{\mathrm{chosen}} \gets a_{\mathrm{teacher}}$ \Comment{rollout close enough to teacher}
    \Else
        \State $a_{\mathrm{chosen}} \gets \text{None}$ \Comment{skip this data point}
    \EndIf
    \State \Return $a_{\mathrm{chosen}}$
\EndFunction
\end{algorithmic}
\end{algorithm}

Two scalar thresholds govern which (episode, step) records become training pairs. $\varepsilon_1$ gates \textsc{Ask} inclusion: a step produces an \textsc{Ask} pair only when the oracle answer log-probability improves by more than $\varepsilon_1$ nats per token after the model receives the reflection question, i.e.\ $\Delta_q > \varepsilon_1$. A lower $\varepsilon_1$ admits more \textsc{Ask} pairs but at the cost of including weak reflection signals. $\varepsilon_2$ gates \textsc{Answer} inclusion: a step produces an \textsc{Answer} pair when the rollout is worse than the teacher answer by less than $\varepsilon_2$ nats per token, i.e.\ $\Delta_t < \varepsilon_2$. A higher $\varepsilon_2$ admits more \textsc{Answer} pairs but risks including teacher answers that are noticeably less fluent than the rollout.

We set a quality ceiling of $\varepsilon_2 \leq 0.30$ based on the empirical distribution of $\Delta_t$: at $\varepsilon_2 = 0.30$ the mean gap is $\approx 0.19$ nats/token, indicating the teacher answer can still reasonably be preferred; beyond 0.50 the mean exceeds 0.45 nats/token, a regime where teacher answers are markedly less natural than rollouts. We impose $\varepsilon_1 \geq 0.05$ to exclude near-zero reflection gains that provide no meaningful learning signal, and cap the \textsc{Ask}:\textsc{Answer} ratio at 4:1 to prevent the policy from becoming reflexively question-happy.

We sweep $\varepsilon_1 \in [0.01, 0.80]$ and $\varepsilon_2 \in [0.02, 0.30]$ on the full reflection corpus (17,090 records) and select the point that maximises total kept pairs subject to the constraints above.

\paragraph{Operating points.}

\begin{center}
\small
\begin{tabular}{lccccccc}
\toprule
\textbf{Setting} & $\varepsilon_1$ & $\varepsilon_2$ & \textbf{Filtered} & \makecell[l]{\textbf{Ask:Ans}\\\textbf{(filtered)}} & \makecell[l]{\textbf{Cap}\\\textbf{$C$}} & \makecell[l]{\textbf{Strat.}\\\textbf{train}} & \makecell[l]{\textbf{Ask:Ans}\\\textbf{(strat.)}} \\
\midrule
Ask-heavy & 0.09 & 0.30 & 5649 & $3.6\,{:}\,1$ & 376 & 2669 & $3.8\,{:}\,1$ \\
Ask-moderate       & 0.13 & 0.30 & 3366 & $1.5\,{:}\,1$ & 218 & 1547 & $1.7\,{:}\,1$ \\
\bottomrule
\end{tabular}
\end{center}

\noindent The \emph{ask-heavy} setting maximises dataset size within the ratio constraint; the \emph{ask-moderate} setting tightens $\varepsilon_1$ to equalise the \textsc{Ask} and \textsc{Answer} signal, trading coverage for a more symmetric preference distribution. Both settings use $\varepsilon_2 = 0.30$ (the quality ceiling), which is also the global optimum for \textsc{Answer} inclusion under our constraints.

\paragraph{Step-stratified sampling.}
Reflection pairs are not uniformly distributed across conversation steps. Early steps ($t = 0$--$2$) are over-represented because few episodes reach late steps under the default ReAct policy. Training on the raw distribution therefore biases the policy toward early-step behaviour and under-trains it on mid-to-late recovery moves.

To correct for this, we apply step-stratified sampling. Let $\mathcal{D}_t$ denote the set of retained pairs at step $t$ and let $C = \underset{t}{\mathrm{median}}\,|\mathcal{D}_t|$. The stratified training set is
\begin{equation}
  \mathcal{D}^* \;=\; \bigcup_{t=0}^{T}\, \mathrm{Sample}\!\bigl(\mathcal{D}_t,\;\min(|\mathcal{D}_t|,\, C)\bigr),
\end{equation}
where each step is independently sampled without replacement up to $C$ pairs. Steps with fewer than $C$ pairs are kept in full; over-represented steps are downsampled to exactly $C$. Within each step, the \textsc{Ask}:\textsc{Answer} ratio is preserved proportionally so that the relative preference signal at each turn is not distorted.

Using the median rather than the minimum prevents the long tail of rare late steps from collapsing the dataset to a trivially small size, while still substantially flattening the step distribution, which shifts the global ratio only slightly (ask-heavy $3.6\,{:}\,1 \to 3.8\,{:}\,1$; ask-moderate $1.5\,{:}\,1 \to 1.7\,{:}\,1$).

\paragraph{Training details.}
We finetune using LoRA~\citep{hu_lora_2022} with rank $r = 4$ and scaling factor $\alpha = 16$, matching the adapter configuration of the original ADAPT codebase~\citep{patel_adapt_2025}.

We deviate from ADAPT in the choice of preference optimization loss. ADAPT uses DPO~\citep{rafailov_direct_2023}, which is appropriate when chosen and rejected responses are short, atomic actions: the log-probability differences are well-calibrated and the relative contrastive objective is stable. In our setting, both chosen responses (rephrased answers or reflection questions) and rejected responses (student rollout actions) are elaborate, multi-sentence texts. DPO applied to long-form outputs is susceptible to reward hacking via likelihood displacement: the optimizer can satisfy the contrastive objective by reducing the log-likelihood of \emph{both} responses while keeping their relative ordering intact, which degrades generation quality without meaningfully improving preference alignment.

We therefore use ORPO~\citep{hong_orpo_2024}, which folds a negative log-likelihood term on the chosen response together with an odds-ratio penalty on the rejected response into a single objective, without requiring a reference model or a separate SFT warm-up phase:
\begin{equation}
    \mathcal{L}_{\mathrm{ORPO}} = -\log P_\theta(y_w \mid x) \;-\; \lambda \cdot \log \sigma\!\left(\log \frac{P_\theta(y_w \mid x)}{1 - P_\theta(y_w \mid x)} - \log \frac{P_\theta(y_l \mid x)}{1 - P_\theta(y_l \mid x)}\right),
\end{equation}
where $y_w$ and $y_l$ are the chosen and rejected responses respectively and $\lambda$ is a weighting coefficient. The NLL term anchors the absolute likelihood of chosen responses upward, closing off the reward-hacking mode available to DPO, while the odds-ratio term suppresses rejected responses relative to this moving anchor rather than relative to a frozen reference policy. A natural alternative would be to anchor on the gold extractive answer $y^*$ directly, grounding the model in factual content independent of how well the teacher rephrased it. However, the extractive answers in CondAmbigQA are terse, unnaturally phrased fragments that make poor generation targets; anchoring on $y^*$ would penalize fluent paraphrases and conflict with the fluency goal that motivated using $a_{\mathrm{teacher}}$ in the first place. We therefore anchor on $y_w = a_{\mathrm{teacher}}$ throughout.

\paragraph{Why it underperforms.}
We speculate that the gap stems primarily from a mismatch between the training signal and test-time requirements: reflection questions are selected based on how much they improve the log-probability of the extractive oracle answer $y$, which may not correlate with whether a question improves open-ended generative answer quality. The small post-curation dataset likely compounds this, under-determining a reliable clarification policy rather than overcoming the base model's prior. More broadly, the result illustrates that finetuning a clarification policy does not transfer cheaply to arbitrary tasks and contexts as it demands task-specific data curation and can still optimize a proxy that diverges from the deployment objective. REVOIR, on the other hand, obtains its clarification behavior at inference time, with no training, by reasoning directly about the value of information for the task at hand.

\section{Extended Details on ADAPT}
\label{sec:adapt_details}
\subsection{Adapting ADAPT to REVOIR}
\label{sec:adapt-2stage}
\textbf{Two-stage Variant.} The original ADAPT runtime freely interleaves questions and physical actions, making one-step VoI reasoning difficult to apply. We therefore decouple each episode into (i) an \textbf{elicitation stage}, in which REVOIR iteratively compares $V^{\mathrm{ask}}(s, b)$ against $V^{\mathrm{act}}(s, b)$ and either poses a clarifying question $m^{\mathrm{ask}} \in \mathcal{M}^{\mathrm{ask}}$ or commits to execution, and (ii) an \textbf{execution stage}, in which $\pi^{\mathrm{act}}$ carries out a full multi-step plan conditioned on the updated belief $b_t(\theta)$, after which the episode terminates. We note that this two-phase decoupling \emph{limits} the information gathering capacity of REVOIR, while preserving the full complexity of grounded planning in ADAPT.

\subsection{REVOIR Variants in ADAPT}~\label{sec:revoir-adapt}
The ADAPT domain introduces two additional design choices in how REVOIR maintains and updates its belief $\hat{b}_t$, yielding four variants reported in our experiments.

\textbf{Belief representation.}
In the \emph{joint belief} variant, each particle $\theta^n$ is a candidate set of atomic preference statements drawn from the space of plausible user preferences for the task. The belief $\hat{b}_t$ is a distribution over these complete preference sets, and the consistency scorer $S_{\mathrm{LLM}}(\theta^n, h_t)$ evaluates how well the full set $\theta^n$ explains the conversation so far ($K = 6$ particles). In the \emph{factored belief} variant, rather than tracking complete sets, the agent maintains an independent Bernoulli probability $p_m \in [0,1]$ for each atomic preference $\gamma_m$, representing the agent's belief that $\gamma_m$ belongs to the user's true preference set $\theta$. This representation supports a larger effective particle count since each particle corresponds to a binary inclusion decision per preference, and enables more targeted clarifying questions directed at individual uncertain dimensions of $\theta$.

\textbf{Weight update scheme.}
In the \emph{batch rescoring} variant, particle weights are recomputed from scratch at each turn by scoring the full conversation history $h_t$ against each particle in a single pass: $w^n_t = S_{\mathrm{LLM}}(\theta^n, h_t)$. In the \emph{incremental update} variant, retained particles cache their accumulated log-weight from the previous turn and augment it with only the new message's contribution, $w^n_t = w^n_{t-1} + S_{\mathrm{LLM}}(\theta^n, h_t \mid h_{t-1})$, avoiding redundant recomputation. Particles that are pruned and replaced by fresh samples cannot benefit from this cache; they instead replay the full sequential accumulation from the start of the conversation to bring their weights up to date. 

\textbf{Goal reward approximation.}
Because plan execution in ADAPT is multi-step, evaluating $\bar{G}(s, b, a^{\mathrm{act}})$ requires simulating the consequences of a full plan rather than a single action. REVOIR approximates this through a pipeline of three specialized modules --- a \emph{planner}, a \emph{simulator}, and a \emph{proxy judge} --- each of which can in principle be instantiated with different models. First, the planner generates a single plan conditioned on the current belief summary and history. Second, the simulator takes the plan and the initial scene graph and produces an estimated final state $s'$ by stepping through the plan's state changes. Third, the proxy judge evaluates $s'$ against the current belief to produce a reward estimate. In our experiments, all three modules are instantiated with the same LLM, but the modular decomposition allows each to be replaced independently --- for instance, with a specialized simulator or a finetuned reward model.

The reward aggregation in the judge differs between belief variants. Under \emph{joint belief}, the judge scores the PSR of $s'$ against each preference set particle $\theta^n$ and the expected reward is the weighted average $\bar{G}(s, b, a^{\mathrm{act}}) = \sum_n w^n \cdot \mathrm{PSR}(s', \theta^n)$. Under \emph{factored belief}, a mean-field weighted reward is used instead:
\begin{equation}
    \bar{G}(s, b, a^{\mathrm{act}}) = \frac{\sum_m p_m \cdot \mathbf{1}[\gamma_m \text{ satisfied in } s']}{\sum_m p_m \cdot \mathbf{1}[\gamma_m \text{ satisfied or violated in } s']},
\end{equation}
where $p_m$ is the Bernoulli inclusion probability of preference $\gamma_m$. This weights each preference's contribution by its posterior probability of belonging in $\theta$, so preferences the agent is confident the user holds influence the reward estimate more than uncertain ones.

Combining the two belief representations with the two update schemes yields the four REVOIR variants --- \textbf{Joint~+~Batch Rescoring}, \textbf{Joint~+~Incremental Update}, \textbf{Factored~+~Batch Rescoring}, and \textbf{Factored~+~Incremental Update}.

\subsection{Full Results on ADAPT}\label{sec:adapt_full_results}

Full results for ADAPT are presented in Table~\ref{tab:adapt}.

\begin{table}[ht]
\footnotesize
\centering
\renewcommand{\arraystretch}{1.3}
\begin{tabular}{clcc}
\toprule
\shortstack{\textbf{Test}\\\textbf{Personas}} & \textbf{Method} & \textbf{PSR (\%)} & \textbf{\#Q} \\
\midrule
\multicolumn{4}{l}{\textit{Privileged / Forced}} \\
\multirow{2}{*}{\shortstack{Seen\\Personas}} 
& Always-Ask LLM & 52.1 $\pm$ 0.9 & 22.2 $\pm$ 0.1 \\
& Teacher LLM & 65.5 $\pm$ 1.1 & 0.0 \\
\cline{2-4}
\multirow{2}{*}{\shortstack{Unseen\\Personas}}
& Always-Ask LLM & 50.8 $\pm$ 1.6 & 22.4 $\pm$ 0.2 \\
& Teacher LLM & 65.5 $\pm$ 1.9 & 0.0 \\
\hline
\multicolumn{4}{l}{\textit{Fine-tuned}} \\
\multirow{2}{*}{\shortstack{Seen\\Personas}}
& STaR-GATE & 34.1 $\pm$ 0.8 & 2.0 $\pm$ 0.0 \\
& Reflection-DPO & 44.1 $\pm$ 1.0 & 9.8 $\pm$ 0.1 \\
\cline{2-4}
\multirow{2}{*}{\shortstack{Unseen\\Personas}}
& STaR-GATE & 33.5 $\pm$ 1.4 & 2.0 $\pm$ 0.0 \\
& Reflection-DPO & 42.9 $\pm$ 1.6 & 9.5 $\pm$ 0.2 \\
\hline
\multicolumn{4}{l}{\textit{Few-shot ICL}} \\
\multirow{4}{*}{\shortstack{Seen\\Personas}}
& \textbf{REVOIR (Joint + Batch Rescoring)} & 49.8 $\pm$ 2.0 & 2.5 $\pm$ 0.3 \\
& \textbf{REVOIR (Joint + Incremental Update)} & 49.5 $\pm$ 1.9 & 2.4 $\pm$ 0.2 \\
& \textbf{REVOIR (Factored + Batch Rescoring)} & \textbf{56.3 $\pm$ 1.3} & \textbf{1.7 $\pm$ 0.2} \\
& \textbf{REVOIR (Factored + Incremental Update)} & \textbf{58.8 $\pm$ 0.4} & \textbf{2.1 $\pm$ 0.2} \\
\cline{2-4}
\multirow{4}{*}{\shortstack{Unseen\\Personas}}
& \textbf{REVOIR (Joint + Batch Rescoring)} & 48.0 $\pm$ 7.4 & 2.4 $\pm$ 0.5 \\
& \textbf{REVOIR (Joint + Incremental Update)} & 50.2 $\pm$ 7.8 & 2.4 $\pm$ 0.3 \\
& \textbf{REVOIR (Factored + Batch Rescoring)} & \textbf{55.8 $\pm$ 8.1} & \textbf{1.9 $\pm$ 0.2} \\ 
& \textbf{REVOIR (Factored + Incremental Update)} & \textbf{56.0 $\pm$ 3.5} & \textbf{2.0 $\pm$ 0.5} \\ 
\hline
\multicolumn{4}{l}{\textit{Zero-shot / No fine-tuning}} \\
\multirow{7}{*}{\shortstack{Seen\\Personas}} 
& Never-Ask LLM & 27.5 $\pm$ 0.7 & 0.0 $\pm$ 0.0 \\
& Baseline LLM & 34.7 $\pm$ 0.8 & 2.3 $\pm$ 0.0 \\ 
& ReAct & 36.1 $\pm$ 0.8 & 2.3 $\pm$ 0.0 \\
& \textbf{REVOIR (Joint + Batch Rescoring)} & 35.7 $\pm$ 3.8 & 2.2 $\pm$ 0.3 \\ 
& \textbf{REVOIR (Joint + Incremental Update)} & 35.0 $\pm$ 2.1 & 2.3 $\pm$ 0.3 \\ 
& \textbf{REVOIR (Factored + Batch Rescoring)} & \textbf{49.2 $\pm$ 3.1} & \textbf{1.9 $\pm$ 0.4} \\ 
& \textbf{REVOIR (Factored + Incremental Update)} & \textbf{51.5 $\pm$ 2.5} & \textbf{2.0 $\pm$ 0.2} \\ 
\cline{2-4}
\multirow{7}{*}{\shortstack{Unseen\\Personas}}
& Never-Ask LLM & 28.6 $\pm$ 1.2 & 0.0 $\pm$ 0.0 \\ 
& Baseline LLM & 34.3 $\pm$ 1.3 & 2.3 $\pm$ 0.1 \\
& ReAct & 36.8 $\pm$ 1.4 & 2.4 $\pm$ 0.1 \\ 
& \textbf{REVOIR (Joint + Batch Rescoring)} & 36.4 $\pm$ 8.5 & 2.4 $\pm$ 0.2 \\ 
& \textbf{REVOIR (Joint + Incremental Update)} & 33.0 $\pm$ 8.0 & 2.0 $\pm$ 0.1 \\ 
& \textbf{REVOIR (Factored + Batch Rescoring)} & \textbf{51.0 $\pm$ 7.4} & \textbf{2.0 $\pm$ 0.4} \\
& \textbf{REVOIR (Factored + Incremental Update)} & \textbf{49.5 $\pm$ 2.6} & \textbf{2.2 $\pm$ 0.3} \\
\bottomrule
\end{tabular}
\vspace{6pt}
\caption{Main results on the ADAPT benchmark (4-fold cross-validation; std across split means).
\citet{patel_adapt_2025} numbers are reproduced from their Table~1. In-context examples from seen personas injected at hypothesis particle regeneration.
}
\label{tab:adapt}
\end{table}

\paragraph{Factored belief improves zero-shot performance significantly.}
In zero-shot, Joint-Belief REVOIR variants achieve $33.0$--$36.4\%$ average PSR, comparable to ReAct and Baseline LLM, while Factored-Belief variants jump to $49.2$--$51.0\%$, nearly matching Always-Ask LLM ($50.9$--$52.1\%$) with $11\times$ fewer questions ($2$ vs.\ $22$). The factored representation enables more targeted elicitation by tracking each preference independently, rather than scoring complete preference sets that may be poorly calibrated. However, Joint-Belief REVOIR still does well in the few-shot case even on unseen personas ($48.0--50.2$), indicating that joint belief updates can still be calibrated if there are good examples of what preference sets look like.

\paragraph{REVOIR with ICL outperforms fully trained methods.}
With few-shot ICL, Factored REVOIR reaches $56.3$--$58.8\%$ average PSR on seen personas and $55.8$--$56.0\%$ on unseen personas, exceeding Always-Ask by $6$ points and outperforming Reflection-DPO ($44\%$) by $13$--$15$ points --- with ${\sim}5\times$ fewer questions. This is significant given that REVOIR requires no task-specific training, showing in-context examples suffice to ground the beliefs and guide the planner. STaR-GATE, by contrast, fails to exceed the non-interactive Baseline LLM ($34\%$), suggesting its training signal is too sparse to learn effective clarification.

\paragraph{Generalization to unseen personas.}
Factored REVOIR variants show minimal PSR degradation from seen to unseen personas (e.g., $58.8 \to 56.0\%$ for Incremental Update with ICL), with notably lower variance than Joint variants, which exhibit standard deviations of $7$--$8\%$ on unseen personas. This suggests that factoring the belief over individual preferences produces representations that transfer more reliably across persona distributions.

\section{ADAPT Prompt Library and Episode Lifecycle}
\label{sec:adapt_prompts}

This section collects the prompt templates used in the preference-aligned household-assistance domain, together with a concise description of the episode lifecycle and user simulator. Listings show the templates with run-time values substituted by \texttt{\{placeholders\}}.

\subsection{Episode Lifecycle}
\label{sec:adapt_lifecycle}

Each episode pairs a breakfast task with a persona whose hidden preference set $\theta = \{\gamma_1, \dots, \gamma_M\}$ the agent never observes. Following the two-stage formulation, REVOIR runs an \textbf{elicitation stage} --- iteratively comparing $V^{\mathrm{ask}}(s,b)$ against $V^{\mathrm{act}}(s,b)$ to either pose a clarifying question $m^{\mathrm{ask}}$ or commit --- followed by an \textbf{execution stage}, in which $\pi^{\mathrm{act}}$ carries out a multi-step plan conditioned on the elicited belief $\hat{b}_t(\theta)$ without further questions. Episodes are scored by the preference satisfaction rate $G(s,\theta) = \mathrm{PSR}$.

\paragraph{Evaluation splits.}
For comparability with the original ADAPT baselines --- which are trained and then tested for generalization across personas --- we reuse the benchmark's four-way split that crosses persona novelty (seen vs.\ held-out personas) with in-context guidance. Our methods, however, are training-free and operate in only two regimes: \emph{few-shot}, with the ground-truth preference profiles of the seen personas injected as reference examples into the proposer and scorer prompts (shown as \texttt{\{example\_profiles\}} below), or \emph{zero-shot}, with no examples provided. Thus, for REVOIR, the seen-versus-held-out persona settings only reflect whether a tested personat appear among the few-shot examples.

\subsection{User Simulator}
\label{sec:adapt_user_simulator}

Following \citet{patel_adapt_2025}, the user simulator $\pi^{\mathrm{user}}(\cdot \mid h_t, \theta)$ is an LLM role-playing the true persona, prompted with the task, the live scene state, the persona's ground-truth preferences, and the dialogue history. It answers preference questions concisely in first person, deflects questions about object location or availability (encouraging the agent to search), and stays consistent with its prior answers. The simulator only answers questions; ADAPT episodes have no satisfaction check or unprompted correction (cf.\ user termination in CondAmbigQA).
\beigefile[label=prt:adapt-user-sim-system]{ADAPT User Simulator System Prompt}{adapt_prompts/user_sim_system.txt}
\beigefile[label=prt:adapt-user-sim-message]{ADAPT User Simulator Message}{adapt_prompts/user_sim_request.txt}

\subsection{Hypothesis Proposers and Consistency Scorers}
\label{sec:adapt_belief_prompts}

The belief proposer $Q_{\mathrm{LLM}}$ and consistency scorer $S_{\mathrm{LLM}}$ (Section~\ref{sec:belief-updating}) are instantiated differently for the two belief representations of Appendix~\ref{sec:revoir-adapt}. In both, the scorer realizes either \emph{batch rescoring} (the full conversation is rescored each turn) or \emph{incremental update} (per-turn scores are accumulated, with retained particles scoring only new turns and fresh particles replaying history).

\paragraph{Joint belief.}
The proposer generates the candidate preference \emph{profiles} (the particle set), regenerated each turn with the current belief as context so high-probability profiles are refined and unlikely ones dropped:
\beigefile[label=prt:adapt-joint-hyp-system]{ADAPT Joint-belief Hypothesis Proposer System Prompt}{adapt_prompts/hyp_gen_system.txt}
\beigefile[label=prt:adapt-joint-hyp-message]{ADAPT Joint-belief Hypothesis Proposer Message}{adapt_prompts/hyp_gen_request.txt}
The scorer rates the holistic consistency of the whole conversation with each candidate profile on a $0$--$10$ scale, forming the particle weights $w^n$:
\beigefile[label=prt:adapt-joint-consistency-score-system]{ADAPT Joint-Belief-Dialogue Consistency Scoring System Prompt}{adapt_prompts/hyp_score_system.txt}
\beigefile[label=prt:adapt-joint-consitency-score-message]{ADAPT Joint-Belief-Dialogue Consistency Scoring Message}{adapt_prompts/hyp_score_request.txt}

\paragraph{Factored belief.}
The proposer generates the atomic preferences $\gamma_m$ over which independent Bernoulli beliefs $p_m$ are maintained, covering the full breakfast experience and regenerated (belief-guided) each turn:
\beigefile[label=prt:adapt-atomic-hyp-system]{ADAPT Factored-belief Atom Proposer System Prompt}{adapt_prompts/atom_gen_system.txt}
\beigefile[label=prt:adapt-atomic-hyp-message]{ADAPT Factored-belief Atom Proposer Message}{adapt_prompts/atom_gen_request.txt}
Under \emph{batch rescoring}, all atoms are scored jointly against the conversation in one call, each $0$--$10$ score mapped to a log-likelihood-ratio update of $p_m$:
\beigefile[label=prt:adapt-atom-consistency-batch-score-system]{ADAPT Belief Atom-Dialogue Consistency Batch Scoring System Prompt}{adapt_prompts/atom_score_system.txt}
\beigefile[label=prt:adapt-atom-consistency-batch-score-message]{ADAPT Belief Atom-Dialogue Consistency Batch Scoring Message}{adapt_prompts/atom_score_request.txt}
Under \emph{incremental update}, each atom is scored independently (one call per atom) so the assessment of one preference is not influenced by the others:
\beigefile[label=prt:adapt-atom-consistency-ind-score-system]{ADAPT Belief Atom-Dialogue Consistency Individual Scoring System Prompt}{adapt_prompts/single_atom_score_system.txt}
\beigefile[label=prt:adapt-atom-consistency-ind-score-messsage]{ADAPT Belief Atom-Dialogue Consistency Individual Scoring Message}{adapt_prompts/single_atom_score_request.txt}

\subsection{Question Generator \texorpdfstring{$\pi^{\mathrm{ask}}$}{πask}}
\label{sec:adapt_qgen}

Candidate clarifying questions are sampled from a prompt containing the task, scene layout, dialogue history, and the current belief summary --- the candidate profiles (joint) or the atoms with their inclusion probabilities (factored). The system prompt is shared; the two belief modes differ only in how the belief summary is rendered in the user turn.
\beigefile[label=prt:adapt-qgen-system]{ADAPT Question Proposal System Prompt}{adapt_prompts/qgen_joint_system.txt}
\beigefile[label=prt:adapt-joint-qgen-message]{ADAPT Joint-belief-based Question Proposal Message}{adapt_prompts/qgen_joint_request.txt}
\beigefile[label=prt:adapt-factored-qgen-message]{ADAPT Factored-belief-based Question Proposal Message}{adapt_prompts/qgen_factored_request.txt}

\subsection{VoI Lookahead: Response Forecasting}
\label{sec:adapt_forecast}

To estimate $V^{\mathrm{ask}}$, the agent samples hypotheses from the belief (all profiles in the joint variant; sampled atom compositions in the factored variant), forecasts each one's likely answer to a candidate question by role-playing a user holding those preferences, then rescores and re-evaluates the goal reward on the updated belief:
\beigefile[label=prt:adapt-user-forecast-system]{ADAPT Planning-time User Response Forecast System Prompt}{adapt_prompts/forecast_system.txt}
\beigefile[label=prt:adapt-user-forecast-message]{ADAPT Planning-time User Response Forecast Message}{adapt_prompts/forecast_request.txt}

\subsection{Goal-Reward Pipeline (Planner--Simulator--Proxy Judge)}
\label{sec:adapt_reward_pipeline}

The expected goal reward $\bar{G}(s, b, a^{\mathrm{act}})$ --- which serves as the elicitation stop value $V^{\mathrm{act}}$ and the quantity re-evaluated inside the VoI lookahead --- is computed by the modular planner--simulator--proxy judge pipeline of Appendix~\ref{sec:revoir-adapt}, conditioned on the full belief rather than a single committed profile. In the reported variants:

\paragraph{Planner.}
Generates a concrete numbered action plan conditioned on the belief summary:
\beigefile[label=prt:adapt-plantime-planner-system]{ADAPT Stage-1 Planner System Prompt}{adapt_prompts/plan_gen_system.txt}
\beigefile[label=prt:adapt-plantime-planner-message]{ADAPT Stage-1 Planner Message}{adapt_prompts/plan_gen_request.txt}

\paragraph{Simulator.}
Deduces the plan's final scene state as structured facts (entities created, items mixed and cooked, serving order, final locations) without stepping the true environment:
\beigefile[label=prt:adapt-plantime-simulator-system]{ADAPT Planning-time Environment Simulator System Prompt}{adapt_prompts/state_deduction_system.txt}
\beigefile[label=prt:adapt-plantime-simulator-message]{ADAPT Planning-time Environment Simulator Message}{adapt_prompts/state_deduction_request.txt}

\paragraph{Proxy Judge.}
Evaluates each preference as satisfied, violated, or inapplicable given the deduced state. The per-preference outcomes are aggregated into $\bar{G}$; under the factored belief this is the mean-field weighted PSR of Appendix~\ref{sec:revoir-adapt}, with each preference weighted by its inclusion probability $p_m$.
\beigefile[label=prt:adapt-proxy-judge-system]{ADAPT Proxy Judge System Prompt}{adapt_prompts/pref_eval_system.txt}
\beigefile[label=prt:adapt-proxy-judge-message]{ADAPT Proxy Judge Message}{adapt_prompts/pref_eval_request.txt}

\subsection{Execution Policy \texorpdfstring{$\pi^{\mathrm{act}}$}{πact}}
\label{sec:adapt_execution}

In the execution stage the planner emits one action per step, conditioned on the full elicited belief (phrased as information gathered from the user). Generation is constrained by a scene-specific grammar rebuilt from the live environment each step, forcing syntactically valid actions over currently valid entities; the loop runs greedily until \texttt{Declare Done}. The system prompt carries the task, scene, belief, and action vocabulary and ends by asking for the next action with no \texttt{user}-role template. At the first step it is the only message; at later steps it is followed by the running rollout, with each prior action appended as an \texttt{Action: <action>} assistant turn and each environment result as an \texttt{Observation: <result>} \texttt{user} turn (omitted when a step yields no observation). The \texttt{user} turns thus carry only observations.
\beigefile[label=prt:adapt-exetime-action-system]{ADAPT Stage-2 Action Generation System Prompt}{adapt_prompts/stage2_planner_system.txt}

\section{ADAPT Parameters}
\label{sec:adapt_hyperparameters}
\subsection{Environment Roles}
\label{sec:adapt_modular_roles}

For comparability with baselines in~\citep{patel_adapt_2025}, both the user simulator and REVOIR's modules --- the hypothesis proposer $Q_{\mathrm{LLM}}$, the consistency scorer $S_{\mathrm{LLM}}$, the question generator $\pi^{\mathrm{ask}}$, the planner--simulator--proxy judge that estimates $\bar{G}$, and the execution policy $\pi^{\mathrm{act}}$ --- are all driven by a single model in our ADAPT experiments: \texttt{meta-llama/Llama-3.1-70B-Instruct}, served locally via a vLLM OpenAI-compatible endpoint. Unlike CondAmbigQA, there is no separate judge model required since the goal reward $G = \mathrm{PSR}$ is computed by the benchmark's programmatic per-persona preference predicates against the final state scene.

\subsection{Benchmark Protocol}
\label{sec:adapt_protocol}

\begin{itemize}[leftmargin=*]
    \item \textbf{Personas and tasks:} 16 personas (12 seen / 4 held-out per split) and 8 breakfast tasks, evaluated under 4-fold cross-validation; we report means and standard deviations across split means.
    \item \textbf{Evaluation splits:} we reuse the benchmark's four-way split for comparability with the original ADAPT baselines.
\end{itemize}

\subsection{Belief and Reward Settings}
\label{sec:adapt_belief_params}

The four variants combine two belief representations with two update schemes (Appendix~\ref{sec:revoir-adapt}); all share the goal-reward pipeline below.
\begin{itemize}[leftmargin=*]
    \item \textbf{Belief representation:} \emph{joint} maintains $K{=}6$ profile particles; \emph{factored} maintains independent Bernoulli beliefs over $K{=}16$ atomic preferences.
    \item \textbf{Goal reward $\bar{G}$:} estimated by the planner--simulator--proxy judge pipeline with an explicit generated plan and independent state-deduction/per-preference scoring.
    \item \textbf{Consistency scoring:} verbalized $0$--$10$ scores, mapped to softmax weights (joint) or per-atom log-likelihood-ratio updates (factored).
\end{itemize}

\subsection{Elicitation-Stage Parameters}
\label{sec:adapt_stage1_params}

\begin{itemize}[leftmargin=*]
    \item \textbf{Candidate questions:} 3 per turn, sampled and deduplicated by the best-of-$K$ proposer.%
    \item \textbf{VoI lookahead samples:} the joint variant enumerates all 6 profiles; the factored variant draws 4 atom compositions per candidate question. %
    \item \textbf{Question cost:} since the original baselines are not cost-aware or question-throttled, we only impose a minimal flat per-question cost $C(m^{\mathrm{ask}}) = 5\times10^{-4}$ subtracted from $V^{\mathrm{ask}}$ to avoid infinite episodes.%
    \item \textbf{Stopping:} no hard question cap; elicitation only stops by the VoI rule $V^{\mathrm{act}} \geq V^{\mathrm{ask}} - C(m^{\mathrm{ask}})$. %
\end{itemize}

\subsection{Execution-Stage Parameters}
\label{sec:adapt_stage2_params}

\begin{itemize}[leftmargin=*]
    \item \textbf{Decoding:} greedy decoding (\texttt{temperature\_stage2: 0.0}) with grammar-constrained generation, the grammar rebuilt from the live scene each step.%
    \item \textbf{Step budget:} unlimited; episodes end at \texttt{Declare Done}.
\end{itemize}

\subsection{Sampling Temperatures}
\label{sec:adapt_temperatures}

\begin{itemize}[leftmargin=*]
    \item \textbf{Question generation:} $0.8$.
    \item \textbf{Mentalized response forecasting (VoI lookahead):} $0.7$.
    \item \textbf{Planner plan generation:} $0.3$.
    \item \textbf{Scoring and judging calls} (consistency scoring, state deduction, per-preference evaluation): $0.0$.
    \item \textbf{Execution-stage action generation:} $0.0$.
\end{itemize}

\end{document}